\documentclass[onefignum,onetabnum]{siamart171218}

\usepackage{lipsum}
\usepackage{amsfonts}
\usepackage{graphicx}
\usepackage{epstopdf}
\usepackage{algorithmic}
\ifpdf
  \DeclareGraphicsExtensions{.eps,.pdf,.png,.jpg}
\else
  \DeclareGraphicsExtensions{.eps}
\fi

\newsiamremark{assumption}{Assumption}
\newsiamremark{remark}{Remark}
\newsiamremark{hypothesis}{Hypothesis}
\crefname{hypothesis}{Hypothesis}{Hypotheses}
\newsiamthm{claim}{Claim}

\headers{Strongly-Polynomial and Validation Analysis of Policy Gradient}{C. Ju, G. Lan}

\title{Strongly-Polynomial Time and Validation Analysis of\\
Policy Gradient Methods\thanks{First submitted on 09/27/2024. Authors listed alphabetically.}}
\author{Caleb Ju\thanks{H. Milton Stewart School of Industrial \& Systems Engineering 
  (\email{calebju4@gatech.edu}). Supported by Department of Energy
Computational Science Graduate Fellowship under
Award Number DE-SC0022158.}
\and Guanghui Lan\thanks{H. Milton Stewart School of Industrial \& Systems Engineering 
  (\email{george.lan@isye.gatech.edu}).
  Supported by Air Force Office of Scientific Research grant FA9550-22-1-0447.
  }}

\usepackage{amsopn}

\makeatletter
\newcommand*{\addFileDependency}[1]{% argument=file name and extension
  \typeout{(#1)}% latexmk will find this if $recorder=0 (however, in that case, it will ignore #1 if it is a .aux or .pdf file etc and it exists! if it doesn't exist, it will appear in the list of dependents regardless)
  \@addtofilelist{#1}% if you want it to appear in \listfiles, not really necessary and latexmk doesn't use this
  \IfFileExists{#1}{}{\typeout{No file #1.}}% latexmk will find this message if #1 doesn't exist (yet)
}
\makeatother

\newcommand*{\myexternaldocument}[1]{%
    \externaldocument{#1}%
    \addFileDependency{#1.tex}%
    \addFileDependency{#1.aux}%
}

\usepackage[a4paper,
            bindingoffset=0.2in,
            top=2.5cm,bottom=2.5cm,right=2.5cm,left=2.5cm,
            footskip=.25in]{geometry}

\newcommand{\cA}{\mathcal{A}}
\newcommand{\cS}{\mathcal{S}}
\newcommand{\cP}{\mathcal{P}}
\newcommand{\edits}[1]{#1}
\newcommand{\fixed}[1]{#1}
\newcommand{\Qt}{Q^{\pi_t}}
\newcommand{\tQt}{\tilde{Q}^{\pi_t}}
\newcommand{\Qts}{Q^{\pi_t}(s, \cdot)}
\newcommand{\tQts}{\tilde{Q}^{\pi_t}(s, \cdot)}

\newcommand{\Qtq}{Q^{\pi_t}(q, \cdot)}
\newcommand{\tQtq}{\tilde{Q}^{\pi_t}(q, \cdot)}

\newcommand{\piq}{\pi(\cdot \vert q)}
\newcommand{\pits}{\pi_t(\cdot \vert s)}
\newcommand{\pitq}{\pi_t(\cdot \vert q)}

\newcommand{\psit}{\psi^{\pi_t}}

\newcommand{\DA}{\Delta_{\vert \cA \vert}}

\newlength{\leftstackrelawd}
\newlength{\leftstackrelbwd}
\def\leftstackrel#1#2{\settowidth{\leftstackrelawd}%
{${{}^{#1}}$}\settowidth{\leftstackrelbwd}{$#2$}%
\addtolength{\leftstackrelawd}{-\leftstackrelbwd}%
\leavevmode\ifthenelse{\lengthtest{\leftstackrelawd>0pt}}%
{\kern-.5\leftstackrelawd}{}\mathrel{\mathop{#2}\limits^{#1}}}

\let\texdisplaystyle\displaystyle
\def\displaytotextstyle{\textstyle\let\displaystyle\texdisplaystyle}
\newenvironment{talign}
 {\let\displaystyle\displaytotextstyle\align}
 {\endalign}
\newenvironment{talign*}
 {\let\displaystyle\displaytotextstyle\csname align*\endcsname}
 {\endalign}

\usepackage[nocompress]{cite}
\usepackage{booktabs}
\usepackage{subcaption}
\usepackage[export]{adjustbox} % https://tex.stackexchange.com/questions/101858/make-two-figures-aligned-at-top % http://ctan.org/pk    g/adjustbox
\usepackage{tikz}

\ifpdf
\hypersetup{
  pdftitle={STRONGLY-POLYNOMIAL TIME AND VALIDATION ANALYSIS OF POLICY GRADIENT METHODS},
  pdfauthor={C. Ju, G. Lan}
}
\fi

\myexternaldocument{ex_supplement}

\begin{document}

\maketitle

\noindent {\sl \scriptsize 
Dedicated to Professor Yinyu Ye in honor of his foundational contributions to optimization and Markov Decision Processes, on the occasion of his retirement celebration from Stanford University.
}

% REQUIRED
\begin{abstract}
    This paper proposes a novel termination criterion,  termed the advantage gap function, for finite state and action Markov decision processes (MDP) and reinforcement learning (RL). 
    By incorporating this advantage gap function into 
    the design of step size rules and deriving a new linear rate of convergence that is independent of the stationary state distribution of the optimal policy, we demonstrate that policy gradient methods can solve MDPs in strongly-polynomial time.
    To the best of our knowledge, this is the first time that such strong convergence properties have been established for policy gradient methods.
Moreover, in the stochastic setting, where only stochastic estimates of policy gradients are available, we show that the advantage gap function provides 
close approximations of
the optimality gap for each individual state and exhibits a 
sublinear rate of convergence at every state.
The advantage gap function can be easily estimated in the stochastic case, and
when coupled with easily computable upper bounds on policy values, they provide
a convenient way to validate the solutions generated by policy gradient methods. Therefore, our developments offer a principled and computable measure of optimality for RL, whereas current practice tends to rely on algorithm-to-algorithm or baseline comparisons with no certificate of optimality.
%    which is used to design policy gradient methods that are strongly polynomial for solving unregularized MDPs and to construct certificates of optimality of stochastic policy gradient methods for MDPs and reinforcement learning (RL) more generally.}
 %   Since RL lacks a principled and computable measure of optimality, practice tends to rely on algorithm-to-algorithm or baselines comparisons with no certificate of optimality. 
  %  The gap function is simple, convenient, and computable, and it can be incorporated into RL algorithms to provide both upper and lower bounds on the optimality gap.
    % Therefore, convergence of the gap function is a stronger mode of convergence than convergence of the optimality gap, and it is equivalent to a new notion we call distribution-free convergence, where convergence is independent of any problem-dependent distribution.
    % We show the basic policy mirror descent exhibits fast distribution-free convergence for both the deterministic and stochastic setting.
    % We leverage the distribution-free convergence to a uncover a couple new results. 
    % First, the deterministic policy mirror descent can solve unregularized MDPs in strongly-polynomial time.
    % Second, accuracy estimates can be obtained with no additional samples while running stochastic policy mirror descent and can be used as a termination criteria, which can be verified in the validation step. 
    % We propose a post-processing, validation step to verify that accuracy of a single policy as well.
\end{abstract}

% REQUIRED

% REQUIRED

\section{Introduction}
Reinforcement learning (RL) generally refers to Markov Decision Processes (MDP) with unknown transition kernels. The increasing interest in applying RL to real-world applications over the last decade is fueled not only by its success in domains like robotics, resource allocation, and optimal control~\cite{hu23towards,khan2020systematic,enda13applying}, but more recently in strategic game play (i.e., artificial intelligence for video games) and training large-language models via reinforcement learning from human feedback~\cite{mnih2015human,ouyang2022training}.
Such empirical successes have inspired intensive research on the development of principled MDP and RL algorithms during the last decade.

Depending on the mathematical formulations and the sub-fields from which the technology originates, MDP/RL algorithms can be grouped into
three different categories: dynamic optimization, linear optimization, and nonlinear optimization methods.
Dynamic optimization methods include the classic value iteration and policy iteration, as well as their stochastic variants, e.g.,~stochastic value iteration and Q-learning~\cite{sidford2018near,watkins1992q,sutton2018reinforcement}. Since certain important MDP/RL problems, such as those in finite state and action spaces, can be formulated as linear programs, several key linear optimization algorithms (e.g.,~Simplex methods, interior point methods, first-order methods and their stochastic variants)
have been proposed for MDP/RL~\cite{jin2020efficiently,wang2020randomized,ye2005new,ye2011simplex,puterman2014markov}. More recently, nonlinear optimization methods, particularly those based on policy gradient methods, 
have attracted much attention in both industry and academia~\cite{sutton1999policy,williams1992simple,mnih2015human,schulman2017proximal,lan2023policy,agarwal2021theory,bhandari2024global}.
Compared to the previous two categories, these nonlinear programming based methods offer several significant advantages.
First, they can handle large,  even continuous, state and action spaces by incorporating value function approximation techniques~\cite{agarwal2021theory,sutton1999policy,sutton2018reinforcement,lan2022policy}.
Second, they can effectively operate in various stochastic environments, such as using generative models or on-policy sampling to access random observations from the transition kernel in RL~\cite{sutton2018reinforcement}. Third, they can efficiently process and even benefit from nonlinear components (e.g.,~regularization terms) existing in
MDP/RL formulations ~\cite{li2021approximate,cen2021fast,neu2017unified}. On the other hand, while there exists a rich theoretical foundation for classical dynamic optimization and linear optimization methods, theoretic studies for nonlinear policy gradient methods are still lacking behind. 

% \edits{Popular algorithms include variants of value iteration, linear programming, and nonlinear methods such as policy gradient methods.
%Policy gradient methods are a popular choice due to their flexibility for handling nonlinear regularization terms and general state and action spaces.}

One prominent issue lies in the theoretical convergence guarantees that policy gradient methods provide for their approximate solutions to MDPs.
Although dynamic and linear optimization methods bound the optimality gap at every state, most policy gradient methods bound the optimality gap that is averaged over states with respect to the stationary state distribution of the optimal policy.
Note that this distribution is unknown and problem-dependent. Moreover, the optimality gap averaged over states is only necessary but not sufficient for the optimality gap to be small at every state.
Therefore, the standard notion of approximate optimality in most policy gradient methods is weaker than in dynamic and linear optimization.
In a similar vein, it is known that MDPs can be solved to optimality in strongly-polynomial time using some dynamic optimization approaches (e.g.,~policy iteration~\cite{ye2011simplex}) as well as linear optimization methods~\cite{ye2005new,ye2011simplex}. Yet, such strong convergence guarantees are unknown for policy gradient methods.

% Sources that support our claim of poor evaluation:
% - engstrom2020implementation: Holdouts some samples to estimate (last 5) rewards from a Monte Carlo sample
%     - https://github.com/MadryLab/implementation-matters/blob/5ee6ecb12545365d9178135e65576adfc0d82f52/src/run.py#L100
% - agarwal2021deep: define optimality gap w.r.t. a score of 1 from "human performance" 
% - henderson2018deep: average reward over 5 seeds. For tables, show final mean + ci rewards
% - schulman2017proximal: run of the algorithm by computing the average total reward (over 3 seeds) of the last 100 episodes.
%     - train for 1 million time steps, each episode truncates after 1000 steps (e.g., https://gymnasium.farama.org/environments/mujoco/half_cheetah/). So they measure average reward over last 10% of total samples
% - mnih2015human: show average reward per episode. Trained for 50m time steps (takes 38 days of total game play experience). 
%       - The trained agents were evaluated by playing each game 30 times for up to 5 min each time with different initial random conditions 
%       - Uses human-level comparison
Another substantial issue exists in the termination for policy gradient methods, especially under the aforementioned stochastic environments for RL. 
In fact, this issue is also shared by other dynamic optimization methods, such as stochastic value iteration and Q-learning.
More specifically, the solutions generated by dynamic optimization or policy gradient type algorithms are typically judged by their empirical costs compared to other competing algorithms or some a priori threshold, e.g., based on prior knowledge of the environment or a human baseline~\cite{mnih2015human,agarwal2021deep,schulman2017proximal,henderson2018deep}.
Linear programming, on the other hand, provides both easily accessible primal objective values and duality gaps to monitor the progress of the algorithm. 
Because such a convenient and computable gap is not currently known for other RL methods, it is in general challenging to determine when a sufficiently good policy has been found, especially for new and complex environments.
Moreover, even obtaining accurate estimates of the objective function is non-trivial in RL, since the underlying MDP and RL algorithm are stochastic. 
This is further exacerbated by the fact the cost function in RL, which is the expectation of a random infinite-horizon sum, can have a large variance.
A common solution is to run an RL algorithm across a small number of seeds (e.g.,~3 or 5), and then plot confidence intervals or report some statistics of Monte Carlo estimates of the objective function from each seed~\cite{engstrom2020implementation,schulman2017proximal,agarwal2021deep}.
However, scaling this to more seeds to further reduce estimation errors is intractable since even training one seed can take millions of simulation steps for certain RL problems~\cite{mnih2015human}.

In this paper, we attempt to address these aforementioned issues associated with policy gradient methods, and some of our results can also be extended to other dynamic optimization methods.
Central to our development is a novel termination criterion called the \textit{advantage gap function};
see~\eqref{eq:gap_function_definition} for a formal definition.
We show the advantage gap function being small is necessary and sufficient for the optimality gap to be small at every state (see Proposition~\ref{prop:gap_functions}). 
This is stronger than previous notions of approximate optimality for policy gradient methods, which only bound the aggregated optimality gap, i.e.,~the optimality gap averaged over the steady state distribution of the (unknown) optimality policy, denoted by $\nu^*$.
We call such strong convergence guarantees \textit{distribution-free}, since they do not depend on the distribution $\nu^*$.
Importantly, by incorporating a novel ``scheduled'' geometrically increasing step size rule, we show that  the policy mirror descent (PMD)~\cite{lan2023policy}, a recently developed first-order method, can achieve linear rates of convergence that are distribution-free.
This is the first time such strong convergence results have been shown for first-order methods.
Additionally, by embedding the advantage gap function into the aforementioned ``scheduled'' step size rule for solving MDPs without regularization, we can improve the runtime of PMD to be strongly-polynomial.
For the first time, this extends the celebrated result of Ye, who showed the simplex method and Howard's policy iteration are strongly-polynomial~\cite{ye2011simplex}, to first-order methods.

It turns out the advantage gap function closely approximates the optimality gap in that it can be used to measure a lower bound on the optimality gap at each state and a universal upper bound on the optimality gap for all states. In particular, 
we show that stochastic PMD for solving RL can minimize the advantage gap function at a sublinear rate of convergence that is distribution-free.
This result ensures the policy value function and advantage gap provide accessible estimates that closely approximate the objective value and optimality gap for RL, respectively, similar to the primal objective and duality gap in linear programming methods. 
Moreover, since both the policy value function and advantage gap function are stochastic in RL and must be estimated by samples, we show their estimation errors can also be reduced at a similar sublinear rate of convergence that is distribution-free.
This ensures the proposed quantities can be reliably estimated and are suitable to use as a termination criterion and performance metric for RL. To the best of our knowledge, this is the first time such validation analysis procedures have been developed for solving these highly nonconvex RL problems, while some related previous studies have been restricted to stochastic convex optimization only~\cite{lan2012validation,lan2020first}.

This paper is organized as follows. We introduce the reinforcement learning and the advantage gap function in Section~\ref{sec:prob_interest}.
With the problem setup done, we also establish some duality theory for policy mirror descent at the end of Section~\ref{sec:prob_interest}.
Section~\ref{sec:deterministic_pmd} then establishes distribution-free convergence rates for the deterministic setting, as well as a (relatively simple) modification to obtain strongly-polynomial runtime for solving unregularized MDPs.
Distribution-free convergence is extended to the stochastic setting in Section~\ref{sec:improved_sublinear_convergence}.
We provide the analysis of online and offline stochastic accurate certificates in Section~\ref{sec:validation_analysis}.
We conclude with preliminary numerical experiments in Section~\ref{sec:experiments}.
% Finally, we conclude with some takeaway messages in Section~\ref{sec:conclusion}.

\subsection{Notation}
For a Hilbert space (e.g.,~real Euclidean space), let $\|\cdot\|$ be the induced norm and let $\|\cdot\|_*$ be its dual norm.
When appropriate, we specify the exact norm (e.g., $\ell_2$ norm, $\ell_1$ norm).

We denote the probability simplex over $n$ elements as
\begin{talign*}
    \Delta_n := \{ x \in \mathbb{R}^{n} : \sum_{i=1}^n x_i = 1, x_i \geq 0 \ \forall i \}.
\end{talign*}
For any two distributions $p,q \in \Delta_n$, we measure their Kullback–Leibler (KL) divergence by
% \begin{talign*}
    $\mathrm{KL}(q \vert \vert p)
    = 
    \sum_{i=1}^n p_i \log \frac{p_i}{q_i}$.
% \end{talign*}
Observe that the KL divergence can be viewed as a special instance of Bregman’s distance (or prox-function) widely used in the optimization literature. 
We define Bregman’s distance associated with a distance generating function $\omega : X \to \mathbb{R}$ for some set $X \subseteq \mathbb{R}^n$ as
\begin{talign*}
    D(q,p)
    :=
    \omega(p) - \omega(q) - \langle \nabla \omega (q), p-q \rangle, \ \forall p,q \in X.
\end{talign*}
The choice of $X = \Delta_n$ and Shannon entropy $\omega(p) := \sum_{i=1}^n p_i \log p_i$ results in Bregman's distance as the KL-divergence.
In this case, one can show Bregman's distance is 1-strongly convex w.r.t.~the $\ell_1$ norm: $D(q,p) \geq \|p-q\|_1^2$.
Another popular choice is $\omega(\cdot) = \frac{1}{2}\|\cdot\|_2^2$, the Euclidean distance squared and $X \subseteq \mathbb{R}^n$. 
Bregman's distance becomes $D(q,p) = \frac{1}{2}\|p-q\|_2^2$.

\section{Markov decision process, a gap function, and connections to (non-)linear programming} \label{sec:prob_interest}
An infinite-horizon discounted Markov decision process (MDP) is a five-tuple
$(\cS, \cA, \cP, c, \gamma)$, where 
$\cS$ is a finite state space,
$\cA$ is a finite action space,
and $\cP : \cS \times \cS \times \cA \to \mathbb{R}$ is the transition kernel where given a state-action $(s,a)$ pair, it reports probability of the next state being $s'$, denoted by $\mathcal{P}(s' \vert s,a)$.
The cost is $c : \cS \times \cA \to \mathbb{R}$ and
$\gamma \in [0,1)$ is a discount factor.
A feasible policy $\pi : \cA \times \cS \to \mathbb{R}$ determines the probability of selecting a particular action at a given state. 
We denote the space of feasible policies by $\Pi$.
Now, we write Bregman's distance between any two policies at state $s$ as
\begin{talign*}
    D_{\pi}^{\pi'}(s)
    :=
    D(\pi(\cdot \vert s), \pi'(\cdot \vert s))
    = 
    \omega(\pi'(\cdot \vert s)) - \omega(\pi(\cdot \vert s)) - \langle \nabla \omega (\pi(\cdot \vert s)), \pi'(\cdot \vert s) - \pi(\cdot \vert s) \rangle.
\end{talign*}
We measure a policy $\pi$'s performance by the action-value
function $Q^\pi : \cS \times \cA \to \mathbb{R}$
defined as
\begin{talign*}
    Q^\pi(s,a) :=
    \mathbb{E}[\sum_{t=0}^\infty \gamma^t \big [c(s_t, a_t) + h^{\pi(\cdot \vert s_t)}(s_t)] \mid
    s_0 = s, a_0 = a, a_t \sim \pi(\cdot \vert s_t),
    s_{t+1} \sim \cP(\cdot | s_t, a_t) \big],
\end{talign*}
where the function $h^\cdot(s) : \DA \to \mathbb{R}$ is closed strong convex function with modulus $\mu_h \geq 0$ with respect to (w.r.t.) the policy $\pi(\cdot \vert s)$, i.e., 
\begin{talign}
    h^{\pi(\cdot \vert s)}(s) - [h^{\pi'(\cdot \vert s)}(s) + \langle (h')^{\pi'(\cdot \vert s)}(s, \cdot), \pi(\cdot \vert s) - \pi'(\cdot \vert s) \rangle 
    \geq
    \mu_h D^\pi_{\pi'}(s), \label{eq:mu_h_defn}
\end{talign}
where $\langle \cdot, \cdot \rangle$ denotes the inner product over the action space $\cA$, and $(h')^{\pi'(\cdot \vert s)}(s, \cdot)$ denotes a subgradient of $h^\cdot(s)$ at $\pi'(\cdot \vert s)$.
The generality of $h$ allows modeling of the popular entropy regularization, which can induce safe exploration and learn risk-sensitive policies~\cite{li2021approximate,neu2017unified}, as well as barrier functions for constrained MDPs. 
Here we separate $h^{\pi(\cdot \vert s)}$ from the cost $c(s,a)$ to take the advantage of
its strong convexity in the design and analysis of algorithms.
Moreover, we define the state-value function
$V^\pi : \cS \to \mathbb{R}$ associated with $\pi$
as
\begin{talign}\label{eq:def_V_function}
V^\pi(s) :=
\mathbb{E}
 \big[
 \sum_{t=0}^\infty \gamma^t [c(s_t, a_t) + h^{\pi_t(\cdot \vert s_t)}(\cdot \vert s_t)] \mid
s_0 = s, a_t \sim \pi(\cdot \vert s_t),
s_{t+1} \sim \cP(\cdot | s_t, a_t)
\big].
\end{talign}
It can be easily seen from the definitions of $Q^\pi$ and $V^\pi$ that
\begin{talign}
V^\pi(s) 
&= 
\sum_{a \in \cA} \pi(a \vert s) Q^{\pi}(s,a)
=
\langle Q^\pi(s, \cdot), \pi(\cdot \vert s) \rangle \label{eq:QV1} \\
Q^\pi(s, a) &= c(s, a) +h^{\pi(\cdot \vert s)}(s) + \gamma \sum_{s' \in \cS}  \cP(s'| s, a) V^\pi(s'). \label{eq:QV2}
\end{talign}

The main objective in MDP is to find an optimal policy $\pi^* : \cA \times \cS \to \mathbb{R}$ such that
\begin{equation} \label{eq:opt_objective}
V^{\pi^*}(s) \le V^\pi(s), \forall \pi(\cdot \vert s) \in \DA, \forall s \in \cS.
\end{equation} 
Sufficient conditions that guarantee the existence of $\pi^*$ have been intensively studied (e.g.,~\cite{BertsekasShreve1996,puterman2014markov}).
Note that~\eqref{eq:opt_objective} can be formulated as a nonlinear optimization  problem with a single objective function.
Given an initial state distribution $\rho \in \Delta_{\vert \cS \vert}$, let $f_\rho$ be defined as 
\begin{talign} \label{eq:RL0}
f_\rho(\pi) 
:= \sum_{s \in \cS} \rho(s) \cdot V^\pi(s).
\end{talign}
When $\rho$ is strictly positive, one can see an optimal solutions to~\eqref{eq:opt_objective} is also optimal for~\eqref{eq:RL0}.
While the distribution $\rho$ can be arbitrarily chosen, prior policy gradient methods typically select $\rho$ to be the 
stationary state distribution induced by the optimal policy $\pi^*$, denoted by $\nu^* := \nu^{\pi*}$~\cite{lan2023policy,lan2022policy,li2022homotopic}.
As such, the problem reduces to $
 \min_{\pi \in \Pi} f_{\nu^*}(\pi)$.
This aggregated objective function is commonly formulated and solved by nonlinear programming approaches such as policy gradient methods.

\subsection{Performance difference and advantage function} \label{sec:perf_diff_and_gap}
Given a policy $\pi(\cdot \vert s) \in \DA$, we
define the discounted state visitation distribution $\kappa^\pi_q : \cS \to \mathbb{R}$ by
\begin{talign}
    \kappa^\pi_q(s)
    &:=
    (1-\gamma) \sum_{t=0}^\infty \gamma^t \mathrm{Pr}^\pi\{s_t=s \vert s_0=q\} \label{eq:visitation_measure},
\end{talign}
where $\mathrm{Pr}^{\pi}\{s_t = \cdot \vert s_0=q\}$ is the distribution of state $s_t$ when following policy $\pi$ and starting at state $q \in \cS$.
In the finite state case, we can also view $\kappa^\pi_q \in \mathbb{R}^{\vert \cS \vert}$ as a vector. 
Clearly, $\kappa^{\pi}_s \in \Delta_{\vert \cS \vert}$, and one can show the lower bound $\kappa^\pi_s(s) \geq 1-\gamma$.

We now state an important ``performance difference'' lemma
which tells us the difference on the value functions for two different policies.
The proof has appeared in numerous previous works (e.g.,~\cite[Lemma 1]{lan2022policy}), so we skip it.
\begin{lemma} \label{lem:performance_diff_deter}
    Let $\pi$ and $\pi'$ be two feasible policies. Then
    we have
    \begin{talign}
        V^{\pi'}(s) - V^\pi(s) 
        &= \frac{1}{1-\gamma} \sum_{q \in \cS}  \psi^\pi(q, \pi'(\cdot \vert q)) \kappa_s^{\pi'} (q), \forall s \in \cS, \label{eq:per_diff}
    \end{talign}
    where for a given $p \in \DA$, the advantage function is defined as 
    \begin{equation} \label{eq:def_advantage}
        \psi^\pi(s, p)
        := \langle Q^\pi(s, \cdot), p \rangle -V^\pi(s) +  h^{p}(s) - h^{\pi(\cdot \vert s)}(s).
    \end{equation}
\end{lemma}
The performance difference lemma is striking because it provides an exact characterization between the values of any two policies.
Historically, the advantage function without regularization seems to have first appeared in~\cite{kakade2002approximately}, and the generalized form including regularization was shown in~\cite{lan2022policy}.
In the former, it was presented as an inequality and used to establish a monotonicity-type result to show convergence to optimality.
Similarly, the performance difference lemma was used to show various convergence results in policy gradient methods~\cite{lan2023policy,khodadadian2021linear,agarwal2021theory,bhandari2024global}.
We take a different approach and use it to derive a computable measure of optimality.

\subsection{Advantage gap function and distribution-free convergence}
For any policy $\pi \in \Pi$, 
the \textit{advantage gap function} is the mapping $g^\pi : \cS  \to \mathbb{R}$ defined as
\begin{talign} \label{eq:gap_function_definition}
    g^\pi(s) := \max_{p \in \DA} \{-\psi^{\pi}(s, p) \}.
\end{talign}
For some common cases, the gap function can be computed exactly.
% In fact, we report two cases where a closed-form expression exists. 
In the unregularized case of $h^\cdot(s) = 0$, then $g^\pi(s) = \max_{a \in \cA} \{-\psi^{\pi}(s,e_a)\}$, where $e_a \in \mathbb{R}^{\vert \cA \vert}$ is the elementary basis vector in Euclidean space, i.e., $g^\pi(s)$ is the largest value in the negative advantage function.
When using a negative entropy regularizer $h^{\pi(\cdot \vert s)}(s) = \sum_{a \in \cA} \pi(a \vert s) \ln \pi(a \vert s)$, then~\cite[Section 4.4.10]{beck2017first} provides the closed-form solution that can be easily computed,
\begin{talign*}
    g^\pi(s) = \log (\sum_{a \in \cA} \mathrm{exp}\{-Q^\pi(s,a)  +V^\pi(s)\}) + h^{\pi(\cdot \vert s)}(s).
\end{talign*}
\edits{For a general convex regularization $h^\cdot(s)$, a closed-formed solution may not exist. We can still approximately evaluate them by solving $\vert \cS \vert$ separable, convex programs,}
\begin{talign} \label{eq:advgap_subproblem}
    \min_{p \in \DA} \{\langle -Q^\pi(s,\cdot), p \rangle + h^p(s)\}, \quad \forall s \in \cS.
\end{talign}
\edits{When $h^\cdot(s)$ is simple, i.e.,~prox-mappings can be computed exactly~\cite[Section 5.3.2]{ben2001lectures}, the convex program can be efficiently solved to arbitrary accuracy $\epsilon > 0$.
If $h^\cdot(s)$ is an $L$-smooth convex function, then $O(\epsilon^{-1/2})$ iterations of Nesterov's accelerated gradient method are required~\cite{nesterov1983method}, where we only show dependence on $\epsilon$. 
If $h^\cdot(s)$ is a continuous convex function, then $O(\epsilon^{-1})$ iterations of Nesterov's accelerated gradient method with a smoothing technique are needed~\cite{nesterov2005smooth}. For simplicity, we assume $g^\pi(s)$ can be computed exactly, but our forthcoming results can be extended to account for an $\epsilon$ error.}

We are ready to establish one of our fundamental yet simple results, which says the gap function from~\eqref{eq:gap_function_definition} can be used to estimate both upper and lower bounds on the optimality gap.
\begin{proposition} \label{prop:gap_functions}
    For any policy $\pi$,
    \begin{talign*}
        % \max_{p \in \Pi} \{-\psi^{\pi}(s,p)\}
        g^\pi(s)
        \leq
        V^{\pi}(s) - V^{\pi^*}(s)
        &\leq
        \textstyle (1-\gamma)^{-1} \max_{s' \in \cS} g^{\pi}(s').
    \end{talign*}
\end{proposition}
\begin{proof}
    We first prove the lower bound.
    With any $\hat{\pi}(\cdot \vert s) \in \mathrm{argmax}_{p \in \DA} \{-\psi^{\pi}(s, p)\}$, we have
    \begin{talign*} 
        -\psi^{\pi}(s,\hat{\pi}(\cdot \vert s)) 
        = \max_{p \in \DA} -\psi^{\pi}(s, p) 
        \geq 
        -\psi^{\pi}(s, \pi(\cdot \vert s))
        = 0.
    \end{talign*}
    where the last equality is by the definition given in~\eqref{eq:def_advantage}. Therefore,
    \begin{talign*}
        V^{\pi}(s) - V^{\pi^*}(s)
        \geq
        V^{\pi}(s) - V^{\hat{\pi}}(s) 
        &\stackrel{\eqref{eq:per_diff}}{=}
        \frac{1}{1-\gamma} \sum_{q \in \cS} -\psi^\pi(q, \hat{\pi}(\cdot \vert q)) \kappa^{\hat{\pi}}_s(q) \\
        &\leftstackrel{\substack{-\psi^{\pi}(s, \hat{\pi}(\cdot \vert q)) \geq 0 \\ \text{and}~\eqref{eq:visitation_measure}}}{\geq}
        -\psi^\pi(s, \hat \pi(\cdot \vert s)),
    \end{talign*}
    which by construction of $\hat{\pi}$ establishes the lower bound.

    As for the upper bound, 
    % in view of the lower bound on $V^{\pi}(s) - V^{\pi^*}(s) \leq 0$, there exists a $p^* \in \DA$ such that $\psi^{\pi}(s,p^*) \geq 0$.
    % In view of the fact $\kappa^{\pi^*}_q(s) \geq 0$ for any $q,s \in \cS$ since $\kappa^{\pi^*}_q$ is a distribution over states, 
    we recall $\kappa_s^{\pi^*}$ from~\eqref{eq:visitation_measure} is a distribution over states.
    So,
    \begin{talign*}
        V^{\pi}(s) - V^{\pi^*}(s)
        &\stackrel{\eqref{eq:per_diff}}{=}
        \frac{1}{1-\gamma} \sum_{q \in \cS} -\psi^\pi(q, {\pi^*(\cdot \vert s)}) \kappa^{{\pi^*}}_s(q) \\
        &~\leq
        \frac{1}{1-\gamma}\sum_{q \in \cS} \max_{p \in \DA} \{ -\psi^{\pi}(q,p) \}  \kappa^{\pi^*}_s(q) \\
        &~\leq
        \frac{1}{1-\gamma} \max_{s' \in \cS, p \in \DA} \{ -\psi^{\pi}(s',p) \} .
    \end{talign*}
\end{proof}
\edits{While this paper focuses on finite state and actions, the above proposition can be easily extended to general state and action spaces~\cite{lan2022policy}. 
To do so, the advantage function $\psi^\pi(s,p)$ must be modified so that it is a function of actions $a \in \cA$ instead of distributions $p \in \DA$. This is because for general action spaces, policies are mappings from states to actions; see~\cite[Lemma 1]{lan2022policy} for a concrete derivation. Let us denote such an advantage function as $\psi^\pi(s,a)$. Similarly, the gap function must be modified to maximize over actions in $\cA$. Finally, the ``max'' in the gap function and Proposition~\ref{prop:gap_functions} is replaced with a supremum. However, unlike the finite space setting, evaluating the gap function to arbitrary accuracy may be intractable in general space. The advantage function $\psi^\pi(\cdot,\cdot)$ becomes a general infinite-dimensional mapping, so computing it exactly is non-trivial. More importantly, $\psi^\pi(s,a)$ may be nonconvex in $a \in \cA$. Understanding and addressing these issues is beyond the scope of this paper, and we refer the reader to~\cite{lan2022policy} for further discussions.}

We also present a similar result for the aggregation of multiple advantage functions, which will be useful in the stochastic setting where one can only estimate the advantage function.
\begin{proposition} \label{prop:gap_functions_2}
    \sloppy For a set of (possibly random) policies $\{\pi_t\}$,
    \begin{talign*}
        \sum_{t=0}^{k-1} \mathbb{E}_{\{\pi_t\}} [V^{\pi_t}(s) - V^{\pi^*}(s)]
        &\leq
        \textstyle (1-\gamma)^{-1} \max_{s' \in \cS} \mathbb{E}_{\{\pi_t\}}[ g^{\pi_{[k]}}(s') ], \ \forall s \in \cS,
    \end{talign*}
    where we define the aggregated advantage gap function as
    \begin{talign} \label{eq:agg_gap_function_definition}
        g^{\pi_{[k]}}(s) := \max_{p \in \DA} \{ -\sum_{t=0}^{k-1} \psi^{\pi_t}(s,p) \}
    \end{talign}
\end{proposition}
\begin{proof}
    Similar to Proposition~\ref{prop:gap_functions},
    \begin{talign*}
        \sum_{t=0}^{k-1} [V^{\pi_t}(s) - V^{\pi^*}(s)]
        &\stackrel{\eqref{eq:per_diff}}{=}
        \frac{1}{1-\gamma} \sum_{q \in \cS} \sum_{t=0}^{k-1} - \psi^{\pi_t}(q, {\pi^*}(\cdot \vert q)) \kappa^{{\pi^*}}_s(q) \\
        &~\leftstackrel{\kappa^{\pi^*}_s(q) \geq 0}{\leq}
        \frac{1}{1-\gamma} \sum_{q \in \cS} \max_{p \in \DA} \big\{ \sum_{t=0}^{k-1} - \psi^{\pi_t}(q, p) \big \} \kappa^{{\pi^*}}_s(q).
    \end{talign*}
    Since $\kappa^{{\pi^*}}_s$ is a deterministic probability distribution, applying expectation yields
    \begin{talign}
        \sum_{t=0}^{k-1} \mathbb{E}[V^{\pi}(s) - V^{\pi^*}(s)]
        &\leq
        \frac{1}{1-\gamma} \sum_{q \in \cS} \kappa^{{\pi^*}}_s(q) \mathbb{E}\big[ \max_{p \in \DA} \big\{ \sum_{t=0}^{k-1} -\psi^{\pi_t}(q, p) \big \} \big] \label{eq:V_gap_ub_agap} \\
        &\leq
        \frac{1}{1-\gamma} \max_{s \in \cS} \mathbb{E}\big[ \max_{p \in \DA} \big\{ \sum_{t=0}^{k-1} - \psi^{\pi_t}(s, p) \big \} \big]. \nonumber
    \end{talign}
\end{proof}

Clearly, when the advantage gap function is small, say $g^\pi(s) \leq (1-\gamma)\epsilon$ for all states $s \in \cS$, then $V^{\pi}(s) - V^{\pi^*}(s) \leq \epsilon$ for all states, which implies $f_{\nu^*}(\pi) - f_{\nu^*}(\pi^*) \leq \epsilon$, where we recall the $\nu^*$-weighted objective $f_{\nu^*}(\pi) = \mathbb{E}_{s \sim \nu^*}[V^\pi(s)]$.
That is, making the negative advantage function small is a sufficient condition for the $\nu^*$-weighted optimality gap to be small.
However, it is not necessary, because
% This is because
\begin{talign} \label{eq:gap_to_opt_conversion}
    g^\pi(s) \leq V^\pi(s) -V^{\pi^*}(s) \leq (\min_{s'} \nu^*(s'))^{-1}[f_{\nu^*}(\pi) - f_{\nu^*}(\pi^*)], \quad \forall s \in \cS,
\end{talign} 
where $\nu^*(s')$ can be arbitrarily small for some state $s'$.
On the other hand, the previous propositions say making the advantage gap small at every state is necessary and sufficient for the value optimality gap $V^{\pi}(s) - V^{\pi^*}(s)$ to be small at every state.

Therefore, an algorithm is \textit{distribution-free} or exhibits distribution-free convergence if for every $\epsilon > 0$, it outputs a policy $\pi_k$ such that
\begin{talign} \label{eq:distribution_free}
    V^{\pi_k}(s) - V^{\pi^*}(s)
    \leq
    \epsilon, \ \forall s \in \cS,
\end{talign}
where the iteration complexity $k$ can depend on $\epsilon > 0$ but not on the steady state distribution $\nu^*$ of the optimal policy.
Equivalently, 
distribution-free convergence provides the convergence guarantee,
\begin{talign} \label{eq:dist_free_bnd}
    \max_{\rho \in \Delta_{\vert \cS \vert}}[f_\rho(\pi) - f_\rho(\pi^*)] \leq \epsilon.
\end{talign}
This is clearly stronger than the guarantee $f_{\nu^*}(\pi) - f_{\nu^*}(\pi^*) \leq \epsilon$ as suggested after~\eqref{eq:RL0}. 

\edits{To crystallize the improvement of distribution-free convergence, we examine a two-state MDP parameterized by any $\Delta \in (0,1)$ as shown in Fig.~\ref{fig:mdp_example}. Here, the stationary distribution $\nu^*$ has a small value of $\Delta/(1+\Delta)$. If one can only ensure the $\nu^*$-weighted optimality gap from~\eqref{eq:gap_to_opt_conversion} is at most $\epsilon$, then one can only guarantee $V^{\pi}(s_1) - V^{\pi^*}(s_1) \leq {(1+\Delta)\epsilon}/{\Delta} \leq 2\epsilon/\Delta$, where state $s_1$ is from Fig.~\ref{fig:mdp_example}.
When $\Delta < (1-\gamma)\epsilon$, this bound is in vain. On the other hand, an $\epsilon$-small distribution-free convergence in the sense of~\eqref{eq:dist_free_bnd} guarantees $V^{\pi}(s_1) - V^{\pi^*}(s_1) \leq \epsilon$ and avoids dependence on $\Delta$.}

\begin{figure}[H]
\centering
\begin{tikzpicture}[scale=0.15]
\tikzstyle{every node}+=[inner sep=0pt]
\draw [black] (27,-31.7) circle (3);
\draw (27,-31.7) node {$s_1$};
\draw [black] (48.4,-31.7) circle (3);
\draw (48.4,-31.7) node {$s_2$};
\draw [black] (45.494,-32.443) arc (-77.96434:-102.03566:37.379);
\fill [black] (45.49,-32.44) -- (44.61,-32.12) -- (44.82,-33.1);
\draw (37.7,-33.76) node [below] {$1$};
\draw [black] (29.93,-31.061) arc (100.32686:79.67314:43.342);
\fill [black] (29.93,-31.06) -- (30.81,-31.41) -- (30.63,-30.43);
\draw (37.7,-29.86) node [above] {$a_1:\Delta,\mbox{ }a_2:1$};
\draw [black] (51.08,-30.377) arc (144:-144:2.25);
\draw (55.65,-31.7) node [right] {$a_1:1-\Delta,\mbox{ }a_2:0$};
\fill [black] (51.08,-33.02) -- (51.43,-33.9) -- (52.02,-33.09);
\end{tikzpicture}
\caption{A two-state and two-action MDP. An arc describes the transition probability with every action, where $\Delta \in (0,1)$ is an arbitrary parameter. From state $s_1$, the next state is always $s_2$. The cost is $c(s,a) = \mathbf{1}[s=s_1]$. Thus, the optimal action avoids state $s_1$. This means an optimal policy at state $s_2$ selects action $a_1$ with probability $1$. One can then derive $\nu^* = [\nu^*(s_1), \nu^*(s_2)] = [\frac{\Delta}{1+\Delta}$, $\frac{1}{1+\Delta}]$.}
\label{fig:mdp_example}
\vspace{-5mm}
\end{figure}

% As explained in the previous paragraph, distribution-free convergence implies we can make $f_{\nu^*}(\pi) - f_{\nu^*}(\pi^*)$ small, but the converse is not true in general.
% Hence, it is a stronger form of convergence.

Before we move onto proving what algorithms exhibit distribution-free convergence, we will briefly explore some convex programming formulations for reinforcement learning, which will offer another perspective to the proposed advantage gap function.

\subsection{Convex programming and duality theory of (regularized) RL}
For a given distribution $\rho \in \Delta_{\vert \cS \vert}$ and policy $\pi(\cdot \vert s) \in \DA$,
we introduce the weighted visitation $\eta^\pi_\rho : \cS \to \mathbb{R}$,
\begin{talign}
    \eta^\pi_\rho(s)
    &:=
    (1-\gamma)^{-1} \sum_{q \in \cS} \rho(q) \cdot \kappa^{\pi}_q(s), \label{eq:weighed_visitation}
\end{talign}
where recall $\kappa^{\pi}_q$ is the state-visitation vector from~\eqref{eq:visitation_measure}.
Since $\rho$ and $\kappa^\pi_q$ are distributions over states, then $\eta^\pi_\rho(s) \in [0, (1-\gamma)^{-1}]$ for every state $s$. We next state a useful supporting result that re-writes the objective function, and its proof is deferred to Appendix~\ref{sec:missing_pfs_from_prob_interest}.

\begin{lemma} \label{lem:fpi_to_fx}
    For any policy $\pi$ and distribution over states $\rho$,
    \begin{talign*}
        f_\rho(\pi)
        =
        \sum_{s \in \cS} \rho(s) V^\pi(s)
        =
        \sum_{s \in \cS} [c(s, \pi(\cdot \vert s)) + h^{\pi(\cdot \vert s)}(s)] \cdot \eta^\pi_\rho(s).
    \end{talign*}
\end{lemma}

Policy optimization problems solve problems of the form $\min_{\pi \in \Pi} f_\rho(\pi)$. 
We will show an equivalent convex optimization problem, which will help derive some primal dual results.
Our problem transformation relies on the following observation~\cite[Theorem 6.9.1]{puterman2014markov}.
\begin{lemma} \label{lem:pi_to_x}
    For any distribution $\rho \in \Delta_{\vert \cS \vert}$ and policy $\pi (\cdot \vert s) \in \DA$, the vector $x^\pi(a,s) := \eta^\pi_\rho(s) \pi(a \vert s)$ satisfies
    \begin{talign}
        &\sum_{a} x(a,s) - \gamma \sum_{s',a} \mathcal{P}(s \vert s', a) x(a,s') = \rho(s) \label{eq:stationary_balance_eqn} \\
        &\sum_{a} x(a,s) = \eta_\rho^\pi(s) \label{eq:state_visitation_constraint} \\
        &x(a,s) \geq 0, ~ \forall a \in \cA, s \in \cS. \nonumber
    \end{talign}
    Equivalently,~\eqref{eq:stationary_balance_eqn} can be written as $(\hat{I} - \gamma P)^Tx = \rho$ for some matrices $\hat{I}$ and $P$.
\end{lemma}
In view of the lemma, we consider the following (possibly nonlinear) program.
Fix some distribution $\rho  \in \Delta_{\vert \cS \vert}$ and policy $\pi(\cdot \vert s) \in \DA$ (e.g., $\pi^*$) in
\begin{talign} \label{eq:rl_np}
    \min_{x} ~& \{ f_{\rho, \pi}(x) := \langle c, x \rangle + \sum_{s \in \cS} \eta^\pi_\rho(s) \bar{h}^{x(\cdot , s)}(s) \} \\
    \text{s.t.} ~& 
    (\hat{I} - \gamma P)^Tx = \rho \nonumber \\
    & x \in X(\rho, \pi) := \{x \in \mathbb{R}^{\vert \cA \vert \times \vert \cS \vert} \vert \sum_{a \in \cA} x(a,s) = \eta^\pi_\rho(s), x \geq \mathbf{0} \}, \nonumber
\end{talign}
where $\bar{h}^{x(\cdot, s)}(s) := h^{u(\cdot, s)}(s)$ and $u(\cdot, s) = x(\cdot, s) / \sum_{a \in \cA} x(a,s)$ for any $x \in X(\rho,\pi)$. Here, the minor notation change from $h$ to $\bar{h}$ is needed since $h^\cdot(s)$ takes probability distribution as inputs (the function $\bar{h}^{x(\cdot, s)}(s)$ is still convex in $x(\cdot, s)$ for all $x \in X(\rho, \pi)$, since the normalization factor used to transform $x(\cdot, s)$ to $u(\cdot, s)$ is the same for any $x \in X(\rho, \pi)$).
% (since to how we defined the regularization $h^{\pi(\cdot \vert s)}(s)$ for a policy $\pi(\cdot \vert s) \in \DA$) 
For the unregularized case, i.e., $h^{x(\cdot, s)}(s) = 0$, the optimization problem is equivalent to the dual linear programming (LP) formulation of MDPs~\cite{puterman2014markov}, but with the additional constraint $\sum_a x(a,s) = \eta^\pi_\rho(s)$. 
The inclusion of this constraint permits one to view the set of values $x(\cdot, s)$ as a ``scaled'' policy $\eta_\rho^{\pi}(s) \cdot \pi'(\cdot \vert s)$ for some policy $\pi'(\cdot \vert s) \in \DA$, i.e., $x(\cdot, s)$ sums to $\eta_\rho^\pi(s)$ instead of summing to 1.

Consider the Lagrange function, $L_{\rho,\pi} (x,v) :=  \langle c, x\rangle + \sum_{s \in \cS}  \eta^{\pi}_\rho(s) \cdot \bar h^{x(\cdot, s)}(s) + \langle v, \rho - (\hat{I} - \gamma P)^Tx \rangle$.
This leads us to the dual program, $\underline{L}_{\rho, \pi} (v) := \min_{x \in X(\rho, \pi)} L_{\rho, \pi} (x,v)$.
\begin{lemma} \label{lem:dual_program}
    For any policy $\pi$ and its value function $V^{\pi} \in \mathbb{R}^{\vert \cS \vert}$,
    \begin{talign} \label{eq:dual_equivalent}
        \underline{L}_{\rho, \pi'} (V^{\pi})
        &=
        \langle V^\pi, \rho \rangle - \sum_{s \in \cS} \eta^{\pi'}_\rho(s) \cdot \max_{p \in \DA} \{-\psi^{\pi}(s, p) \}.
    \end{talign}
\end{lemma}
\begin{proof}
    We have
    \begin{talign*}
        &\underline{L}_{\rho, \pi'} (V^\pi) \\
        &=
        \min_{x \in X(\rho, \pi')} \big\{ \langle c, x \rangle + \sum_{s \in \cS} \eta^{\pi'}_\rho(s) \cdot \bar h^{x(\cdot, s)}(s) + \langle V^\pi, \rho - (\hat{I} - \gamma P)^Tx \rangle \big\} \\
        &=
        \langle V^\pi, \rho \rangle - \max_{x \in X(\rho, \pi')} \big\{  \langle {- c - (\gamma P - \hat{I})V^{\pi}}, x \rangle - \sum_{s \in \cS}\eta^{\pi'}_\rho(s) \cdot \bar h^{x(\cdot, s)}(s)  \big\} \\
        &=
        \langle V^\pi, \rho \rangle - \sum_{s \in \cS} \eta^{\pi'}_\rho(s) \cdot \max_{p \in \DA}  \big\{  \langle -c(s,\cdot) - [(\gamma P - \hat{I})V^\pi](s, \cdot), p \rangle - h^{p}(s)  \big\},
    \end{talign*}
    and we also have
    \begin{talign*}
        -c(s,a) - \big[(\gamma P - \hat{I}) V^\pi](s,a) 
        &= -(c(s,a) + h^{\pi(\cdot \vert s)}(s) + \gamma  {\mathbb{E}_{s' \sim P(\cdot \vert s, \pi(\cdot \vert s))}} [V^{\pi}(s')]\big) + V^{\pi}(s) + h^{\pi(\cdot \vert s)}(s) \\
        &\leftstackrel{\eqref{eq:QV2}}{=} 
        -Q^{\pi}(s,a) + V^{\pi}(s) + h^{\pi(\cdot \vert s)}(s).
    \end{talign*}
    In view of the advantage function~\eqref{eq:def_advantage}, we get~\eqref{eq:dual_equivalent}.
\end{proof}

Maximizing the dual program is also solving the Lagrangian relaxation of the convex program
\begin{talign}
  \max_{v \in \mathbb{R}^{\vert \mathcal S \vert}} & \ \rho^Tv \label{eq:rl_np_dual} \\
  \text{s.t.} & \ \max_{p \in \DA} 
  \{ \langle -c(s,\cdot) - [(\gamma P - \hat{I})v](s, \cdot), p \rangle - h^{p}(s) \} \leq 0, \ \forall s \in \cS. \nonumber
  % (I - \gamma P_a)v - c_a + u(\cdot, a) = 0, \ \forall a \in \mathcal A, \\
  % & \ u(\cdot, a) \geq 0,
\end{talign}
For the unregularized case, i.e., $h^{\pi(\cdot \vert s)}(s) = 0$, one can show the resulting optimization problem is equivalent to the (primal) LP formulation of MDPs~\cite{puterman2014markov}.
A similar duality result was shown in~\cite{neu2017unified}, but with the main difference being they replace $x \in X(\pi,\rho)$ with $x \in \Delta_{\vert \cA \vert \times \vert S \vert}$ and set the weighting vector $\eta^\pi_\rho(s) = \eta$ as a constant.

Next, we analyze a policy gradient-type method for minimizing the advantage gap function.

\section{Distribution-free convergence for PMD and strongly-polynomial runtime} \label{sec:deterministic_pmd}
Our goal is to show the basic policy mirror descent (PMD) method can achieve distribution-free convergence that matches the best rates for bounding just the aggregated optimality gap.
That is, we aim to show one can get both sublinear and linear convergence rates for $V^{\pi_t}(s) - V^{\pi^*}(s)$ at any state, rather than for the aggregated gap $\mathbb{E}_{s \sim \nu^*}[V^{\pi_t}(s) - V^{\pi^*}(s)]$.

% \begin{algorithm}[H]
\begin{algorithm}
\caption{Policy mirror descent}
\label{alg:pmd}
\begin{algorithmic}[1]
    \STATE{\textbf{Input}: $\pi_0(\cdot \vert s) \in \Delta_{\vert \cA \vert}$ and step sizes $\eta_t$}
    \FOR{$t=0,1,\ldots,$}
        \STATE{Update for all $s \in \cS$
        \begin{talign}
            % \vspace{-2em} 
            \pi_{t+1}(\cdot \vert s) 
            &= \mathrm{argmin}_{\pi'(\cdot \vert s) \in \DA} \{ \eta_t [\langle \Qts, \pi'(\cdot \vert s) \rangle + h^{\pi'(\cdot \vert s)}(s)] + D^{\pi'}_{\pi_t}(s)\} \nonumber \\
            &= \mathrm{argmin}_{\pi'(\cdot \vert s) \in \DA} \{ \eta_t \psit(s,\pi'(\cdot \vert s)) + D^{\pi'}_{\pi_t}(s)\}. \label{eq:pmd_subproblem}
        \end{talign}}
    \ENDFOR
\end{algorithmic}
\end{algorithm}

\subsection{Basic PMD method} \label{sec:basic_pmd}
We consider a basic policy mirror descent (PMD) method, as first introduced in~\cite{lan2023policy}.
% The algorithm is very simple. 
Starting from an arbitrarily policy $\pi_0$, 
we compute the state-action value function $Q^{\pi_t}(s,\cdot)$ in each iteration  (equivalently, one can compute the advantage function since the two are equivalent up to a constant additive factor at a fixed state $s$).
Then we solve the subproblem~\eqref{eq:pmd_subproblem} at every state $s$, which involves an inner product with $Q^{\pi_t}$, the regularization term $h^p$, step size $\eta_t$, and an arbitrary Bregman's distance $D^p_{\pi_t}(s)$.
This subproblem is referred to as the prox-mapping.
In some cases, a closed-form solution is known for~\eqref{eq:pmd_subproblem}. 
See~\cite{lan2023policy} for more details.
% For simplicity, we assume the initial policy $\pi_0(s)$ is the uniform distribution over actions, i.e.,
% \begin{talign} \label{eq:uniform_policy}
%     \pi_0(a \vert s) = \frac{1}{\vert \cA \vert}, \ \forall s \in \cS.
% \end{talign}

Recall the modulus of strong convexity parameter $\mu_h$ from~\eqref{eq:mu_h_defn}.
The following can be derived by the optimality conditions of~\eqref{eq:pmd_subproblem}, see for example~\cite[Lemma 3.1]{lan2022policy}.
\begin{lemma}  \label{lem:pmd_descent}
    Let $\pi_t$ be defined according to~\eqref{eq:pmd_subproblem}.
    If the step size $\eta_t$ satisfies $\mu_h + \eta_t^{-1} \geq 0$, then
    \begin{talign*}
     % &\psi^{\pi_t}(s, \pi_{t+1}(s)) + {\eta_t}^{-1}D_{\pi_t}^{\pi_{t+1}}(s)) \\
     % & + (\mu_h + {\eta_t}^{-1}) D_{\pi_{t+1}(s)}^\pi(s) 
     % \leq 
     % \psi^{\pi_t}(s, \pi(\cdot \vert s)) +  {\eta_t}^{-1} D_{\pi_t}^\pi(s), \ \forall \pi(\cdot \vert s) \in \DA, s \in \cS.
     &\langle \Qt(s, \cdot),\pi_{t+1}(\cdot \vert s) \rangle + h^{\pi_{t+1}(\cdot \vert s)}(s) + \eta_t^{-1}D_{\pi_t}^{\pi_{t+1}}(s) + (\mu_h + {\eta_t}^{-1}) D_{\pi_{t+1}}^\pi(s) \\
     &\leq
     \langle \Qt(s, \cdot),\pi(\cdot \vert s) \rangle + h^{\pi(\cdot \vert s)}(s) + {\eta_t}^{-1} D_{\pi_t}^\pi(s), \ \forall \pi(\cdot \vert s) \in \DA, s \in \cS.
    \end{talign*}
\end{lemma}

Next, monotonicity of PMD is shown.
In the sequel, a step size $\eta_t=1/0$ simply means Bregman's distance $D^a_{\pi_t}(s)$ is set to 0 in the subproblem~\eqref{eq:pmd_subproblem} at iteration $t$.
We skip the proof, which can be found in~\cite[Proposition 3.2]{lan2022policy}.
\begin{lemma} \label{lem:pmd_monotone}
    For any $\eta_t \in [0, +\infty) \cup \{1/0\}$, $V^{\pi_{t+1}}(s) - V^{\pi_{t}}(s) \leq \psi^{\pi_t}(s, \pi_{t+1}(\cdot \vert s)) \leq 0$.
\end{lemma}

Now we show that a direct application of the PMD method achieves a sublinear rate of convergence of the value function for all states.
\begin{theorem} \label{thm:sublinear_distribution_free_deterministic}
    Let $\eta_t > 0$ be a non-decreasing step size used in the PMD method. Then
    \begin{talign*}
        V^{\pi_k}(s) - V^{\pi^*}(s)
        \leq
        \frac{ \sum_{q \in \cS} \kappa^{\pi^*}_s(q) \cdot [\eta_0(V^{\pi_0}(q) - V^{\pi^*}(q)) + D^{\pi^*}_{\pi_0}(q)] }{\eta_0 (1-\gamma) k}, ~ \forall s \in \cS.
    \end{talign*}
\end{theorem}
\begin{proof}
    We have for any policy $\pi(\cdot \vert s) \in \Delta_{\vert \cA \vert}$
    \begin{talign}
        &(1-\gamma)[V^{\pi_t}(s) - V^\pi(s)] \nonumber \\
        &\stackrel{\eqref{eq:per_diff}}{=}
        \sum_{q \in \cS} \kappa^\pi_s(q) (-\psit(q, \pi(\cdot \vert q))) \nonumber \\
        &~\leftstackrel{\text{Lemma~\ref{lem:pmd_descent}}}{\leq}
        \sum_{q \in \cS} \kappa^\pi_s(q) [-\psit(q, \pi_{t+1}(\cdot \vert q)) + \eta_t^{-1}D_{\pi_t}^\pi(q) - \eta_t^{-1} D_{\pi_{t+1}}^\pi(q)] \nonumber \\
        &~\leftstackrel{\text{Lemma~\ref{lem:pmd_monotone}}}{\leq}
        \sum_{q \in \cS} \kappa^\pi_s(q) [V^{\pi_t}(q) - V^{\pi_{t+1}}(q) + \eta_t^{-1}D_{\pi_t}^\pi(q) - \eta_t^{-1} D_{\pi_{t+1}}^\pi(q)]. \nonumber %  \label{eq:pmd_refined_descent} \\
    \end{talign}
    Fixing $\pi=\pi^*$ and taking a telescopic sum from $t=0,\ldots,k-1$, we get
    \begin{talign}
        &k(1-\gamma)[V^{\pi_{k}}(s) - V^{\pi^*}(s)] \nonumber \\
        &\leftstackrel{\text{Lemma~\ref{lem:pmd_monotone}}}{\leq}
        (1-\gamma) \sum_{t=0}^{k-1} [V^{\pi_{t}}(s) - V^{\pi^*}(s)] \nonumber \\
        &\leq
        \sum_{q \in \cS} \kappa^{\pi^*}_s(q)[\sum_{t=0}^{k-1} V^{\pi_t}(q) - V^{\pi_{t+1}}(q) + \sum_{t=0}^{k-1} \eta_t^{-1}D^{\pi^*}_{\pi_t}(q) - \eta_t^{-1} D^{\pi^*}_{\pi_{t+1}}(q)] \nonumber \\
        &\leftstackrel{\substack{V^{\pi_{k}}(s) \geq V^{\pi^*}(s) \\ \text{and}~ \eta_{t+1} \geq \eta_{t}}}{\leq}
        \sum_{q \in \cS} \kappa^{\pi^*}_s(q) [V^{\pi_0}(q) - V^{\pi^*}(q) + \eta_0^{-1}D_{\pi_0}^{\pi^*}(q)
        - \eta_{k-1}^{-1} D^{\pi^*}_{\pi_k}(q)]. \label{eq:telescopic_sum}
    \end{talign}
    % which completes the proof of the first inequality.
    % The second inequality comes from the fact $\kappa^{\pi^*}_s(\cdot)$ is a distribution over states as well as the choice of Bregman divergence and the uniform distribution $\pi_0$ ensure $D_{\pi_0}^{\pi}(s) \leq \log \vert \cA \vert$~\cite[Eq 34]{lan2023policy}.
\end{proof}

This result strengthens~\cite[Theorem 2]{lan2023policy} by improving the convergence to be distribution free.
See~\eqref{eq:distribution_free} and the surrounding discussions for more details.
Note that a similar sublinear distribution-free rate was already shown by some policy gradient methods~\cite{agarwal2021theory,bhandari2024global}.

Now, by choosing a geometrically increasing step size, the averaged optimality gap $f_{\nu^*}(\pi) - f_{\nu^*}(\pi^*)$ can also decrease at a linear rate~\cite{xiao2022convergence,li2022homotopic}.
However, in the analysis it is crucial to invoke the stationarity of $\nu^*$. 
Hence, it is not straightforward to extend this convergence to be distribution-free.
In the next section, we show by using a similar increasing but slightly more involved step size schedule, one can strengthen the linear convergence to be distribution-free as well.

\subsection{Distribution-free linear convergence for PMD}
We will show by directly using PMD with a step size that increases geometrically at fixed intervals, then one can obtain linear convergence of the value function over any state.
We present two step sizes: one for general Bregman's distances, and one for bounded Bregman's distances.
Note that this result applies to general convex and strongly convex regularizers, i.e., $\mu_h \geq 0$.

\begin{theorem} \label{thm:V_linear_converence}
    Let $N := \lceil 4(1-\gamma)^{-1} \rceil$.
    By using the step size
    \begin{talign*}
        \eta_t = {4^{\lfloor t/N \rfloor} \bar{D}_0}/\Delta_{0},
    \end{talign*}
    where $\Delta_0 := (1-\gamma)^{-1} \max_{s \in \cS} g^{\pi_0}(s)$ and $\bar{D}_0 := \max_{s} D^{\pi^*}_{\pi_0}(s)$ for an arbitrary Bregman's distance, then
    \begin{talign} \label{eq:all_value_geometric_decrease}
        V^{\pi_t}(s) - V^{\pi^*}(s) \leq 2^{-\lfloor t/N \rfloor} \Delta_{0}, \ \forall s \in \cS.
    \end{talign}
\end{theorem}
\begin{proof}
    To simplify our analysis, we say epoch $i$ is the set of iterations $t=iN,iN+1,\ldots,(i+1)N-1$.
    Our proof will be by mathematical induction over epoch $i$.
    We will prove for any $s \in \cS$ and integer $i \geq 0$,
    \begin{talign} 
        V^{\pi_{iN}}(s) - V^{\pi^*}(s) 
        &\leq 
        2^{-i} \Delta_{0} \label{eq:step_value_geometric_decrease} \\
        \sum_{q\in \cS} \kappa^{\pi^*}_s(q) D^{\pi^*}_{\pi_{iN}}(q)
        &\leq
        2^{i} \bar{D}_0. \label{eq:bounded_distance_to_opt}
    \end{talign}
    In view of Lemma~\ref{lem:pmd_monotone}, then~\eqref{eq:step_value_geometric_decrease} implies~\eqref{eq:all_value_geometric_decrease}.
    
    For the base case of $i=0$,~\eqref{eq:step_value_geometric_decrease} is from Proposition~\ref{prop:gap_functions}, while~\eqref{eq:bounded_distance_to_opt} is from the assumption $D^{\pi^*}_{\pi_0}(q) \leq \bar{D}_0$ and $\kappa^{\pi^*}_s$ being a distribution over states. 
    We consider $i+1$ for some $i \geq 0$.
    Applying~\eqref{eq:telescopic_sum} over $t=iN, \ldots, (i+1)N-1$, which uses a constant step size of $\eta^{(i)} := {4^{i} \bar{D}_0}/\Delta_{0}$,
    % ({\bf GL: this term $\log \vert \cA \vert$ is different from the statement in theorem.}),
    \begin{talign}
        &N(1-\gamma)[V^{\pi_{(i+1)N}}(s) - V^{\pi^*}(s)] + \frac{1}{\eta^{(i)}} \sum_{q\in \cS} \kappa^{\pi^*}_s(q) D^{\pi^*}_{\pi_{(i+1)N}}(q) \nonumber \\
        &\leq 
        \sum_{q\in \cS} \kappa^{\pi^*}_s(q)[V^{\pi_{iN}}(q) - V^{\pi^*}(q)] + \frac{1}{\eta^{(i)}} \sum_{q\in \cS} \kappa^{\pi^*}_s(q) D^{\pi^*}_{\pi_{iN}}(q) \nonumber \\
        &\leftstackrel{\substack{\eqref{eq:step_value_geometric_decrease},~\eqref{eq:bounded_distance_to_opt}\\ \text{and}~\eta^{(i)}}}{\leq}
        % 2^{-i} \Delta_0 + \frac{\Delta_0}{2^{2i} \log \vert \cA \vert } \cdot 2^{i} \log \vert \cA \vert  \\
        % &=
        2^{-(i-1)} \Delta_0. \label{eq:recursive_geometric}
    \end{talign}
    In view of $N$ and $\eta^{(i)}$, the above clearly implies~\eqref{eq:step_value_geometric_decrease} and~\eqref{eq:bounded_distance_to_opt} for epoch $i+1$, which completes the proof by induction.
\end{proof}

This result strengthens the linear convergence from~\cite[Theorem 1]{lan2022policy} to distribution-free linear convergence, alleviating the solution quality's dependence on the unknown stationary distribution $\nu^*$ of the optimal policy.
This is the first time the value function decreases at a linear rate simultaneously at every state for policy gradient type methods.
The main innovation is to perform a larger geometric increase in step size at fixed intervals instead of a slower geometric increase every iteration~\cite[Theorem 1]{lan2022policy}.
Note that the initial bound $\bar{D}_0$ is often known when we choose $\pi_0(\cdot \vert s)$ as the uniform distribution~\cite{lan2023policy,li2023policy}.

We now present a second step size for the case where the Bregman's distance has a universal upper bound, such as when $D^{\pi'}_{\pi}(s) = \frac{1}{2}\|\pi'(\cdot \vert s) - \pi(\cdot \vert s)\|_2^2$ is the Euclidean distance squared.
This step size is more aggressive compared to Theorem~\ref{thm:V_linear_converence} since we increase the step size every iteration rather than at fixed intervals.
% The proof is skipped since it is similar to Theorem~\ref{thm:V_linear_converence} ({\bf GL: Maybe you still want to write a simple proof by pointing out steps that are different from Theorem \ref{thm:V_linear_converence}}).
\begin{theorem} \label{thm:V_linear_converence_2}
    Let $N := \lceil 4(1-\gamma)^{-1} \rceil$.
    By using the step size
    \begin{talign*}
        \eta_t = 2^t \cdot \bar{D}/\Delta_{0},
    \end{talign*}
    where $\Delta_0 := (1-\gamma)^{-1} \max_{s \in \cS} g^{\pi_0}(s)$ and $\bar{D}_0 := \max_{s} \max_{\pi,\pi' \in \Pi} D^{\pi'}_{\pi}(s)$ for an arbitrary Bregman's distance, then
    \begin{talign} \label{eq:all_value_geometric_decrease_2}
        V^{\pi_t}(s) - V^{\pi^*}(s) \leq 2^{-\lfloor t/N \rfloor} \Delta_{0}, \ \forall s \in \cS.
    \end{talign}
\end{theorem}
\begin{proof}
    We sketch the proof, since 
    it is a similar to the one for Theorem~\ref{thm:V_linear_converence}. 
    We will use mathematical induction to show~\eqref{eq:step_value_geometric_decrease} and~\eqref{eq:bounded_distance_to_opt} occur.
    In particular, we can simplify the proof by replacing~\eqref{eq:bounded_distance_to_opt} with the inequality $\sum_{q\in \cS} \kappa^{\pi^*}_s(q) D^{\pi^*}_{\pi_{iN}}(q) \leq \bar{D}$, which always holds by definition of $\bar{D}$.
    Therefore, by applying~\eqref{eq:telescopic_sum} over $t=iN, \ldots, (i+1)N-1$, which uses an increasing size of $\eta_t := 2^t \cdot \bar{D}/\Delta_0 \geq 2^{-i} \cdot \bar{D}/\Delta_0$, then applying an inequality similar to~\eqref{eq:recursive_geometric} derives for us $N(1-\gamma)[V^{\pi_{(i+1)N}}(s) - V^{\pi^*}(s)] \leq 2^{-(i-1)}\Delta_0$, which by choice in $N$, completes the proof by induction.
    % ({\bf GL: I think you will need to modify Theorem 3.3 to allow an non-decreasing $\eta_t$ rather than a constant $\eta$.})
\end{proof}

In the next subsection, we leverage the distribution-free convergence, i.e., independence of $\nu^*$, to design a strongly-polynomial time algorithm for unregularized MDPs.

\subsection{A strongly-polynomial time PMD}
Recall an algorithm is strongly-polynomial when the number of arithmetic operations is polynomial in the input size, and the memory usage is polynomial in the input data size.
% For (unregularized) MDPs, an algorithm is strongly-polynomial for a \textit{fixed} discount factor $\gamma$ if its runtime is polynomial in the size of all data from the MDP excluding the discount factor.
% \sloppy To solve unregularized MDPs in strongly-polynomial for a \textit{fixed} discount factor $\gamma$, the number of arithmetic operations must polynomial in $\vert \cS \vert$ and $\vert \cA \vert$ and the space used by the algorithm must be at most a polynomial of the MDP data size. 
We consider (unregularized) MDPs with a \textit{fixed} discount factor $\gamma$, where both the number of arithmetic operations and memory usage can have a $(1-\gamma)^{-1}$ dependence. Our result extends the work of Ye~\cite{ye2011simplex}, who showed combinatorial methods like simplex and Howard's policy iteration are strongly-polynomial for a fixed $\gamma$, to gradient methods like PMD.
Our developments are adapted from~\cite{ye2011simplex}, where the main difference is that we work in policy space and leverage the advantage gap function (Proposition~\ref{prop:gap_functions}) in lieu of strict complementary slackness.

First, we describe some structural properties of the RL problem.
Recall the weighted visitation vector $\eta^\pi_\rho$ from~\eqref{eq:weighed_visitation}.
\begin{lemma} \label{lem:visitation_bounds}
    Recall $\eta^\pi_\rho(s) \in [\rho(s), (1-\gamma)^{-1}]$.
    For any $\pi(\cdot \vert s) \in \DA$,
    the vector $x$ from Lemma~\ref{lem:pi_to_x} satisfies
    $x(s,a) = \eta^\pi_\rho(s) \pi(a \vert s) \in [0, (1-\gamma)^{-1}]$ for all states $s \in \cS$ and actions $a \in \cA$.
\end{lemma}

We say the state-action pair $(s,a)$ is \textit{non-optimal} when $\pi^*(a \vert s) = 0$, where $\pi^*$ is the optimal policy to~\eqref{eq:opt_objective}.
% In some cases, we may simply say ``$a$ is non-optimal'' to mean there exists an $s$ such that $(s,a)$ is non-optimal.
We say $\pi$ is a \textit{non-optimal policy} if $f_\rho(\pi) - f_\rho(\pi^*) > 0$, implying there is a non-optimal $(s,a)$ s.t.~$\pi(a\vert s) > 0$.
Throughout this section, we denote the (unregularized) advantage function $A^\pi : \cS \times \cA \to \mathbb{R}$ by
% \begin{talign*}
    $A^\pi(s,a) := Q^{\pi}(s,a) - V^\pi(s)$.
% \end{talign*}
Equivalently in the unregularized case, we have $A^\pi(s,a) = \psi^\pi(s,e_a)$, where $e_a \in \mathbb{R}^{\vert \cA \vert}$ is all zeros vector with a value of one at index $a$.
We also write the unregularized advantage function as the vector $A^\pi = \{A^\pi(s,a)\}_{(s,a) \in \cS \times \cA} \in \mathbb{R}^{\vert \cS \vert \cdot \vert \cA \vert}$.
The advantage gap function is $g^\pi(s) = \max_{a \in \cA} \{-A^\pi(s,a)\}$.
\begin{lemma} \label{lem:psi_lb_and_nonopt}
    Let $\rho \in \Delta_{\vert \cS \vert}$ be positive (i.e., have only positive elements). 
    For a non-optimal policy $\pi$, there exists a non-optimal $(\bar{s}, \bar{a})$ such that $\pi(\bar{a} \vert \bar{s}) > 0$ and $A^{\pi^*}(\bar{s}, \bar{a}) \geq \frac{1-\gamma}{\vert \cS \vert \vert \cA \vert}[f_\rho(\pi) - f_\rho(\pi^*)] > 0$.
\end{lemma}

\begin{proof}
    Let $x$ be the primal feasible solution w.r.t~$\pi$ as defined in Lemma~\ref{lem:pi_to_x}, 
    and recall matrices $\hat{I}$ and $P$ from the lemma.
    Let $V^{\pi^*} \in \mathbb{R}^{\vert \cS \vert}$ be the optimal value function.
    In view of the definition of $A^\pi$ and the state and state-action value function in~\eqref{eq:QV1} and~\eqref{eq:QV2}, respectively,
    then $A^{\pi'}(s,a) = [c + (\hat{I} - \gamma P)V^{\pi'}](s,a)$
    % the advantage function can be written in vector form as $\psi^p \equiv \{\psi^{p}(s,a) = c(s,a) + [(\hat{I} - \gamma P)V^{p}](s,a) \}_{s \in \cS, a \in \cA}$
    for any policy $\pi'$. 
    Now, the following linear inequalities hold,
    \begin{talign} 
        A^{\pi^*} = c - (\hat{I} - \gamma P)V^{\pi^*} 
        &\geq \mathbf{0} \nonumber \\
        x^T(\hat{I} - \gamma P) &= \rho \label{eq:psi_x_nonnegative} \\
        x &\geq \mathbf{0}, \nonumber
    \end{talign}
    where the first line is 
    % by optimality of $\pi^*$ and Proposition~\ref{prop:gap_functions}, 
    by the first inequality in Proposition~\ref{prop:gap_functions} (with $\pi=\pi^*$),
    and the last two are by Lemma~\ref{lem:pi_to_x}. 
    We also denoted $\mathbf{0}$ as the all zeros vector of appropriate dimension.
    By non-optimality of $\pi$,
    \begin{talign}
        0 
        <
        f_\rho(\pi) - f_\rho(\pi^*) 
        &\stackrel{\eqref{eq:rl_np}}{=}
        c^Tx - \rho^TV^{\pi^*} \nonumber \\
        &\stackrel{\eqref{eq:psi_x_nonnegative}}{=}
        (A^{\pi^*})^Tx \label{eq:fgap_to_xpsi} \\
        &\stackrel{\eqref{eq:psi_x_nonnegative}}{\leq}
        \vert \cS \vert \vert \cA \vert A^{\pi^*}(\hat{s}, \hat{a}) x(\hat{a}, \hat{s}) \nonumber \\
        &~\leftstackrel{\text{Lemma~\ref{lem:visitation_bounds}}}{\leq}
        (1-\gamma)^{-1} \vert \cS \vert \vert \cA \vert A^{\pi^*}(\hat{s}, \hat{a}), \nonumber
    \end{talign}
    where $(\hat{s}, \hat{a}) \in \mathrm{argmax}_{(s,a)} \{ A^{\pi^*}(s,a) x(s,a)\}$. 
    The above inequalities and Lemma~\ref{lem:visitation_bounds} imply $\pi(\hat{s}, \hat{a}) > 0$ and $A^{\pi^*}(\hat{s}, \hat{a}) > 0$. 
    % But in view of Proposition~\ref{prop:gap_functions}, a non-optimal policy will set $\pi^*(\hat{s}, \hat{a}) = 0$, otherwise 

    Let $x^*$ be the solution associated with the optimal policy $\pi^*$ defined by Lemma~\ref{lem:pi_to_x}.
    Then
    \begin{align*}
        0 
        =
        f_\rho(\pi^*) - f_\rho(\pi^*)
        &\stackrel{\eqref{eq:fgap_to_xpsi}}{=}
        (A^{\pi^*})^Tx^* 
        \stackrel{\eqref{eq:psi_x_nonnegative}}{\geq}
        A^{\pi^*}(\hat{s},\hat{a}) x^*(\hat{s},\hat{a})
        \stackrel{\eqref{eq:psi_x_nonnegative}}{\geq}
        0,
    \end{align*}
    which in view of $A^{\pi^*}(\hat{s}, \hat{a}) > 0$ ensures
    $x^*(\hat{s}, \hat{a}) = 0$. Finally, by Lemma~\ref{lem:visitation_bounds} and the assumption $\rho(s) > 0, \ \forall s\in \cS$, we find $\pi^*(\hat{a} \vert \hat{s}) = 0$, i.e., $(\bar{s}, \bar{a}) := (\hat{s}, \hat{a})$ is non-optimal.
\end{proof}

The following technical result gives us a way to upper bound the probability of selecting a non-optimal $(\bar{s}, \bar{a})$.
This lemma also highlights the importance of selecting a proper initial distribution $\rho$ to not be too small, which is only possible when one bounds the value functions for every state (rather than on average).
\begin{lemma} \label{lem:psi_ub_and_nonopt}
    Let $\rho \in \Delta_{\vert \cS \vert}$ be positive and let the non-optimal $(\bar{s}, \bar{a})$ be defined as in Lemma~\ref{lem:psi_lb_and_nonopt} w.r.t.~a non-optimal $\pi_0 \in \Pi$.
    Then for any $\pi \in \Pi$,
    \begin{talign*}
        \pi(\bar{a} \vert \bar{s}) 
        \leq 
        \frac{\vert \cS \vert \vert \cA \vert}{(1-\gamma) \rho(\bar{s})} \cdot \frac{f_\rho(\pi) - f_\rho(\pi^*)}{f_\rho(\pi_0) - f_\rho(\pi^*)}.
    \end{talign*}
\end{lemma}
\begin{proof}
    Let $x(s,a) = \eta^\pi_\rho(s) \pi(a \vert s)$ be the corresponding primal solution to $\pi$ (Lemma~\ref{lem:pi_to_x}). 
    Then
    \begin{talign*}
        f_\rho(\pi) - f_\rho(\pi^*)
        &\stackrel{\eqref{eq:fgap_to_xpsi}}{=}
        (A^{\pi^*})^Tx \\
        &\stackrel{\eqref{eq:psi_x_nonnegative}}{\geq}
        A^{\pi^*}(\bar{s}, \bar{a}) x(\bar{s}, \bar{a}) \\
        &~\leftstackrel{\substack{\text{Lemma~\ref{lem:visitation_bounds}} \\ \text{and Lemma~\ref{lem:psi_lb_and_nonopt}}}}{\geq}
        \frac{(1-\gamma) \rho(\bar{s})}{\vert \cS \vert \vert \cA \vert} [f_\rho(\pi_0) - f_\rho(\pi^*)] \pi(\bar{a} \vert \bar{s} ).
    \end{talign*}
\end{proof}

For policy $\pi$, denote the greedy policy $\hat{\pi}$ by $\hat{\pi}(\cdot \vert s) \in \mathrm{argmax}_{p \in \DA} \psi^{\pi}(s,p)$.
In the tabular setting without regularization, $\psi^{\pi}(s,p)$ is linear in $p \in \DA$, so without loss of generality we assume $\hat{\pi}(\cdot \vert s)$ is an extreme point of the probability simplex.
We let ties between extreme points be broken arbitrarily. 

\begin{proposition} \label{prop:kick_out_one}
    Let $N := \lceil 4(1-\gamma)^{-1}\rceil$ and
    $T := \lceil \log_2(\vert \cS \vert^3 \vert \cA \vert/(1-\gamma)^2) \rceil+1$.
    Suppose policies $\{\pi_t\}_{t = 0}^{NT}$ are generated by PMD described in Theorem~\ref{thm:V_linear_converence_2}, and $\{\pi_t\}_{t \geq NT}$ are policies generated by PMD with any nonnegative step size, where the two sets share policy $\pi_{NT}$. 
    If $\pi_0$ is non-optimal, then there exists a non-optimal $(\bar{s}, \bar{a})$ \fixed{such that $\pi_0(\bar a \vert \bar s) > 0$} and
    $\hat{\pi}_{\tau}(\bar{a} \vert \bar{s}) = 0$ for any integer $\tau \geq TN$.
\end{proposition}
\begin{proof}
    We have
    \begin{talign*}
        f_\rho(\hat{\pi}_{\tau}) - f_\rho(\pi^*)
        &\stackrel{\text{Lemma~\ref{lem:fpi_to_fx}}}{=}
        \mathbb{E}_{s \sim \rho} [V^{\hat{\pi}_{\tau}}(s) - V^{\pi^*}(s) ] \\
        &~~~~\leftstackrel{\substack{\text{Lemma~\ref{lem:pmd_monotone}}\\ \text{with $\eta_{\tau}=1/0$}}}{\leq}
        \mathbb{E}_{s \sim \rho} [V^{\pi_{\tau}}(s) - V^{\pi^*}(s)] \\
        &~~~~\leftstackrel{\text{Lemma~\ref{lem:pmd_monotone}}}{\leq}
        \mathbb{E}_{s \sim \rho} [V^{\pi_{NT}}(s) - V^{\pi^*}(s) ] \\
        &~~~~\leftstackrel{\text{Theorem~\ref{thm:V_linear_converence_2}}}{\leq}
        2^{-T} \Delta_0,~~ \forall t \geq 0,
    \end{talign*}
    where $\Delta_0 = (1-\gamma)^{-1} \max_{s \in \cS} g^{\pi_0}(s)$ is from Theorem~\ref{thm:V_linear_converence_2}.
    In view of Proposition~\ref{prop:gap_functions}, there exists a state $s'$ where $V^{\pi_0}(s') - V^{\pi^*}(s') \geq (1-\gamma) \Delta_0$.
    Then by optimality of $\pi^*$, we get by fixing $\rho = \vert \cS \vert^{-1} \mathbf{1}_{\vert \cS \vert}$
    \begin{talign*}
        f_\rho(\pi_{0}) - f_\rho(\pi^*)
        &\stackrel{\text{Lemma~\ref{lem:fpi_to_fx}}}{=}
        \sum_{s \in \cS} \rho(s) \cdot (V^{\pi_{0}}(s) - V^{\pi^*}(s)) 
        \geq
        (1-\gamma) \vert \cS \vert^{-1} \Delta_0.
    \end{talign*}
    Since $\pi_0$ is assumed to be non-optimal, let $(\bar{s}, \bar{a})$ be non-optimal s.t.~${\pi}_0(\bar{a} \vert \bar{s}) > 0$ as described in Lemma~\ref{lem:psi_lb_and_nonopt}.
    Putting the above two bounds together,
    \begin{talign*}
        \hat{\pi}_{\tau}(\bar{a} \vert \bar{s})
        &\stackrel{\text{Lemma~\ref{lem:psi_ub_and_nonopt}}}{\leq}
        \frac{\vert \cS \vert \vert \cA \vert}{(1-\gamma) \rho(\bar{s})}  
        \cdot
        \frac{f_\rho(\hat{\pi}_{\tau}) - f_\rho(\pi^*)}{f_\rho(\pi_{0}) - f_\rho(\pi^*)} 
        \leq
        \frac{\vert \cS \vert^3 \vert \cA \vert \cdot 2^{-T}}{(1-\gamma)^2} 
        \stackrel{\text{Choice in $T$}}{<}
        1.
    \end{talign*}
    Since $\hat{\pi}_{\tau}(\cdot \vert \bar{s})$ is an extreme point of the probability simplex, it must be $\hat{\pi}_{\tau}(\bar{a} \vert \bar{s}) = 0$.
\end{proof}

The proposition guarantees after a certain number of iterations, at least one non-optimal action will never be selected by the greedy policy.
To remove all non-optimal actions, we repeat the argument multiple times, leading to the following iteration complexity that is polynomial in only the number of states and actions for any fixed discount $\gamma$.
But before doing so, we first fix Bregman's distance to the Euclidean distance squared, $D(\cdot,\cdot) = \frac{1}{2}\|\cdot - \cdot\|_2^2$.
It turns out this choice has several important consequences when choosing the step size and solving the subproblem~\eqref{eq:pmd_subproblem} that make the algorithm strongly-polynomial and also efficient in practice.
Further discussions for this choice will take place after Corollary~\ref{cor:strongly_polynomial}.
\begin{theorem} \label{thm:strongly_poly_euclidean}
    Fix Bregman's distance to the Euclidean distance squared.
    Let $N := \lceil 4(1-\gamma)^{-1} \rceil$ and $T := \lceil \log_2 (\vert \cS \vert^3 \vert \cA \vert/(1-\gamma)^2)\rceil + 1$.
    By using the step size
    \begin{talign*}
        % \eta_t = 2^{t+1}/\Delta_{NT\lfloor t/(NT) \rfloor},
        \fixed{
        \eta_t 
        =
        \begin{cases}
            2^{t+1}/\Delta_{(NT+1)\lfloor t/(NT+1) \rfloor} & : ~ (t+1) ~\mathrm{mod}~(NT+1) \ne 0 \\
            +\infty &: ~\text{otherwise}
        \end{cases},
        }
    \end{talign*}
    where $\Delta_{t} := (1-\gamma)^{-1} \max_{s \in \cS} g^{\pi_{t}}(s)$,
    then for any iteration $\tau \geq \vert \cS \vert(\vert \cA \vert-1)(NT+1)$, the greedy policy $\hat{\pi}_\tau$ is optimal. % , $\hat{\pi}_{\tau} = \pi^*$.
\end{theorem}
\begin{proof}
    % For notational convenience, define $\Delta_{\ell NT} := (1-\gamma)^{-1} \max_{s,a} \{-\psi^{\pi_{\ell NT}}(s,a)\}$ for an integer $\ell \geq 0$.
    Recall epoch $i \geq 0$ consists of iterations $iN,iN+1,\ldots,(i+1)N-1$ (see the proof for Theorem~\ref{thm:V_linear_converence}). 
    Similarly, we say round $\ell \geq 0$ consists of epochs $i=\ell T, 1 + \ell T, \ldots, (\ell+1)T$. 
    Our goal is to show within a round $\ell$, we observe a linear decrease in the objective relative to the optimality gap of the first policy, 
    i.e., for any round $\ell \geq 0$ and any epoch $i =0,\ldots,T-1$ within the round,
    \begin{talign} 
        V^{\pi_{iN + \ell \cdot (NT+1)}}(s) - V^{\pi^*}(s) 
        &\leq 
        2^{-i} \Delta_{\ell \cdot (NT+1)} \label{eq:step_value_geometric_decrease_3}. %  \\
        % \sum_{q\in \cS} \kappa^{\pi^*}_s(q) D^{\pi^*}_{\pi_{iN + \ell NT}}(q)
        % &\leq
        % 2. \label{eq:bounded_distance_to_opt_3}
    \end{talign}
    By choice in the step size during round $\ell$ and $\bar{D} = 2$ (where $\max_{s \in \cS, \pi,\pi' \in \Pi} D^{\pi}_{\pi'}(s) \leq \bar{D}$ by choice of Euclidean norm), then one can use an argument similar to the proof for Theorem~\ref{thm:V_linear_converence_2} to establish~\eqref{eq:step_value_geometric_decrease_3}.
    % Similar to~\eqref{eq:step_value_geometric_decrease} and~\eqref{eq:bounded_distance_to_opt},
    % we will prove for any $s \in \cS$ and 
    % Clearly,~\eqref{eq:bounded_distance_to_opt_3} is true since $\kappa_s^{\pi^*}$ is a distribution over states and choice of the Euclidean distance squared implies $D^{\pi'}_{\pi}(s) \leq 2$ for any policies $\pi(\cdot \vert s), \pi'(\cdot \vert s) \in \DA$.
    % Showing~\eqref{eq:step_value_geometric_decrease_3} is nearly identical to showing~\eqref{eq:step_value_geometric_decrease}, so we skip it.
    % Showing~\eqref{eq:step_value_geometric_decrease_3} is similar to~\eqref{eq:step_value_geometric_decrease}, so we skip it.

    To complete the proof, we will apply Proposition~\ref{prop:kick_out_one} to every round $\ell$ (the proof needs to be slightly modified, where we invoke~\eqref{eq:step_value_geometric_decrease_3} in place of Theorem~\ref{thm:V_linear_converence_2} to show geometric decrease of the optimality gap w.r.t.~$\Delta_{\ell \cdot (NT+1)}$).
    That is, if the starting policy in a round is non-optimal, then Proposition~\ref{prop:kick_out_one} ensures that \fixed{applying the greedy step size $\eta_t = +\infty$ after the round removes (i.e.,~updates the probability of selecting the action to be 0) at least one non-optimal state-action pair that was not previously removed}.
    There are at most $\vert \cS \vert (\vert \cA \vert -1)$ non-optimal state-action pairs to remove, and each round \fixed{and greedy step} together involve $NT+1$ iterations, which finishes the proof.
\end{proof}

Theorem~\ref{thm:strongly_poly_euclidean} implies the iteration complexity to find the optimal solution is only polynomial in $\vert \cS \vert$ and $\vert \cA \vert$ if we assume $\gamma$ is a fixed constant.
This result is new for policy gradient methods.
\fixed{Note that the greedy step size $\eta_t = +\infty$ corresponds to one iteration of policy iteration. The policy iteration step helps ensure the iteration complexity is independent of the so-called \textit{gap value}, $\underline{A}^* := \min_{(s,a) : A^{\pi^*}(s,a) > 0} A^{\pi^*}(s,a)$, which can be arbitrarily small~\cite{li2022homotopic}.
However, if we allow the iteration complexity to depend logarithmically on $1/\underline{A}^*$, then one can derive an alternative iteration complexity that is (nearly) dimension-independent, with only logarithmic dependence on $\vert \cS \vert$ and $\vert \cA \vert$. The resulting iteration complexity can possibly be smaller than Theorem~\ref{thm:strongly_poly_euclidean}'s when the state and action spaces are large. We refer to~\cite{li2022homotopic} for more details on the gap value.}
% \fixed{The key innovation is the new step size. While the step size follows a similar doubling step size as Theorem~\ref{thm:V_linear_converence_2}, it has two important new elements. First, the step size periodically updates the gap function $\Delta_t$. This enhancement ensures the iteration complexity does not depend on the gap function, which can be arbitrarily small. Second, the step size periodically becomes $\eta_t = +\infty$, which corresponds to a single policy iteration update. This ensures non-optimal state-action pairs are removed. Finally, it is possible derive another iteration complexity that is (nearly) dimension-independent, with only logarithmic dependence on $\vert \cS \vert$ and $\vert \cA \vert$. However, the resulting iteration complexity will also have a logarithmic dependence on $1/\underline{A}^*$, where $\underline{A}^* := \min_{(s,a) : A^{\pi^*}(s,a) > 0} A^{\pi^*}(s,a)$ is the so-called \textit{gap value}~\cite{li2022homotopic}. See~\cite{li2022homotopic} for details on the convergence rate involving~$\underline{A}^*$. This problem-dependent constant depends on the data in the cost $c$ and transition kernel $\cP$, and so it can be arbitrarily small. On the other hand, if $\underline{A}^*$ is not too small, then the resulting dimension-independent iteration complexity can be significantly smaller than Theorem~\ref{thm:strongly_poly_euclidean} to solve MDPs with large state and action spaces.
% }

We will now show PMD is a strongly-polynomial algorithm for a fixed $\gamma$, as advertised in the beginning of this section.
Let the five-tuple $\mathcal{M} = (\cS, \cA, \cP, c, \gamma)$ define the MDP, and suppose $\mathcal{M}$ is rational.
Recall the length of a rational number $p/q$ is $\lceil \log(p+1) \rceil + \lceil \log(q+1) \rceil + 1$, and the size of the variable is the sum of the length of all its data.
Let $L := L(\mathcal{M})$ be the size of the MDP.
We write ``$x$ is $\mathrm{poly}(L)$'' to mean the variable $x$ has size that is a polynomial of $L$.
In view of Theorem~\ref{thm:strongly_poly_euclidean}, we just need to show subproblem~\eqref{eq:pmd_subproblem} can be solved in strongly-polynomial time. If $\pi_t$ is rational and is $\mathrm{poly}(L)$, then so is the advantage function $A^{\pi_t}(s,a) =  Q^{\pi_t}(s,a) - V^{\pi_t}(s)$. 
This is because the value function, when viewed as the vector $V^{\pi_t} \in \mathbb{R}^{\vert \cS \vert}$, is the solution to a linear system defined by $\mathcal{M}$ and $\pi_t$~\cite[Theorem 6.1.1]{puterman2014markov}, and solving a rational linear system can be done in strongly-polynomial time~\cite[Theorem 3.3]{schrijver1998theory}. 
% then $V^{\pi}$ is rational and is $\mathrm{poly}(L)$.
In view of~\eqref{eq:QV2}, then $Q^{\pi_t}$ and $A^{\pi_t}$ are rational and $\mathrm{poly}(L)$ as well.
% It remains to show the subproblem~\eqref{eq:pmd_subproblem} can be solved in strongly-polynomial time. 
% at time $t$ when $\pi_t(\cdot \vert s)$ is rational and $\mathrm{poly}(L)$.
By choice of the Euclidean distance squared,
% and lack of a regularization term, 
the subproblem~\eqref{eq:pmd_subproblem} is equivalent to the projection problem,
\begin{talign*}
    \min_{p \in \DA} \|p - [\pi_t(\cdot \vert s) - \eta_t \cdot \Qt(s, \cdot)] \|_2, \ \forall s \in \cS.
\end{talign*}
\sloppy When the previous advantage functions are rational and $\mathrm{poly}(L)$, then the gap function $\Delta_{(NT+1) \lfloor t/(NT+1) \rfloor}$ from Theorem~\ref{thm:strongly_poly_euclidean} (which is defined with $A^{\pi_{(NT+1)\lfloor t/(NT+1) \rfloor}}$) is also rational and $\mathrm{poly}(L)$. Furthermore, Proposition~\ref{prop:gap_functions} ensures $\Delta_{(NT+1) \lfloor t/(NT+1) \rfloor}$ is positive whenever $\pi_{(NT+1)\lfloor t/(NT+1) \rfloor}$ is non-optimal. Based on these two observations, the step size $\eta_t$ from Theorem~\ref{thm:strongly_poly_euclidean} is rational and $\mathrm{poly}(L)$ \fixed{whenever $\eta_t < +\infty$} (\fixed{if $\eta_t=+\infty$, then subproblem~\eqref{eq:pmd_subproblem} corresponds to finding the largest element in $Q^{\pi_t}(s,\cdot)$}). 
So when $Q^{\pi_t}(s, \cdot)$ and $\pi_t(\cdot \vert s)$ are rational and $\mathrm{poly}(L)$, then so are the inputs to the projection problem.
Since projection onto the simplex can be done in strongly-polynomial time~\cite{chen2011projection}, 
this ensures $\pi_{t+1}(\cdot \vert s)$ is rational and $\mathrm{poly}(L)$.
Thus, starting with the uniform distribution of $\pi_0(a \vert s) = \vert \cA \vert^{-1}$,
% or \fixed{a greedy policy}, 
a simple successive argument implies PMD solves MDPs in strongly-polynomial time.
We have come to the following conclusion.

\begin{corollary} \label{cor:strongly_polynomial}
    Suppose we are given an unregularized MDP problem with rational data and fixed discount factor $\gamma$. By using PMD as described in Theorem~\ref{thm:strongly_poly_euclidean} and setting $\pi_0(\cdot \vert s) \in \DA$ as the uniform distribution for all $s \in \cS$, then PMD runs in strongly-polynomial time.
\end{corollary}

We make some remarks about using the Euclidean distance squared for Bregman's distance.
If we instead used the KL-divergence, which does not have a fixed upper bound, it can grow exponentially within PMD, requiring the step size $\eta_t$ to grow exponentially, as large as $2^{O(\vert \cS \vert \vert \cA \vert)}$ (see~\eqref{eq:bounded_distance_to_opt}). 
As the solution to~\eqref{eq:pmd_subproblem} under the KL-divergence requires computation of $\exp\{\eta_t \Qt(s,\cdot)\}$~\cite{lan2023policy}, the large step size $\eta_t$ will incur memory that is not $\mathrm{poly}(L)$.
Another difficulty is that exponentials are irrational.
% \edits{Although other distance generating functions, such as the centered $\ell_p$ norm, have bounded domain and can use the step size from  Theorem~\ref{thm:V_linear_converence_2}, solving PMD's subproblem with this distance generating function may not result in a strongly-polynomial time algorithm~\cite{ilandarideva2024accelerated}.
% For example, the aforementioned shifted $\ell_p$ norm and its gradient involve powers of irrationals, and hence the Bregman distance induced by the shifted $\ell_p$ norm will be irrational as well. 
% The Euclidean norm squared, on the other hand,  results in a projection problem that can solved efficiently.
% It seems the choice of Euclidean distance squared is essential to attaining strongly-polynomial runtime.}
Similarly, showing strongly-polynomial runtime for solving MDPs with general regularizations may be impossible, because even solving the subproblem with negative entropy regularization involves taking an exponential. 
Still, the linear distribution-free convergence from Theorem~\ref{thm:V_linear_converence} yields a polynomial-time algorithm for the original problem~\eqref{eq:opt_objective} in the sense of nonlinear programming, where the arithmetic cost per digit of accuracy is $\mathrm{poly}(L)$~\cite{nesterov1994interior}.

This ends our tour of value convergence of PMD in the deterministic setting, i.e., when the advantage function can be computed exactly.
In the next section, we consider the more realistic stochastic setting where one can only estimate the advantage function.

\section{Distribution-free convergence for stochastic PMD} \label{sec:improved_sublinear_convergence}
We assume throughout that given a policy $\pi_t$, we are given an estimator $\tQt$ generated by random vectors $\xi_t$ instead of the true advantage function $\Qt$.
Our goal is to show the basic policy mirror descent (PMD) method also can achieve distribution-free convergence when only given $\tQt$.

We make the following assumption regarding the underlying noise.
It covers independent and identically distributed (iid) random data and bounded stochastic estimates, as well as non-iid with time-dependent noise (e.g.,~Markovian noise~\cite{kotsalis2022simple}) and noise with bounded moments.
This latter setup is more common in reinforcement learning and stochastic optimal control, where data is generated along a single trajectory and subject to some (possibly Gaussian) noise~\cite{ju2022model}.
\begin{assumption} \label{asmp:expectation}
We have $\max_{p \in \DA}\|p\| \leq 1$, and
there exists $\varsigma, \sigma, \bar{Q} \geq 0$ satisfying
\begin{talign}
    \|\mathbb{E}_{\xi_{t} \vert \xi_{[t-1]}} [\tQt] - \Qt\|_* &\leq \varsigma \label{eq:bias} \\
    \mathbb{E}_{\xi_{t} \vert \xi_{[t-1]}} \|\tQt - \Qt\|_*^2 &\leq \sigma^2 \label{eq:variance}\\
    \mathbb{E}_{\xi_t \vert \xi_{[t-1]}} \|\tQt\|_*^2 &\leq \bar{Q}^2 \label{eq:second_moment}.
\end{talign}
\end{assumption}
The assumption on the norm is to simplify results, and it clearly holds for all $\ell_p$ norms.
The assumption~\eqref{eq:bias} bounds the bias, while~\eqref{eq:variance} bounds the variance.
%Clearly, when $\sigma = 0$, then we have exact information about $\Qt$.
% Our results are meaningful in the regime where the bias is small and the variance is bounded.
In the finite state and action MDP, numerous works have developed efficient methods to satisfy these assumptions~\cite{li2023accelerated,li2023policy,lan2023policy,kotsalis2022simple}.
% Therefore, these assumptions can often be satisfied. % , but with a slight caveat.

% \begin{remark}
%     In various policy evaluation algorithms~\cite{li2023accelerated,li2023policy,lan2023policy,kotsalis2022simple}, the sampling complexity depends on problem-dependent parameters or distributions, e.g., the mixing rate or stationary distribution of intermediate polices $\pi_t$. That is, even satisfying Assumption~\ref{asmp:expectation} will have sampling complexity that depends on problem-dependent parameters. 
%     However, this dependence seems fundamental to policy evaluation, as it is related to the issue of exploration in reinforcement learning.
%     Moreover, our intention is to show that policy optimization -- not necessarily policy evaluation -- can be done without dependence on such parameters.
% \end{remark}

We consider another set of assumptions that are crucial to obtain high-probability results.
In contrast to assumptions~\eqref{eq:variance} and~\eqref{eq:second_moment}, which rely on the second moments, this next set of assumptions bound the moment generating function, as previously appeared in stochastic optimization~\cite{nemirovski2009robust,lan2012validation}.
\begin{assumption} \label{asmp:high_prob}
% The norm satisfies $\max_{p \in \DA}\|p\| \leq 1$.
We have $\max_{p \in \DA}\|p\| \leq 1$, and
there exists $\varsigma, \sigma, \bar{Q} \geq 0$ satisfying~\eqref{eq:bias} and
\begin{talign}
    \mathbb{E}_{\xi_{t} \vert \xi_{[t-1]}} \mathrm{exp} \{ {\|\tQt - \Qt\|_*^2}/{\sigma^2} \} &\leq 2 \label{eq:variance_mgf}\\
    \mathbb{E}_{\xi_t \vert \xi_{[t-1]}} \mathrm{exp} \{ {\|\tQt\|_*^2}/{\bar{Q}^2} \} &\leq 2.
\end{talign}
\end{assumption}
Clearly, all these assumptions are satisfied whenever both $\Qt$ and $\tQt$ are bounded almost surely.
Equipped with these assumptions, we examine the convergence properties of the stochastic PMD.

\subsection{Basic stochastic policy mirror descent} \label{sec:basic_spmd}
Stochastic policy mirror descent (SPMD) is the same as PMD (Algorithm~\ref{alg:pmd}) except the exact Q-function in~\eqref{eq:pmd_subproblem} is replaced with a stochastic one, i.e., the update becomes
\begin{talign}
        \pi_{t+1}(\cdot \vert s)
        &= \mathrm{argmin}_{\pi'(\cdot \vert s) \in \DA} \{ \eta_t [\langle \tQts, \pi'(\cdot \vert s) \rangle + h^{\pi'(\cdot \vert s)}(s)] + D^{\pi'}_{\pi_t}(s)\} \nonumber. % \\
        % &= \mathrm{argmin}_{\pi'(\cdot \vert s) \in \DA} \{ \eta_t\tpsit(s,\pi'(\cdot \vert s)) + D^{\pi'}_{\pi_t}(s)\}, \ \forall s\in \cS, \label{eq:spmd_subproblem}
\end{talign}
% \sloppy where for any $p \in \DA$ and $s \in \cS$, the stochastic advantage function is defined as $\tpsit(s,p) = \langle \tQts, p \rangle - \langle \tQts, \pi_t(\cdot \vert s)\rangle + h^{p}(s) -h^{\pi_t(\cdot \vert s)}(s)$.
% {\bf GL: $Q$ should be $\tilde Q$ in the above relation in the above relations.}
\edits{Note that the results for the remainder of the paper hold for arbitrary Bregman's distance, such as KL-divergence}.
We start with a descent lemma under noise. We skip the proof since similar results can be found in prior works like~\cite[Proposition 2]{lan2022policy} and~\cite[Lemma 13]{lan2023policy}.
\begin{lemma} \label{lem:pmd_value_descent}
    Suppose the regularization $h$ is $M_h$-Lipschitz continuous, or 
    \begin{talign} \label{eq:M_h_continuous}
        h^{\pi(\cdot \vert s)}(s) - h^{\pi'(\cdot \vert s)}(s)
        \leq
        M_h\|\pi(\cdot \vert s) - \pi'(\cdot \vert s)\|, \ \forall s \in \cS, \ \forall \pi,\pi' \in \Pi.
    \end{talign}
    Then for any fixed $\pi$,
    \begin{talign*}
        &(1-\gamma)[V^{\pi_t}(s) - V^\pi(s)]
        \\
        &\leq
        \mathbb{E}_{q \sim \kappa^\pi_s}[-\eta_t^{-1}(1+\eta_t \mu_h) D^\pi_{\pi_{t+1}}(q) + \eta_t^{-1} D^\pi_{\pi_t}(q) + \eta_t\|\tQtq\|_*^2 + \zeta_t(q, \pi)] + \eta_t M_h^2, \ \forall s \in \cS
    \end{talign*}
    where $\zeta_t(q, \pi) := \langle \Qtq -\tQtq, \pitq - \piq \rangle$.
\end{lemma}

The following technical result will be useful to derive high probability bounds.
We defer the proof to Appendix~\ref{sec:proofs_for_improved_sublinear_convergence}.
\begin{lemma} \label{lem:light_tail_azuma_hoeffding}
    Fix an integer $N \geq 1$.
    Let $\xi_1, \xi_2, \ldots$ be a sequence of random variables, $\sigma_t > 0$, $t = 1,\ldots,$
    be a sequence of deterministic numbers and $\phi_t = \phi_t (\xi_t)$ be deterministic (measurable) functions of $\xi_{[t]} = (\xi_1, \ldots , \xi_t )$ such that either of two cases takes place:
    \begin{enumerate}
        \item $\mathbb{E}_{\vert \xi_{[t-1]}} \phi_t \leq \sigma_t/N$ w.p.~1 and $\mathbb{E}_{\vert \xi_{[t-1]}}[\exp\{\phi_t^2/\sigma_t^2\}] \leq \exp\{1\}$ w.p.~1 for all $t$, or
        \item $\mathbb{E}_{\vert \xi_{[t-1]}}[\{\vert \phi_t\vert/\sigma_t\}] \leq \exp\{1\}$ w.p.~1 for all $t$.
    \end{enumerate}
    Then for any $\Omega \geq 0$, we have for case 1:
    \begin{talign*}
        \textstyle \mathrm{Pr} \big\{ \sum_{t=1}^N \phi_t > \Omega \sqrt{\sum_{t=1}^N \sigma_t^2} \big\}
        \leq
        \exp\{-\Omega^2/3 + 1 \}.
    \end{talign*}
    For case 2 with $\sigma^N := (\sigma_1, \ldots, \sigma_N)$:
    \begin{talign*}
        \textstyle \mathrm{Pr} \big\{ \sum_{t=1}^N \phi_t > \|\sigma^N\|_1 + \Omega \|\sigma^N\|_2 \big\}
        \leq
        \exp\{-\Omega^2/12 \} + \exp\{-3\Omega/4\}.
    \end{talign*}
\end{lemma}

We are ready to establish the main convergence result of the value function at every state.
We first consider the case for general convex regularizers, i.e. $\mu_h \geq 0$.
\begin{theorem} \label{thm:agg_convergence}
    Suppose Assumption~\ref{asmp:expectation} and~\eqref{eq:M_h_continuous} take place.
    With \edits{$\eta_t = \sqrt{\frac{\bar{D}_0}{(\bar{Q}^2 + \sigma^2 + M_h^2)k}}$}, then
    \begin{talign*}
        {\textstyle k^{-1} \sum_{t=0}^{k-1} \mathbb{E}[V^{\pi_t}(s) - V^{\pi^*}(s)]}
        \leq
        \frac{2\sqrt{\bar{D}_0 \cdot (\bar{Q}^2 + \sigma^2 + M_h^2)} + 2\varsigma \sqrt{2k}}{(1-\gamma) \sqrt{k}}, \ \forall s \in \cS,
    \end{talign*}
    where $\bar{D}_0:= \max_{s,\pi} D^{\pi}_{\pi_0}(s)$ for an arbitrary Bregman's distance.
    %% This is useful only for last-iterate, let's prove it later:
    % vvvvvvvvvvvvv
    % and for any $0 \leq k_0 \leq k-1$, 
    % \begin{talign*}
    %     (1-\gamma)\sum_{t=k_0}^{k-1} \eta_t \mathbb{E}[V^{\pi_t}(s) - V^{\pi_{k_0}}(s)]
    %     \leq
    %     2\alpha^2 (\bar{Q}^2 + M_h^2) \cdot \log\big(\frac{k}{k_0+1}\big)  + 8\alpha \varsigma \sqrt{k-k_0}.
    % \end{talign*}
    % ^^^^^^^^^^^^
    Suppose instead Assumption~\ref{asmp:high_prob} and~\eqref{eq:M_h_continuous} occur and $\varsigma \leq \sigma/k$.
    Then for any $\delta \in (0,1]$,
    \begin{talign*}
        &\mathrm{Pr} \big\{ 
        \exists s \in \cS : k^{-1} \sum_{t=0}^{k-1} [V^{\pi_t}(s) - V^{\pi^*}(s)] 
        >
        \frac{15\sqrt{\bar{D}_0 \cdot (\bar{Q}^2 + \sigma^2 + M_h^2)}\log(4 \vert \cS \vert/\delta) + 1}{(1-\gamma) \sqrt{k}}
        \big\}
        \leq
        \delta.
    \end{talign*}
\end{theorem}
\begin{proof}
    To simplify our derivations, we denote the step size as $\eta_t = \frac{\alpha}{\sqrt{k}}$, where $\alpha = \sqrt{\frac{\bar{D}_0}{\bar{Q}^2 + \sigma^2 + M_h^2}}$.
    
    Recall $\zeta_t(q,\pi)$ from Lemma~\ref{lem:pmd_value_descent}.
    First, observe $\pi_t(\cdot \vert q) \in \DA$ is a deterministic function when conditioned on $\xi_{[t-1]}$. 
    Therefore, for any $t =0,\ldots,k-1$ 
    \begin{talign} 
        \mathbb{E}\zeta_t(q,\pi^*) 
        &= 
        \mathbb{E}_{\xi_{[t-1]}}[ \langle \mathbb{E}[\Qtq - \tQtq \vert \xi_{[t-1]}], \pitq - \pi^*(\cdot \vert q) \rangle] \nonumber \\
        & \leq
        \mathbb{E}_{\xi_{[t-1]}}\|\mathbb{E}[\Qtq - \tQtq \vert \xi_{[t-1]}]\|_* \|\pitq - \pi^*(\cdot \vert q)\| \nonumber  \\
        &\leftstackrel{\eqref{eq:bias}}{\leq}
        \varsigma \mathbb{E}_{\xi_{[t-1]}} D_{\|\cdot\|, [0,t]},
        \label{eq:bias_bound}
    \end{talign}
    where the second line used the Cauchy-Schwarz inequality and in the third line we define the adaptive diameter $D_{\|\cdot\|, [0,t]} := \max_{\tau=0,\ldots,t} \max_{q \in \cS} \|\pi_\tau(\cdot \vert q)-\pi^*(\cdot \vert q)\| \leq 2 \max_{p \in \DA} \|p\| \leq 2$.
    From these observations, we can deduce
    \begin{talign} 
        &(1-\gamma)\sum_{t=0}^{k-1} \eta_t \mathbb{E}[V^{\pi_t}(s) - V^{\pi^*}(s)] \nonumber \\
        &~~~\leftstackrel{\text{Lemma~\ref{lem:pmd_value_descent}}}{\leq}
        \mathbb{E}_{q \sim \kappa^{\pi^*}_s} \big[ D^{\pi^*}_{\pi_0}(q) + \sum_{t=0}^{k-1} \eta_t^2(\mathbb{E}\|\tQtq\|_*^2 + M_h^2) + \sum_{t=0}^{k-1} \mathbb{E} \eta_t \zeta_t(q,\pi^*) \big]  \label{eq:agg_convergence_general} \\
        &~~~\leftstackrel{\substack{\text{Choice of $\eta_t$,} \\~\eqref{eq:bias_bound},~\eqref{eq:second_moment}}}{\leq}
        \bar{D}_0 + \alpha^2 (\bar{Q}^2 + M_h^2) + 2\alpha \varsigma \sqrt{2k}.\nonumber 
    \end{talign}
    Dividing the above by $(1-\gamma) \sum_{t=0}^{k-1} \eta_t = (1-\gamma) \alpha \sqrt{k}$, recalling $\sigma^2 \geq 0$, and substituting our choice of $\alpha$ completes the bound in expectation.
    
    % The second result can be shown similarly, so we omit it.
    %% Second result is below!
    %% vvvvvvvvvvvvvvvvv
    % Similarly,
    % \begin{talign*}
    %     &(1-\gamma)\sum_{t=k_0}^{k-1} \eta_t \mathbb{E}[V^{\pi_t}(s) - V^{\pi_{k_0}}(s)] \\
    %     &~~~\leftstackrel{\text{Lemma~\ref{lem:pmd_value_descent}}}{\leq}
    %     \mathbb{E} \mathbb{E}_{q \sim d^\pi_s}[D^{\pi_{k_0}}_{\pi_{k_0}}(q)] + (\bar{Q}^2 + M_h^2) \sum_{t=k_0}^{k-1} \eta_t^2 + 2\varsigma \sum_{t=k_0}^{k-1} \eta_t \\
    %     &~~~\leftstackrel{\text{Choice of $\eta_t$}}{\leq}
    %     (\bar{Q}^2 + M_h^2) \cdot \alpha^2 2\log(k/(k_0+1))  + 2\varsigma \cdot 4\sqrt{k-k_0}.
    % \end{talign*}
    % ^^^^^^^^^^^^^^^^^^^^^

    % For the third result, 
    For the second result,
    we need a high probability bound on the terms in expectation within~\eqref{eq:agg_convergence_general}.
    Recall the error term $\zeta_t(q,\pi)$ from Lemma~\ref{lem:pmd_value_descent}.
    To bound the sum over the terms $\eta_t \zeta_t(q,\pi^*)$ in~\eqref{eq:agg_convergence_general}, we first utilize Lemma~\ref{lem:light_tail_azuma_hoeffding} with $\phi_t := \eta_t \zeta_t(q,\pi^*)$ and $\sigma_t := D_{\|\cdot\|, [0,k-1]} \eta_t \sigma$ (the assumptions for Case 1 are satisfied since we assumed~\eqref{eq:bias}, $\varsigma \leq \sigma/k$, and~\eqref{eq:variance_mgf}, and we also recall~\eqref{eq:bias_bound}) and union bound over all states  to show with probability $1-\delta/2$,
    \begin{talign}
        \sum_{t=0}^{k-1} \eta_t \zeta_t(q, \pi^*) 
        &\leq
        \textstyle \sigma D_{\|\cdot\|, [0,k-1]} \sqrt{3 \log(4 \vert \cS \vert/{\delta}) \sum_{t=0}^{k-1} \eta_t^2} \label{eq:high_tail_ah_bias} \\
        &\leftstackrel{\eta_t = \alpha/\sqrt{k}}{\leq}
        2\alpha\sigma \sqrt{3 \log(4 \vert \cS \vert/{\delta})} \nonumber \\
        &\leq
        1 + 3\alpha^2 \sigma^2 \log(4 \vert \cS \vert/\delta),  \ \forall q \in \cS, \nonumber
    \end{talign}
    where the last line is by the inequality $2ab \leq a^2 + b^2$.

    To bound the sum over the $\eta_t^2 \|\tQt(q, \cdot)\|_*^2$ terms in~\eqref{eq:agg_convergence_general},
    we again apply Lemma~\ref{lem:light_tail_azuma_hoeffding} with $\phi_t := \eta_t^2 \|\tQt(s, \cdot)\|_*^2$ and $\sigma_t := \eta_t^2 \bar{Q}^2$ (the assumptions for Case 2 are satisfied since we assumed~\eqref{eq:variance_mgf}) and again union bound over all states to show with probability $1-\delta/2$,
    \begin{talign}
        \sum_{t=0}^{k-1} \eta_t^2 \|\tQt(q, \cdot)\|^2_*
        &\leq
        \textstyle \bar{Q}^2 \sum_{t=0}^{k-1} \eta_t^2 + 12 \bar{Q}^2 \log ({4 \vert \cS \vert}/{\delta}) \sqrt{\sum_{t=0}^{k-1} \eta_t^4} \label{eq:high_tail_ah_second_momement} \\
        &\leq
        \alpha^2 \bar{Q}^2 + 12\alpha^2 \bar{Q}^2 \log ({4 \vert \cS \vert}/{\delta})/\sqrt{k}, \ \forall q \in \cS. \nonumber
    \end{talign}
    Plugging the resulting bounds back into~\eqref{eq:agg_convergence_general}, dividing by $(1-\gamma) \sum_{t=0}^{k-1} \eta_t = (1-\gamma) \alpha \sqrt{k}$, recalling $\sigma^2 \geq 0$, and then substituting in our choice of $\alpha$ finishes the proof.
\end{proof}

Similar to Theorem~\ref{thm:sublinear_distribution_free_deterministic}, this result extends~\cite[Theorem 3.6]{lan2022policy} to be distribution-free in the stochastic setting.
This result also appears to be the first time distribution-free convergence under Assumption~\ref{asmp:expectation}.
In contrast, prior works like the variance-reduced Q-value iteration~\cite{sidford2018near} show a similar distribution-free convergence with better dependence on $\gamma$, but they require a stronger oracle, where one can generate iid samples of the transition dynamic $\cP$ at any state-action pair.
Assumption~\ref{asmp:expectation} only requires the bias of the estimator to be small and have bounded second moments, which can be done without the stronger oracle using, for example, Monte-Carlo sampling along a single trajectory~\cite{li2023policy}.
Finally, we note the upper bound $\bar{D}_0$ is often known when the initial policy $\pi_0$ is the uniform distribution over actions at every state.
For example, when Bregman's distance is the KL-divergence, then $\bar{D}_0 = \log \vert \cA \vert$~\cite{lan2023policy}.
If Bregman's distance is induced by the negative Tsallis entropy with an entropic-index $p \in (0,1)$, then $\bar{D}_0 = -1+\vert \cA \vert^{1-p}$~\cite{li2023policy}.

We now consider strongly convex regularizations, i.e., $\mu_h > 0$.
The proof is similar to when $\mu_h \geq 0$, except the bias is handled more carefully in the high-probability regime by showing the distance to optimality is decreasing.
Crucially, this avoids the need to shrink the feasible region~\cite{lan2020first}, which permits the use of the basic (i.e., without any modification) PMD.
Due to the technical aspect of the proof, we defer it to Appendix~\ref{sec:proofs_for_improved_sublinear_convergence}.
\begin{theorem} \label{thm:strongly_convex_agg_convergence}
    Suppose Assumption~\ref{asmp:high_prob} and~\eqref{eq:M_h_continuous} take place.
    When $\eta_t = \frac{1}{\mu_h(t+1)}$, then
    \begin{talign*}
        \textstyle k^{-1}\sum_{t=0}^{k-1} \mathbb{E}[V^{\pi_t}(s) - V^{\pi^*}(s)] + \frac{\mu_h}{1-\gamma} \mathbb{E}_{q \sim \kappa^{\pi^*}_s} \mathbb{E}[D^{\pi^*}_{\pi_k}(q)] 
        &\leq
        \frac{\mu_h \bar{D}_0 + \mu_h^{-1} (\bar{Q}^2 + M_h^2) \log(2k)  + 2 \varsigma k}{(1-\gamma)k}, \ \forall s \in \cS,
    \end{talign*}
    %% This is useful only for last-iterate, let's prove it later:
    % vvvvvvvvvvvvv
    % and for any $0 \leq k_0 \leq k-1$, 
    % \begin{talign*}
    %     (1-\gamma)\sum_{t=k_0}^{k-1} \mathbb{E}[V^{\pi_t}(s) - V^{\pi_{k_0}}(s)]
    %     \leq
    %     2\alpha^2 (\bar{Q}^2 + M_h^2) \cdot \log\big(\frac{k}{k_0+1}\big)  + 2\alpha \varsigma (k-k_0).
    % \end{talign*}
    % ^^^^^^^^^^^^
    where $\bar{D}_0:= \max_{s,\pi} D^{\pi}_{\pi_0}(s)$ for an arbitrary Bregman's distance.
    Suppose instead Assumption~\ref{asmp:high_prob} and~\eqref{eq:M_h_continuous} occur and $\varsigma \leq \sigma/k$.
    Then for any $\delta \in (0,1]$,
    \begin{talign*}
        &\mathrm{Pr} \big\{ 
        \exists s \in \cS : k^{-1}\sum_{t=0}^{k-1} [V^{\pi_t}(s) - V^{\pi^*}(s)]  + \frac{\mu_h}{1-\gamma} \mathbb{E}_{q \sim \kappa^{\pi^*}_s} [D^{\pi^*}_{\pi_k}(q)] \\
        &\hspace{60pt} >
        \frac{\mu_h \bar{D}_0 + [25\mu_h^{-1} (\bar{Q}^2 + M_h^2) + 2\sigma \sqrt{3C(k)}](\log(4k\vert \cS \vert/\delta))^{3/2}}{(1-\gamma) k} 
        \big\}
        \leq (k+1)\delta,
    \end{talign*}
    where
    % \begin{talign*}
        $ C(k) 
        := 
        \frac{6\bar{D}_0}{1-\gamma} 
        + 
        \frac{75(\bar{Q}^2 + M_h^2)(\log(4k\vert \cS \vert/\delta))^{3/2}}{(1-\gamma)\mu_h^2}
        +
        \frac{108\sigma^2 (\log(4k \vert \cS \vert/\delta))^3}{(1-\gamma)^2\mu_h^2} = O\{\log(k/\delta)^3\}$.
         % \label{eq:epoch_length}
    % \end{talign*}
\end{theorem}
This result seems to be the first distribution-free $O(k^{-1})$ convergence rate for MDPs with strongly convex regularization.

% In the general convex (resp.~strongly convex) regime, we show the average optimality gap decreases at a $O(k^{-1/2})$ (resp.~$\tilde{O}(k^{-1})$, where $\tilde{O}$ hides polylogarithmic terms).
To summarize, this section shows distribution-free convergence, which implies the expected advantage gap function (Proposition~\ref{prop:gap_functions_2}) is small.
But, since this expected value is not known exactly, it is unclear how to use it as a termination criterion.
In the next section, we provide an efficient way to provide accurate estimates of the expected advantage gap function.

\section{Validation analysis and last-iterate convergence of SPMD} \label{sec:validation_analysis}
The main goal of this section is to show one can develop computationally and statistically efficient ways to obtain accuracy estimates of a policy generated by stochastic policy mirror descent.
We call this the \textit{validation step}.
We provide two approaches: one with no additional samples (i.e., the online estimate) and another with a sampling complexity similar to computing the policy itself (i.e., the offline estimate).
This work is based upon~\cite{lan2012validation}, but it extends the results to RL, where an important difference is that RL is nonconvex in policy space~\cite{agarwal2021theory}.
% The basis of our development stems from the optimality gaps developed Propositions~\ref{prop:gap_functions} and~\ref{prop:gap_functions_2}.

\subsection{Online accuracy certificates} \label{sec:online_accuracy}
We consider the following aggregate upper bound of $V^{\pi^*}(s)$ and advantage gap function, respectively, at any time step $k \geq 0$:
\begin{talign}
    \textstyle {V^{*}}^k(s) &:= k^{-1} \sum_{t=0}^{k-1} V^{\pi_{t}}(s) \nonumber \\
    % {V_{*}}_1^k(s) &:= \frac{1}{\bar{\eta}_k} \Big[ \sum_{t=0}^{k-1} \eta_t V^{\pi_t}(s) 
    % - \frac{\log \vert \cA \vert + (\bar{Q}^2 + M_h^2) \sum_{t=0}^{k-1} \eta_t^2 + 2\varsigma \sum_{t=0}^{k-1} \eta_t }{1-\gamma} \Big], \\
    G^k(s) &:= 
    k^{-1} g^{\pi_{[k]}}(s)  
    \stackrel{\eqref{eq:agg_gap_function_definition}}{=}
    k^{-1} 
    \max_{p \in \DA} \{ -\sum_{t=0}^{k-1} \psi^{\pi_t}(s,p) \}. \label{eq:exact_gap}
\end{talign}
According to Proposition~\ref{prop:gap_functions_2}, a lower bound for the optimal value at any state $s \in \cS$ can be constructed by ${V^*}^k(s) - (1-\gamma)^{-1}\max_{s' \in \cS} G^k(s') \leq V^{\pi^*}(s)$.

In practice, we cannot measure these quantities since they rely on knowing the exact value function and advantage function. 
So we instead measure their noisy, computable counterparts,
\begin{talign}
    \overline{V}^k(s) &:= k^{-1} \sum_{t=0}^{k-1} \tilde{V}^{\pi_{t}}(s) \label{eq:noisy_upper_bound} \\
    \tilde{G}^k(s) &:= 
    k^{-1} \max_{p \in \DA} \{ -\sum_{t=0}^{k-1} \tilde{\psi}^{\pi_t}(s,p) \}, \label{eq:noisy_gap} 
\end{talign}
where $\tilde{V}^{\pi_t}(s) = \langle \tilde{Q}^{\pi_t}(s,\cdot), {\pi_t}(\cdot \vert s) \rangle$ 
and the noisy advantage function is defined as~\eqref{eq:def_advantage} but with the noisy $\tilde{V}^{\pi_t}$ and $\tQt$.
Our goal is to show the stochastic estimates converge at an $O(k^{-1/2})$ rate towards their exact counterpart.
% This matches the rate obtained for the general convex regularization, i.e., $\mu_h \ge 0$, which is why we only focus on this case.
\begin{theorem} \label{thm:validation_analysis}
    Suppose Assumption~\ref{asmp:expectation} and~\eqref{eq:M_h_continuous} take place.
    With $\eta_t = \sqrt{\frac{\bar{D}_0}{(\bar{Q}^2 + \sigma^2 + M_h^2)k}}$, then
    \begin{talign}
        \mathbb{E}[{V^{*}}^k(s) - V^{\pi^*}(s)]
        &\leq
        (1-\gamma)^{-1} \max_{q \in \cS} \mathbb{E}[G^k(q)]  
        % \frac{1}{\bar{\eta}_k(1-\gamma)} \mathbb{E}[\Delta_1^k(s)]
        \leq
        \frac{4\sqrt{\bar{D}_0 \cdot (\bar{Q}^2 + \sigma^2 + M_h^2)} + 4\sqrt{\sigma^2 + k\varsigma^2}}{(1-\gamma) \sqrt{k}}, \ \forall s \in \cS \label{eq:adv_convergence} \\
        \mathbb{E}\vert {V^{*}}^k(s) - \bar{V}^k(s) \vert 
        &\leq 
        \sqrt{\frac{\sigma^2 + k\varsigma^2}{k}}, \ \forall s \in \cS \nonumber \\
        \mathbb{E}\vert G^k(s) -  \tilde{G}^k(s)\vert 
        &\leq 
        \frac{4\sqrt{\bar{D}_0 \cdot (\bar{Q}^2 + \sigma^2 + M_h^2)} + 4\sqrt{\sigma^2 + k\varsigma^2}}{\sqrt{k}}, \ \forall s \in \cS,\nonumber
    \end{talign}
    where $\bar{D}_0:= \max_{s,\pi} D^{\pi}_{\pi_0}(s)$ for an arbitrary Bregman's distance.
    Suppose instead Assumption~\ref{asmp:high_prob} occurs and $\varsigma \leq \sigma/k$.
    Then for any $\delta \in (0,1]$,
    \begin{talign*}
        &\textstyle \mathrm{Pr}\big \{ {V^{*}}^k(s) - V^{\pi^*}(s) \leq (1-\gamma)^{-1} \max_q G^k(q), \ \forall s \in \cS \big\} = 1  \\
        &\textstyle \mathrm{Pr}\big\{ (1-\gamma)^{-1} \max_{q \in \cS} G^k(q) 
        > 
        \frac{28\sqrt{\bar{D}_0 \cdot (\bar{Q}^2 + \sigma^2 + M_h^2)}\log ({16 \vert \cS \vert}/{\delta}) + 4\sigma\sqrt{3\log(8 \vert \cS \vert/\delta)}}{(1-\gamma)\sqrt{k}}
        \big\}
        \leq \delta \\
        & \textstyle \mathrm{Pr} \big \{ 
        \exists s \in \cS : \vert {V^*}^k(s) - \bar{V}^k(s) \vert 
        >
        \sigma \sqrt{\frac{3 \log(2 \vert \cS \vert/{\delta})}{k}}
        \big\} \leq \delta  \\
        &
        \textstyle \mathrm{Pr} \big\{
        \exists s \in \cS : 
        \vert G^k(s) -  \tilde{G}^k(s)\vert
        >
        \frac{28\sqrt{\bar{D}_0 \cdot (\bar{Q}^2 + \sigma^2 + M_h^2)}\log ({16 \vert \cS \vert}/{\delta}) + 4\sigma\sqrt{3\log(8 \vert \cS \vert/\delta)}}{\sqrt{k}} 
        \big\}
        \leq \delta.
    \end{talign*}
\end{theorem}
\begin{proof}
    Similar to the proof of Theorem~\ref{thm:agg_convergence}, we rewrite $\eta_t = \frac{\alpha}{\sqrt{k}}$, where $\alpha = \sqrt{\frac{\bar{D}_0}{\bar{Q}^2 + \sigma^2 + M_h^2}}$.
    We start with the results in expectation.
    For the first two inequalities in~\eqref{eq:adv_convergence}, the former is by Proposition~\ref{prop:gap_functions_2}. We will prove the latter later in this proof since it requires introducing an auxiliary sequence~\eqref{eq:auxiliary_seq_defn}.

    For the third inequality, we first notice $\pi_t(\cdot \vert s)$ is a deterministic function when conditioned on $\xi_t$ (which recall are the random vectors used to form $\tQt$). 
    Then in view of Assumption~\ref{asmp:expectation},
    \begin{talign}
        \mathbb{E}_{\xi_t}[\tilde{V}^{\pi_t}(s) - V^{\pi_t}(s)]
        &=
        \mathbb{E}_{\xi_{[t-1]}} [\langle \mathbb{E}_{\xi_t \vert \xi_{[t-1]}}[\tQts - \Qts], \pits \rangle] \nonumber  \\
        &\leq
        \mathbb{E}_{\xi_{[t-1]}} \|\mathbb{E}_{\xi_t \vert \xi_{[t-1]}}[\tQts - \Qts]\|_* \|\pits\|
        \leq
        \varsigma,\label{eq:value_bias}
    \end{talign}
    where we used the Cauchy-Schwarz inequality.
    Similarly, one can show
    \begin{talign*}
        &\mathbb{E}(\tilde{V}^{\pi_t}(s) - V^{\pi_t}(s))^2
        \\
        &=
        \mathbb{E}(\langle \tQts - \Qt(s, \cdot), \pits \rangle)^2 
        % \nonumber \\
        =
        \mathbb{E} [\|\tQts - \Qts\|_*^2 \|\pits\|^2 ]
        \leq
        \sigma^2, 
    %\label{eq:value_second_moment} 
    \end{talign*}
    where the inequality is in part by $\|\pi_t(\cdot \vert s)\| \leq 1$.
    Consequently,
    \begin{talign}
        \mathbb{E}( {V^{*}}^k(s) - \bar{V}^k(s) )^2 \label{eq:llog_decrease} 
        &=
        k^{-2} \big[ 
        \sum_{t=0}^{k-1} \mathbb{E}(\tilde{V}^{\pi_t}(s) - V^{\pi_t}(s))^2 \\
        &\hspace{10pt}+ 
        2\sum_{t=0}^{k-1} \sum_{t'=0}^{t-1} \mathbb{E}(\tilde{V}^{\pi_t}(s) - V^{\pi_t}(s)) (\tilde{V}^{\pi_{t'}}(s) - V^{\pi_{t'}}(s)) 
        \big] \nonumber \\
        &\leq
        k^{-2} \big[ \sum_{t=0}^{k-1} \sigma^2 + 2\sum_{t=0}^{k-1} \sum_{t'=0}^{t-1} \varsigma^2 \big] 
        \leq
        \frac{\sigma^2 + k \varsigma^2}{k}. \nonumber
    \end{talign}
    Using $\mathbb{E} \vert Y\vert \leq \sqrt{\mathbb{E}Y^2}$ for any random variable $Y$ finishes the result.

    For the last inequality in expectation, we first define the stochastic Q-function error 
    % \[
        $\delta_t(s,a) := \Qt(s,a) - \tQt(s,a)$,
    % \]
    and auxiliary sequences $\{u_t\}$ and $\{v_t\}$ with $u_0(\cdot \vert s) = v_0(\cdot \vert s) = \pi_0(\cdot \vert s)$ and
    \begin{talign}
        u_{t+1}(\cdot \vert s)  &= \mathrm{argmin}_{\pi(\cdot \vert s) \in \DA} \{ \eta_t \langle \delta_t(s, \cdot), \pi(\cdot \vert s) \rangle + D_{u_t}^\pi(s) \}, \ \forall s \in \cS \label{eq:auxiliary_seq_defn} \\
        v_{t+1}(\cdot \vert s) &= \mathrm{argmin}_{\pi(\cdot \vert s) \in \DA} \{ \eta_t \langle -\delta_t(s, \cdot), \pi(\cdot \vert s) \rangle + D_{v_t}^\pi(s) \}, \ \forall s \in \cS. \nonumber
    \end{talign}
    By construction, both $u_t$ and $v_t$ only depend on the random variables $\{\xi_\tau\}_{\tau=0}^{t-1}$, where the random vector $\xi_\tau$ helps construct the stochastic estimate $\tilde{Q}^{\pi_{t}}$ (and hence determines the estimation error $\delta_\tau$). 
    And by recalling $D_{u_0}^p(s) \leq \bar{D}_0$ and $\eta_t = \alpha k^{-1/2}$, we have for any $\pi(\cdot \vert s) \in \DA$ and $s \in \cS$,
    \begin{talign}
        % &\alpha k^{-1/2} \sum_{t=0}^{k-1} \langle \delta_t(s), p_{t}(s) - a \rangle \nonumber \\ 
        % &\leftstackrel{\eta_t = \alpha/\sqrt{k}}{=}
        & k^{-1} \sum_{t=0}^{k-1} \langle \delta_t(s, \cdot), u_{t}(\cdot \vert s) - \pi(\cdot \vert s) \rangle \nonumber \\
        &=
        k^{-1} [ \sum_{t=0}^{k-1} \langle \delta_t(s, \cdot), u_{t+1}(\cdot \vert s) - \pi(\cdot \vert s) \rangle + \langle \delta_t(s, \cdot), u_{t}(\cdot \vert s) - u_{t+1}(\cdot \vert s) \rangle] \nonumber \\
        &~~~~\leftstackrel{\text{Lemma~\ref{lem:pmd_descent}}}{\leq}
        ({\alpha \sqrt{k}})^{-1} \big[  D^{\pi}_{u_0}(s) + \sum_{t=1}^{k-1} [D^\pi_{u_t}(s) - D^\pi_{u_t}(s)] - D^\pi_{u_k}(s) \big] \nonumber \\
        &\hspace{30pt}+
        ({\alpha \sqrt{k}})^{-1} \big[  \sum_{t=0}^{k-1} -D_{u_t}^{u_{t+1}}(s) + \eta_t \langle \delta_t(s, \cdot), u_t(\cdot \vert s) - u_{t+1}(\cdot \vert s) \rangle \big] \nonumber \\
        &~~~~\leq
        \frac{\alpha^{-1}\bar{D}_0}{\sqrt{k}} + ({\alpha \sqrt{k}})^{-1}  \sum_{t=0}^{k-1} \eta_t^2 \|\delta_t \|_*^2]  \label{eq:auxiliary_seq_bound}. % \\
        % &~~~~\leftstackrel{\text{Assumption~\ref{asmp:expectation}}}{\leq}
        % \frac{\alpha^{-1} \log \vert \cA \vert + \alpha \sigma^2}{\sqrt{k}}. \label{eq:auxiliary_seq_bound}
    \end{talign}
    where the last line used strong convexity of $D^{u_{t+1}}_{u_t}(s)$, as well as the Cauchy-Schwarz inequality and Young's inequality to bound the inner product by $\eta_t^2\|\delta_t(s, \cdot)\|_*^2/2 + \|u_t(\cdot \vert s) - u_{t+1}(\cdot \vert s)\|_1^2/2$ and $\|\delta_t(s,\cdot)\|_* \leq \|\delta_t\|_*$.
    One can construct a similar upper bound but with $u_t$ replaced with $v_t$ and $\delta_t$ replaced by its negative. 
    Now, recalling $G^k(s)$ and $\tilde{G}^k(s)$ from~\eqref{eq:exact_gap} and~\eqref{eq:noisy_gap}, respectively, as well as action-value error $\delta_t(s,a) = \Qt(s, \cdot) - \tQts$, we have
    \begin{talign}
        & \mathbb{E}\vert G^k(s) - \tilde{G}^k(s) \vert 
        % &=
        % k^{-1} \mathbb{E} \big \vert 
        % \max_{p \in \Pi} \sum_{t=0}^{k-1} \langle \Qt(s, \cdot), \pits - p(\cdot \vert s) \rangle 
        % - \max_{p' \in \Pi} \sum_{t=0}^{k-1} \langle \tQts, \pits - p'(\cdot \vert s) \rangle \big \vert \nonumber \\
        \leq
        k^{-1} \mathbb{E} \max_{\pi(\cdot \vert s) \in \DA} \big \vert \sum_{t=0}^{k-1} \langle \delta_t(s, \cdot), \pi_t(\cdot \vert s) - \pi(\cdot \vert s) \rangle \big \vert \nonumber \\
        &=
        k^{-1} \mathbb{E} \max_{\pi(\cdot \vert s) \in \DA} \max \big \{ \sum_{t=0}^{k-1} \langle {\delta_t(s, \cdot)}, \pits - \pi(\cdot \vert s) \rangle,  \sum_{t=0}^{k-1} \langle {-\delta_t(s, \cdot)}, \pits - \pi(\cdot \vert s) \rangle \big \} \nonumber \\
        &\leq
        k^{-1} \mathbb{E} \max \big \{ \sum_{t=0}^{k-1} \langle {\delta_t(s, \cdot)}, \pits - u_t(\cdot \vert s) \rangle,  \sum_{t=0}^{k-1} \langle {-\delta_t(s, \cdot)}, \pits - v_t(\cdot \vert s) \rangle \big \} \nonumber \\ 
        &\hspace{10pt} + 
        k^{-1} \mathbb{E}\max_{\pi(\cdot \vert s) \in \DA} \max \big \{ \sum_{t=0}^{k-1} \langle {\delta_t(s, \cdot)}, u_t(\cdot \vert s) - \pi(\cdot \vert s) \rangle,  \sum_{t=0}^{k-1} \langle {-\delta_t(s, \cdot)}, v_t(\cdot \vert s) - \pi(\cdot \vert s) \rangle \big \} \nonumber \\
        &\leftstackrel{\eqref{eq:auxiliary_seq_bound}}{\leq}
        \mathbb{E} \big \vert k^{-1} \sum_{t=0}^{k-1} \langle {\delta_t(s, \cdot)}, \pits - u_t(\cdot \vert s) \rangle \big \vert
        +
        \mathbb{E} \big \vert k^{-1} \sum_{t=0}^{k-1} \langle {-\delta_t(s, \cdot)}, \pits - v_t(\cdot \vert s) \rangle \big \vert \nonumber \\
        &\hspace{20pt}+
        \frac{\alpha^{-1}\bar{D}_0}{\sqrt{k}} + ({\alpha \sqrt{k}})^{-1}  
        \mathbb{E} \big[\sum_{t=0}^{k-1} \eta_t^2 \|\delta_t \|_*^2 \big]
         \label{eq:G_deviation_penulatimate} \\
        &\leftstackrel{\text{Assumption~\ref{asmp:expectation}}}{\leq}
        \frac{4\sqrt{\sigma^2 + k\varsigma^2}}{\sqrt{k}} + \frac{\alpha^{-1} \bar{D}_0 + \alpha \sigma^2}{\sqrt{k}} 
        \stackrel{\substack{\text{Choice of $\alpha$}\\ \text{and $\bar{Q},M_h \geq 0$}}}{\leq}
        \frac{4\sqrt{\sigma^2 + k\varsigma^2}}{\sqrt{k}} + \frac{2\sqrt{\bar{D}_0 \cdot (\bar{Q}^2 + \sigma^2 + M_h^2)}}{\sqrt{k}}. \nonumber
    \end{talign}
    \sloppy Here, the penultimate inequality follows by $\mathbb{E} \vert Y \vert \leq \sqrt{\mathbb{E}Y^2}$ (for any random variable $Y$) and 
    \begin{talign} \label{eq:weighted_deltas}
        \big ( k^{-1} \sum_{t=0}^{k-1}  \langle {\delta_t(s, \cdot)}, \pits - u_t(\cdot \vert s) \rangle \big )^2 
        \leq
        4(\sigma^2 + k\varsigma^2)/k,
    \end{talign}
    which can be shown similarly to~\eqref{eq:llog_decrease}, 
    and one can similarly derive~\eqref{eq:weighted_deltas} but with $\delta_t$ and $u_t$ replaced by $-\delta_t$ and $v_t$, respectively. 
    We now return to the second inequality in~\eqref{eq:adv_convergence}. By applying Lemma~\ref{lem:pmd_value_descent} combined with Lemma~\ref{lem:pmd_descent}, we can bound the advantage gap as follows (for any $s \in \cS$ and $\pi \in \Pi$),
    \begin{talign*}
        -\psi^{\pi_t}(s, \pi(\cdot \vert s))
        \leq 
        \eta_t^{-1}[D^\pi_{\pi_t}(s) - D^\pi_{\pi_{t+1}}(s) - D^{\pi_{t+1}}_{\pi_t}(s)] + \eta_t(\bar{Q}^2 + M_h^2) + \zeta_t(s,\pi),
    \end{talign*}
    where $\zeta_t(s,\pi) = \langle \delta_t(s, \cdot), \pi_t(\cdot \vert s) - \pi(\cdot \vert s) \rangle$ is from Lemma~\ref{lem:pmd_value_descent}. Taking a telescopic sum and a max over all $\pi \in \Pi$ and then applying expectations, we derive for any $s \in \cS$,
    \begin{talign}
        &\mathbb{E}[k^{-1}\max_{\pi \in \Pi} \sum_{t=0}^{k-1} -\psi^{\pi_t}(s, \pi(\cdot \vert s))]
        \nonumber \\
        &\leq
        k^{-1}[\eta_0^{-1}\max_{\pi \in \Pi} D^\pi_{\pi_0}(s) + \sum_{t=0}^{k-1} \eta_t (\bar{Q}^2 + M_h^2)] + k^{-1} \mathbb{E}[\max_{\pi \in \Pi}\sum_{t=0}^{k-1} \zeta_t(s,\pi)]
        \nonumber
        \\
        &\leftstackrel{\text{\eqref{eq:agg_convergence_general} and~\eqref{eq:G_deviation_penulatimate}}}{\leq}
        \frac{2\sqrt{\bar{D}_0 \cdot (\bar{Q}^2 + \sigma^2 + M_h^2)}}{\sqrt{k}} + \frac{4\sqrt{\sigma^2 + k\varsigma^2}}{\sqrt{k}} + \frac{2\sqrt{\bar{D}_0 \cdot (\bar{Q}^2 + \sigma^2 + M_h^2)}}{\sqrt{k}}.
        \label{eq:adv_convergence_closed}
    \end{talign}
    In view of $G^k$ from~\eqref{eq:exact_gap}, we have the second inequality in~\eqref{eq:adv_convergence} after simplifying terms above.

    Now we move onto proving the high probability bounds. 
    The first inequality can be proven similarly to the bound in expectation. 
    For the third inequality, using an argument similar to~\eqref{eq:high_tail_ah_bias} (which requires Assumption~\ref{asmp:high_prob}) then delivers with probability $1-\delta$,
    \begin{talign}
        \vert {V^*}^k(s) - \bar{V}^k(s) \vert
        &=
        k^{-1} \big \vert \sum_{t=0}^{k-1} \langle \Qts - \tQts, \pits \rangle \big \vert \nonumber \\
        % &\leftstackrel{\eqref{eq:high_tail_ah_bias}}{\leq}
        &\leq
        k^{-1} \sigma \sqrt{3k \log(2 \vert \cS \vert/{\delta})}. \label{eq:llog_decrease_high_prob}
    \end{talign}

    For the second and last inequalities in high probability, we follow closely to the corresponding bound in expectation. 
    We use the auxiliary sequence $\{u_t\}$, $\{v_t\}$ from~\eqref{eq:auxiliary_seq_defn}.
    Assumption~\ref{asmp:high_prob} affirms that with probability $1-\delta/4$,
    \begin{talign*}
        k^{-1} \sum_{t=0}^{k-1} \langle \delta_t(s,\cdot), \pits - \pi(\cdot \vert s) \rangle
        &\leftstackrel{\eqref{eq:auxiliary_seq_bound}}{\leq}
        (\alpha \sqrt{k})^{-1} \big( \bar{D}_0 + \sum_{t=0}^{k-1} \eta_t^2 \|\delta_t(s,\cdot)\|_*^2 \big) \\
        &\leftstackrel{\eqref{eq:high_tail_ah_second_momement}}{\leq}
        \frac{\alpha^{-1} \bar{D}_0 + \alpha \sigma^2 + 12\alpha \sigma^2 \log ({16 \vert \cS \vert}/{\delta})}{\sqrt{k}}
        \\
        &\leftstackrel{\substack{\text{Choice of $\alpha$}\\ \text{and $\bar{Q},M_h \geq 0$}}}{\leq}
        \frac{14\sqrt{\bar{D}_0 \cdot (\bar{Q}^2 + \sigma^2 + M_h^2)} \log ({16 \vert \cS \vert}/{\delta})}{\sqrt{k}}
    \end{talign*}
    and the same bound holds with the sequence for $-\delta_t$ with probability $1-\delta/4$.
    Similar to~\eqref{eq:llog_decrease_high_prob}, we also have 
    \begin{talign*}
        \mathbb{E} \big \vert k^{-1} \sum_{t=0}^{k-1} \langle {\delta_t(s, \cdot)}, \pits - u_t(\cdot \vert s) \rangle \big \vert \leq 2k^{-1/2} \sigma \sqrt{3\log(8 \vert \cS \vert/\delta)}, 
    \end{talign*}
    with probability $1-\delta/4$, and the same bound holds for the sequence involving $-\delta_t$ and $v_t$.
    Plugging these bounds back into~\eqref{eq:G_deviation_penulatimate} and~\eqref{eq:adv_convergence_closed} (without the expectation) derives the wanted bounds.
\end{proof}

The theorem shows that as the number of iterates $k$ grows, the observable upper bound $\bar{V}^k$ and observable gap $\tilde{G}^k$ approach, in a probabilistic sense, to their exact but unobservable counterparts.
We note that the computable optimality gap $(1-\gamma)^{-1}\max_{q \in \cS} \mathbb{E}[G^k(q)]$ decreases as $O((1-\gamma)^{-1}k^{-1/2})$ in the worst-case, which matches the worst-case rate for the true optimality gap from Theorem~\ref{thm:agg_convergence}.
% while the uncomputable worst-case optimality gap fromis bounded by $O((1-\gamma)^{-1}k^{-1/2})$.
% That is, the computable optimality gap decreases at the same worst-case rate as the true optimality gap.
% That is, the main disadvantage is that the computable optimality gap can be $O((1-\gamma)^{-1})$ times larger than the true optimality gap.

Although the computable optimality gap ${V^*}^k(s) - (1-\gamma)^{-1}\max_{s' \in \cS} G^k(s')$ provides a lower bound of $V^{\pi^*}(s)$ at any state $s \in \cS$,
the max operator makes the evaluation of the advantage gap function the same at every state, i.e., it is not adaptive to the state $s$, which can be over-conservative.
Below, we provide one possible alternative lower bound.
Although this may not be a valid lower bound of $V^{\pi^*}(s)$ for every state $s$, it is close to a lower bound on average taken w.r.t.~the states.
In the following, we denote $[\cdot]_+ := \max\{0, \cdot\}$ 
and recall $f_\rho(\pi) := \mathbb{E}_{s \sim \rho}V^{\pi}(s)$ and $\kappa^\pi_\rho$ from~\eqref{eq:visitation_measure}.
\begin{corollary} \label{cor:adaptive_lb}
    For any distribution $\rho$ over states, we have
    \begin{talign} \label{eq:adapative_lb}
        \mathbb{E}_{s \sim \rho}[{V^*}^k(s) - (1-\gamma)^{-1}[G^k(s)]_+]
        \leq
        f_\rho(\pi^*) + (1-\gamma)^{-1}\varepsilon_k(\rho),
    \end{talign}
    where $\kappa^\pi_\rho(\cdot) := \sum_{s \in \cS} \kappa^\pi_s(\cdot) \rho(s)$ and $\varepsilon_k(\rho) := \mathbb{E}_{s \sim \kappa^{\pi^*}_\rho} [{G}^k(s)]- \mathbb{E}_{s \sim \rho}[[{G}^k(s)]_+]$ is the overestimation error. 
    Moreover, under the same assumptions as Theorem~\ref{thm:agg_convergence} for the result in expectation, then
    \begin{talign*}
        \mathbb{E} \varepsilon_k(\rho)
        \leq
        \frac{4\sqrt{\bar{D}_0 \cdot (\bar{Q}^2 + \sigma^2 + M_h^2)} + 4\varsigma \sqrt{2k}}{(1-\gamma) \sqrt{k}}.
    \end{talign*}
\end{corollary}
\begin{proof}
    In view of~\eqref{eq:V_gap_ub_agap} and $G^k$ from~\eqref{eq:exact_gap},  we have
    \begin{talign*}
        k^{-1}\sum_{t=0}^{k-1}[f_\rho(\pi_t) - f_\rho(\pi^*)] 
        \leq 
        (1-\gamma)^{-1} \mathbb{E}_{q \sim \kappa^{\pi^*}_\rho}[G^k(q)],
    \end{talign*}
    which implies~\eqref{eq:adapative_lb}.
    Finally for the last inequality, we use the bound
    \begin{talign*}
        \mathbb{E}G^k(s) 
        % &\stackrel{\eqref{eq:exact_gap}}{=}
        % k^{-1} \mathbb{E}\max_{p \in \Pi} \{\sum_{t=0}^{k-1} -\psi^{\pi_t}(s,p) \} \\
        % &~\leq
        &\stackrel{\eqref{eq:exact_gap}}{\leq}
        k^{-1} \mathbb{E}\sum_{t=0}^{k-1} \max_{p \in \DA} \{-\psi^{\pi_t}(s,p) \} \\
        &~\leftstackrel{\text{Proposition~\ref{prop:gap_functions}}}{\leq}
        k^{-1} \mathbb{E}\sum_{t=0}^{k-1} [V^{\pi_t}(s) - V^{\pi^*}(s)] 
        \stackrel{\text{Theorem~\ref{thm:agg_convergence}}}{\leq}
        \frac{2\sqrt{\bar{D}_0 \cdot (\bar{Q}^2 + \sigma^2 + M_h^2)} + 2\varsigma \sqrt{2k}}{(1-\gamma) \sqrt{k}}.
        % \frac{\alpha^{-1} \bar{D}_0 + \alpha (\bar{Q}^2 + M_h^2) + 2\varsigma \sqrt{2k}}{(1-\gamma) \sqrt{k}}.
    \end{talign*}
    Since the above holds for any state $s \in \cS$, then the same upper bound holds when taking expectation w.r.t.~$s \sim \kappa^{\pi^*}_\rho$, and we complete the proof by observing $[\cdot]_+$ is nonnegative.
\end{proof}

\sloppy In view of $\max_{s' \in \cS} G^k(s') \geq 0$ (Proposition~\ref{prop:gap_functions_2}), then at any state $s \in \cS$, ${V^*}^k(s) - \max_{s' \in \cS} \frac{G^k(s')}{1-\gamma} \leq {V^*}^k(s) - \frac{[G^k(s)]_+}{1-\gamma}$.
That is, the lower bound from~\eqref{eq:adapative_lb} is tighter than the lower bound from Proposition~\ref{prop:gap_functions_2}.
Moreover,~\eqref{eq:adapative_lb} gets closer to becoming a valid lower bound as the iteration $k$ increases; in the worst-case, the overestimation error is at most $O(k^{-1/2})$.

To finish, we observe Theorem~\ref{thm:validation_analysis} only says the average (over iterations) function gap -- not a single policy's function gap -- can be made arbitrarily small.
While in stochastic convex optimization one can take the ergodic average to derive an optimality gap for a single solution, reinforcement learning over policy space is nonconvex.
So, similar to stochastic nonconvex optimization, one can possibly output a (uniformly) random policy for each state~\cite{ghadimi2013stochastic,lan2023policy}.
However, this strategy provides no guarantees about a single policy nor its expected (over the random vectors) value.
In the next section, we argue the last iterate has meaningful convergence properties.

\subsection{Last-iterate convergence}
The following establishes last-iterate convergence.
% Similar results have appeared with an adaptive step size~\cite{shamir2013stochastic}.
We defer the proof to Appendix~\ref{sec:proofs_for_validation_analysis}.
\begin{proposition} \label{prop:convex_last_iterate_with_step_size}
    Suppose Assumption~\ref{asmp:expectation} and~\eqref{eq:M_h_continuous} take place.
    With $\eta_t = \sqrt{\frac{\bar{D}_0}{(\bar{Q}^2 + \sigma^2 + M_h^2)k}}$, then 
    \begin{talign*}
        \mathbb{E}[V^{\pi_{k-1}}(s) - V^{\pi^*}(s)]
        \leq
        \frac{2\sqrt{\bar{D}_0 \cdot (\bar{Q}^2 + \sigma^2 + M_h^2)}\log(2k)}{(1-\gamma)\sqrt{k}} + \frac{4\varsigma(\sqrt{2} + 8\sqrt{k})}{1-\gamma}, \ \forall s \in \cS,
    \end{talign*}
    where $\bar{D}_0:= \max_{s,\pi} D^{\pi}_{\pi_0}(s)$ for an arbitrary Bregman's distance.
    If we also have $\mu_h > 0$ and $\eta_t = \frac{1}{\mu_h(t+1)}$, then 
    \begin{talign*}
        \mathbb{E}[V^{\pi_{k-1}}(s) - V^{\pi^*}(s)]
        \leq
        \frac{\mu_h \bar{D}_0 + 3\mu_h^{-1}(\bar{Q}^2 + M_h^2) \ln(2k)}{(1-\gamma)k} + \frac{4\varsigma \ln(2k)}{1-\gamma}, \ \forall s \in \cS.
    \end{talign*}
\end{proposition}

The rates of convergence match those in Theorem~\ref{thm:agg_convergence} and Theorem~\ref{thm:strongly_convex_agg_convergence} up to log factors.
The major advantage is we identified single policy -- not the average over iterations -- that achieves distribution-free convergence.
A limitation with last-iterate convergence is that we cannot provide the same validation analysis as with the average iterate.
This is because unless we generate more samples w.r.t.~the last iterate $\pi_{k-1}$, we cannot reduce the estimation error of $\tilde{Q}^{\pi_{k-1}}$ by averaging.
Nevertheless, the proposition serves a valuable purpose. It gives guarantees on the last-iterate, from a statistical perspective. 
In the next section, we show how to enhance the accuracy estimate of any chosen policy, e.g.~the last-iterate $\pi_{k-1}$.

\subsection{Offline accuracy certificates} \label{sec:offline_step}
Consider some policy $\hat \pi$ (e.g., a random policy or the last-iterate), and we want to assess its solution quality.
Given a set of $N$ random samples $\{\xi_t\}_{t=0}^{N-1}$ (separate from the samples used to obtain $\hat \pi$), we form the estimate $\tilde{Q}^{\hat{\pi}}_t$ from each $\xi_t$. 
% We assume $\tilde{Q}^{\hat{\pi}}_t$ satisfies either Assumption~\ref{asmp:expectation} or~\ref{asmp:high_prob}.
The empirical value and advantage function can be defined as, respectively, $\tilde{V}^{\hat{\pi}}_t = \langle \tilde{Q}^{\hat{\pi}}_t(s,\cdot), \hat{\pi}(\cdot \vert s) \rangle$, and $\tilde{\psi}^{\hat{\pi}}_t$ is defined similarly to~\eqref{eq:def_advantage} but with the noisy estimates.
We refer to these quantities and samples as offline, since they are computed after a policy is found.
In contrast, the online ones from subsection~\ref{sec:online_accuracy} are generated on-the-fly as one improves the policy.
% Using these offline samples to assess the solution quality is often referred to as the \textit{validation step}~\cite{lan2012validation}.

If for example Assumption~\ref{asmp:expectation} takes place with some small bias $\varsigma$,
then the ergodic average 
\begin{talign}
    \tilde{V}_N(s) := N^{-1} \sum_{t=0}^{N-1}\tilde{V}^{\hat{\pi}}_t(s) \label{eq:offline_noisy_value}
\end{talign}
provides an estimate of $V^{\hat{\pi}}(s)$.
In particular, the expected value can be bounded by $\mathbb{E}[\tilde{V}_N(s) - V^{\hat \pi}(s)] \leq \varsigma$ (see~\eqref{eq:value_bias}), and the deviation converges as $\mathbb{E} \vert \tilde{V}_N(s) - V^{\hat{\pi}}(s) \vert  \leq \sigma N^{-1/2} + \varsigma$ for all states, which can be shown similar to the proof for Theorem~\ref{thm:validation_analysis}.
The main difference between these aforementioned bounds and the ones from Theorem~\ref{thm:validation_analysis} is that the former are taken w.r.t.~the value function for a single policy while the latter are w.r.t.~the averaged (over policies) value function $\bar{V}^k(s)$ from~\eqref{eq:noisy_upper_bound}.
Thus, one can obtain better performance estimates on, e.g., the last-iterate policy.
Likewise, one can use the offline samples $\{\xi_t\}_{t=0}^{N-1}$ to estimate the advantage gap function:
\begin{talign}
    \tilde{G}_N(s) :=
    N^{-1}
    \max_{p \in \DA} \{ -\sum_{t=0}^{N-1} \tilde{\psi}^{\hat{\pi}}(s,p) \}. \label{eq:offline_noisy_gap}
\end{talign}
Similar to $\tilde{V}_N(s)$, the function $\tilde{G}_N(s)$ is the noisy advantage gap function w.r.t.~a single policy rather than the averaged (over policies) advantage gap $\tilde{G}^k(s)$ from~\eqref{eq:noisy_gap}.

We will now show by combining both the estimated value function $\tilde{V}_N(s)$ and advantage gap function $\tilde{G}_N(s)$, one can obtain estimates of the optimal value function.
\begin{proposition} \label{prop:offline_estimates}
    Suppose Assumption~\ref{asmp:expectation} and Assumption~\ref{asmp:high_prob} take place, as well as $\varsigma \leq \sigma/(2N)$.
    Denoting $[\cdot]_+ = \max\{0,\cdot\}$, we then have
    \begin{talign*}
        &\mathbb{E} \Big[ \big[
        \big(\tilde{V}_N(s) - (1-\gamma)^{-1} \max_{s' \in \cS} \tilde{G}_N(s')\big) - V^{\pi^*}(s)
        \big]_+\Big]  \\
        &\hspace{20pt} \leq
        % Has flexible alpha
        % vvvvvvv
        % \frac{2(\alpha^{-1} \log \vert \cA \vert + \alpha \sigma^2) + 4\sqrt{\sigma^2 + N\varsigma^2}}{\sqrt{N}}, \ \forall s \in \cS.
        \sqrt{\frac{\sigma^2 + N\varsigma^2}{N}}
        +
        \frac{2\sqrt{\bar{D}_0 \cdot (\bar{Q}^2 + \sigma^2 + M_h^2)}}{(1-\gamma)\sqrt{N}}
        % \frac{2\sqrt{\sigma^2 \bar{D}_0 }}{(1-\gamma)\sqrt{N}}
        +
        \frac{8C\sigma(\log \vert \cS \vert + 1) + 4\sqrt{\sigma^2 + N\varsigma^2}}{(1-\gamma)\sqrt N}, ~~~~\forall s \in \cS,
    \end{talign*}
    for some absolute constant $C > 0$ and $\bar{D}_0:= \max_{s,\pi} D^{\pi}_{\pi_0}(s)$ for an arbitrary Bregman's distance.
    Suppose instead Assumption~\ref{asmp:high_prob} occurs and $\varsigma \leq \sigma/N$.
    Then for any $\delta \in (0,1]$,
    \begin{talign*}
        \mathrm{Pr}\big\{ 
        \exists s \in \cS : 
        \big(\tilde{V}_N(s) - (1-\gamma)^{-1} \max_{s' \in \cS} \tilde{G}_N(s')\big) - V^{\pi^*}(s)
        >
        A_N + B_N 
        \big\}
        \leq
        \delta,
    \end{talign*}
    where
    $A_N := \sigma \sqrt{\frac{3 \log(4 \vert \cS \vert/{\delta})}{N}}$ 
    and 
    % Flexible alpha
    % vvvvvvvvv
    % $B_N := \frac{2(\alpha^{-1} \log \vert \cA \vert + \alpha \sigma^2) + 4\sigma\sqrt{3\log(16 \vert \cS \vert/\delta)} + 24N^{-1} \alpha \sigma^2 \log ({32 \vert \cS \vert}/{\delta})}{\sqrt{N}}$.
    $B_N := \frac{4\sqrt{\sigma^2 \bar{D}_0} + 4\sigma\sqrt{3\log(16 \vert \cS \vert/\delta)} + 24 \sqrt{\sigma^2 \bar{D}_0} \log ({32 \vert \cS \vert}/{\delta})}{(1-\gamma)\sqrt{N}}$.
\end{proposition}
\begin{proof}
    We start by bounding
    \begin{talign}
        &\big(\tilde{V}_N(s) - (1-\gamma)^{-1} \max_{s' \in \cS} \tilde{G}_N(s')\big) - V^{\pi^*}(s) \nonumber \\
        &\leftstackrel{\text{Proposition~\ref{prop:gap_functions}}}{\leq}
        \big(\tilde{V}_N(s) - (1-\gamma)^{-1} \max_{s' \in \cS} \tilde{G}_N(s')\big) - \big(V^{\hat{\pi}}(s) - (1-\gamma)^{-1} \max_{s' \in \cS} g^{\hat{\pi}}(s') \big) \nonumber \\
        &\leq
        \vert \tilde{V}_N(s) - V^{\hat{\pi}}(s) \vert + (1-\gamma)^{-1} \max_{s' \in \cS} \vert  \tilde{G}_N(s') - g^{\hat{\pi}}(s') \vert \nonumber.
        % &\leftstackrel{\text{Theorem~\ref{thm:validation_analysis}}}{\leq}
    \end{talign}
    Recall the Q-function error $\delta_t(s,a) = \Qt(s,a) - \tQt(s,a)$.
    Since all the terms in the last line above are nonnegative,
    \begin{talign}
        &\mathbb{E} \Big[ \big[
        \big(\tilde{V}_N(s) - (1-\gamma)^{-1} \max_{s' \in \cS} \tilde{G}_N(s')\big) - V^{\pi^*}(s)
        \big]_+\Big] \nonumber \\
        &\leq
        \mathbb{E}\vert \tilde{V}_N(s) - V^{\hat{\pi}}(s) \vert + (1-\gamma)^{-1} \mathbb{E}\max_{s' \in \cS} \vert  \tilde{G}_N(s') - g^{\hat{\pi}}(s') \vert \nonumber \\
        &\leftstackrel{\substack{\text{Theorem~\ref{thm:validation_analysis}}\\ \text{and}~\eqref{eq:G_deviation_penulatimate}}}{\leq}
        \frac{\sqrt{\sigma^2 + N \varsigma^2}}{\sqrt{N}}
        +
        \frac{\mathbb{E} \max_{s \in \cS} \vert N^{-1} \sum_{t=0}^{N-1} \langle {\delta_t(s, \cdot)}, \hat{\pi}(\cdot \vert s) - u_t(\cdot \vert s) \rangle \vert}{1-\gamma} \nonumber \\
        &\hspace{10pt}+
        \frac{\mathbb{E} \max_{s \in \cS} \vert N^{-1} \sum_{t=0}^{N-1} \langle {-\delta_t(s, \cdot)}, \hat{\pi}(\cdot \vert s) - v_t(\cdot \vert s) \rangle \vert}{1-\gamma}
        +
        \frac{2\sqrt{\bar{D}_0 \cdot (\bar{Q}^2 + \sigma^2 + M_h^2)}}{(1-\gamma)\sqrt{N}}
        % \frac{\alpha^{-1} \bar{D}_0 + \alpha \sigma^2}{(1-\gamma)\sqrt{N}}, 
        \label{eq:advgap_dev}
    \end{talign}
    where the auxiliary variables $\{u_t\}$ and $\{v_t\}$ are from~\eqref{eq:auxiliary_seq_defn}.
    It remains to upper bound the second and third terms from the last line above.
    Consider the random variable $U(s) := N^{-1} \sum_{t=0}^{N-1}  \langle {\delta_t(s, \cdot)}, \hat{\pi}(\cdot \vert s) - u_t(\cdot \vert s) \rangle$. 
    By triangle inequality and the assumption $\max_{p \in \DA} \|p\| \leq 1$, $\vert U(s) \vert \leq 2N^{-1} \sum_{t=0}^{N-1} \|\delta_t(s,\cdot)\|_*$.
    Then we have for any $\theta > 0$,
    \begin{talign*}
        &\mathbb{E}\exp\{\theta^{-1} \cdot\vert U(s) \vert\} \\
        % &=
        % \mathbb{E}\big[
        %     \exp \{ (N \theta)^{-1} \vert \sum_{t=0}^{N-1}   \langle {\delta_t(s, \cdot)}, \hat{\pi}(\cdot \vert s) - u_t(\cdot \vert s) \rangle \vert
        %     \}
        % \big ] \\
        &\leq
        \mathbb{E}\big[
            \exp \{ 2(N \theta)^{-1} \sum_{t=0}^{N-1}  \|\delta_t(s, \cdot)\|_*
            \}
        \big] \\
        &=
        \mathbb{E}_{\xi_{[N-2]}} \big[
            \mathbb{E}_{\vert \xi_{[N-2]}} [\exp \{ 2(N \theta)^{-1}  \|\delta_{N-1}(s, \cdot)\|_*
            \} ]
            \cdot
            \exp \{ 2(N \theta)^{-1} \sum_{t=0}^{N-2}  \|\delta_t(s, \cdot)\|_*
            \}
        \big] \\
        &\leftstackrel{\substack{\text{Assumption~\ref{asmp:high_prob},~\eqref{eq:mgf_quadratic_ub},}\\ \text{and}~\varsigma \leq \sigma/(2N)}}{\leq}
        (1 + 1/(2N)) \exp\{ \frac{3\sigma^2}{\theta^2 N^2}\}
        \cdot
        \mathbb{E}_{\xi_{[N-2]}} \big[
            \exp \{ 2(N \theta)^{-1} \sum_{t=0}^{N-2}  \|\delta_t(s, \cdot)\|_*
            \} \big] \\
        &\leftstackrel{\substack{\text{Repeatedly}\\ \text{apply}~\eqref{eq:mgf_quadratic_ub}}}{\leq}
        (1 + 1/(2N))^{N} \exp\{ \frac{3\sigma^2}{\theta^2 N}\} \\
        &\leq
        \sqrt{e} \cdot \exp\{ \frac{3\sigma^2}{\theta^2 N}\},
    \end{talign*}
    where the third line used the law of total expectation and the fact the sample $\xi_{t-1}$ is observed after $\xi_{[t-2]}$.
    Fixing $\theta = 4\sigma/\sqrt{N}$, we get $\mathbb{E}\exp\{\vert U(s) \vert/\theta\} \leq 2$.
    Equivalently, $\|U(s)\|_{\psi_1} \leq 4\sigma/\sqrt{N}$, where the sub-exponential norm of a random variable $X$ is $\|X\|_{\psi_1} := \inf\{t > 0 : \mathbb{E}\mathrm{exp}\{\vert X \vert/t\} \leq 2\}$~\cite[Section 2.7]{vershynin2018high}.
    Moreover, in view of Assumption~\ref{asmp:expectation},
    one can use~\eqref{eq:llog_decrease} to show $\mathbb{E}\vert U(s) \vert\leq \sqrt{4(\sigma^2 + N\varsigma^2)/N}$ for all $s \in \cS$.
    Therefore, 
    it follows from~\cite[Lemma 25]{lan2023policy} that
    \begin{talign*}
        \mathbb{E} \max_{s \in \cS} \big \vert N^{-1} \sum_{t=0}^{N-1} \langle {\delta_t(s, \cdot)}, \hat{\pi}(\cdot \vert s) - u_t(\cdot \vert s) \rangle \big \vert
        &=
        \mathbb{E}\max_{s \in \cS} U(s)
        \leq
        \frac{4C\sigma(\log \vert \cS \vert + 1) + 2\sqrt{\sigma^2 + N\varsigma^2}}{\sqrt N},
    \end{talign*}
    for some absolute constant $C$.
    And a similar bound can be shown with $\delta_t$ and $u_t$ replaced by $-\delta_t$ and $v_t$, respectively.
    Applying the just-derived inequalities back into~\eqref{eq:advgap_dev}, we arrive at 
    \begin{talign*}
        &\mathbb{E} \Big[ \big[
        \big(\tilde{V}_N(s) - (1-\gamma)^{-1} \max_{s' \in \cS} \tilde{G}_N(s')\big) - V^{\pi^*}(s)
        \big]_+\Big] \\
        &\leq
        \frac{\sqrt{\sigma^2 + N\varsigma^2}}{\sqrt{N}}
        +
        \frac{2\sqrt{\bar{D}_0 \cdot (\bar{Q}^2 + \sigma^2 + M_h^2)}}{(1-\gamma)\sqrt{N}}
        +
        \frac{8C\sigma(\log \vert \cS \vert + 1) + 4\sqrt{\sigma^2 + N\varsigma^2}}{(1-\gamma)\sqrt N},
    \end{talign*}
    which yields the first bound from the proposition we are trying to prove.
    % Since $\alpha > 0$ can be arbitrarily chosen (this is because it only appears in the step size $\eta_t = \alpha N^{-1/2}$ within the auxiliary sequence~\eqref{eq:auxiliary_seq_defn}), then by selecting $\alpha = \sqrt{\bar{D}_0/\sigma^2}$, we get the first bound.

    The high probability bound can derived by applying the high-probability bounds from Theorem~\ref{thm:validation_analysis} (which still holds when fixing a single policy rather than taking the average over policies) into~\eqref{eq:advgap_dev}.
\end{proof}

A couple remarks are in order.
First, the offline estimate can be applied to a policy $\hat{\pi}$ from any RL algorithm.
Second, the above result only bounds the expected nonnegative component because we want to ensure $\tilde{V}_N(s) - (1-\gamma)^{-1} \max_{s' \in \cS} \tilde{G}_N(s')$ is a lower bound of $V^{\pi^*}(s)$.
Third, the above proposition requires the additional Assumption~\ref{asmp:high_prob} for the result in expectation, while Theorem~\ref{thm:validation_analysis} does not.
Without Assumption~\ref{asmp:high_prob}, there will be an additional $\vert \cS \vert$ dependence, whereas the above result has a milder $\log \vert \cS \vert$ dependence.
Next, we conduct experiments to examine the effectiveness of the online and offline validation steps for (stochastic) PMD.

\section{Numerical experiments} \label{sec:experiments}
We conduct preliminary numerical experiments for deterministic and stochastic PMD.
The source code can be found in \url{https://github.com/jucaleb4/pg-termination},
which contains additional information on step size tuning and the MDP environments.

\subsection{Exact solutions for deterministic MDPs} \label{sec:det_mdp_exp}
\begin{table}[t]
% \caption{Least$|$most number of iterations to find the optimal solution (run over 10 different seeds for the environment). % Note that worst-case complexity to find the optimal policy for policy mirror descent is 1,279,200 and 2,767,500 for {GridWorld} and {Taxi}, respectively, for a discount factor of $\gamma=0.9$.
\caption{Number of iterations to find optimal solution. Methods that cannot find a solution within the iteration limit of 20,000 are marked with a ``-''. % Note that worst-case complexity to find the optimal policy for policy mirror descent is 1,279,200 and 2,767,500 for {GridWorld} and {Taxi}, respectively, for a discount factor of $\gamma=0.9$.
}
\centering
\label{tab:exact_results}
% \begin{tabular}{@{\extracolsep{1cm}}l@{}lr@{}} \toprule
\begin{tabular}{@{}llrrr@{}} \toprule
Alg & Env & Iters ($\gamma=0.9$) & Iters ($\gamma=0.99$) & Iters ($\gamma=0.999$) \\ \midrule
% PI & GridWorld & $8 | 21 $ & $9 | 11$ & $8 | 11$  \\
PI & GridWorld & $8 $ & $9$ & $8$  \\
% PMD (KL) & GridWorld & $51 | 95$ & $553 | 838$ & $2162 | 6795$ \\
% PMD (Euc-Agg) & GridWorld & $18 | 24$ & $20 | 28 $ & $24 | 31$ \\
PMD (Euc-Agg) & GridWorld & $18$ & $20$ & $24$ \\
% PMD (Euc) & GridWorld & $208 | 334$ & $2002 | 3331 $ & $18122 | 33841$ \\
PMD (Euc) & GridWorld & $208$ & $2002$ & $18122$ \\
REINFORCE & GridWorld & - & - & - \\
TRPO & GridWorld & - & - & - \\
\addlinespace[2.5pt]
% PMD (KL) & Taxi & $3 | 3$ & $219 | 219$ & $3 | 3$ \\
PI & Taxi & $16$ & $20$ & $17$ \\ 
PMD (Euc-Agg) & Taxi & $4$ & $33$ & $20$ \\
PMD (Euc) & Taxi & $4$ & $52$ & $68$ \\
REINFORCE & Taxi & $4265$ & $7278$ & - \\
TRPO & Taxi & $2$ & $2$ & $3$ \\
\bottomrule
\end{tabular}
\end{table}

% We start by demonstrating policy mirror descent (PMD) can find the optimal solution to finite state and action space MDPs.
We consider two environments: the GridWorld~\cite{dann2014policy} and Taxi environment. 
In GridWorld, there is a $20 \times 20$ 2-D grid with a single target with a large (desirable) negative cost and multiple traps with large (undesirable) positive costs. 
The agent moves in one of the four cardinal directions and stays within the grid, and each step incurs a cost of +1.
The environment has a random action rule, where the chosen action has a 95\% chance of being applied and a 5\% chance of another random action being applied instead.
% To ensure the environment is ergodic, 
Once the target is reached the agent moves to a random non-trap space.
The Taxi environment is a similar $5 \times 5$ 2D grid, where the agent must first pick up the passenger 
% (located at one of the four hot-spot locations) 
followed by dropping off the passenger at a pre-specified destination.
% (located at one of the remaining three hot-spot locations).
% The cost is designed to incentivize the agent to efficiently transport the passenger to the destination.
% By correctly dropping off the passenger, the agent receives a large (negative) cost.
% Additionally, there is a penalty for picking up or dropping when the passenger's location does not coincide with the agent (and likewise with dropping off when not at the destination).
% We also removed the traps and random action modification from GridWorld.
See \url{https://gymnasium.farama.org/environments/toy_text/taxi/} for more details.

% To solve the corresponding MDPs, we applied three algorithms:
\edits{We apply five algorithms: policy iteration (labeled PI), PMD with Euclidean distance squared and an aggressive step size from Theorem~\ref{thm:strongly_poly_euclidean} (labeled PMD (Euc-Agg)), a similar PMD but with a more conservative step size from  Theorem~\ref{thm:V_linear_converence} (labeled PMD (Euc)), a policy gradient method called REINFORCE~\cite{williams1992simple} with a softmax parameterization~\cite{agarwal2021theory}, and trust region policy optimization (labeled TRPO)~\cite{schulman2015trust}. The latter  two are popular policy gradient methods used in modern RL solvers. In our implementation of TRPO, the trust region subproblem is approximately solved via a PMD update with the KL-divergence as suggested by~\cite{neu2017unified} and a backtracking line search~\cite{schulman2015trust}.}
% We consider three discount factors $\gamma \in \{0.9, 0.99, 0.999\}$.
% Note that PMD (KL) uses both the KL divergence as the Bregman divergence and a geometrically increasing step size of $\gamma^{-t}$~\cite{xiao2022convergence,li2022homotopic} 

To assess when the optimal policy has been found, we check whether the advantage gap function from~\eqref{eq:gap_function_definition} of the policy $\pi_t$ or its greedy counterpart $\hat{\pi}_t$ (i.e., select actions with highest probability in $\pi_t$) are at most $(1-\gamma)^{-1} \times 10^{-14}$ (due to numerical errors) or if two consecutive greedy policies match, the latter being similar to the termination rule for policy iteration~\cite{puterman2014markov}.

Let us now discuss the results from Table~\ref{tab:exact_results}.
First, PI often performs the best, while PMD (Euc-Agg) performs similarly. In some cases (e.g., Taxi with $\gamma=0.9$), we even see PMD (Euc-Agg) outperforms PI.
Second, PMD (Euc) converges in all cases but has inferior performance to PMD (Euc-Agg). Specifically to GridWorld, the performance of PMD (Euc) becomes significantly worse compared to PMD (Euc-Agg) as $\gamma$ gets closer to 1. One possible reason is that the step size for PMD (Euc) is less aggressive compared to PMD (Euc-Agg) and PI, so the empirical performance may be more sensitive to the discount factor.
\edits{Third, REINFORCE has the worst performance, while TRPO only has competitive performance for Taxi but fails to converge within the iteration limit for GridWorld. The lack of robust performance from existing policy gradient method confirms the importance of supporting theory for strongly-polynomial runtime from our PMD methods}.

\subsubsection{Scaling of strongly-polynomial runtimes}
\begin{figure}[h]
   \centering
   \begin{subfigure}{.29\textwidth}
     \centering
     \includegraphics[width=0.99\linewidth,valign=b]{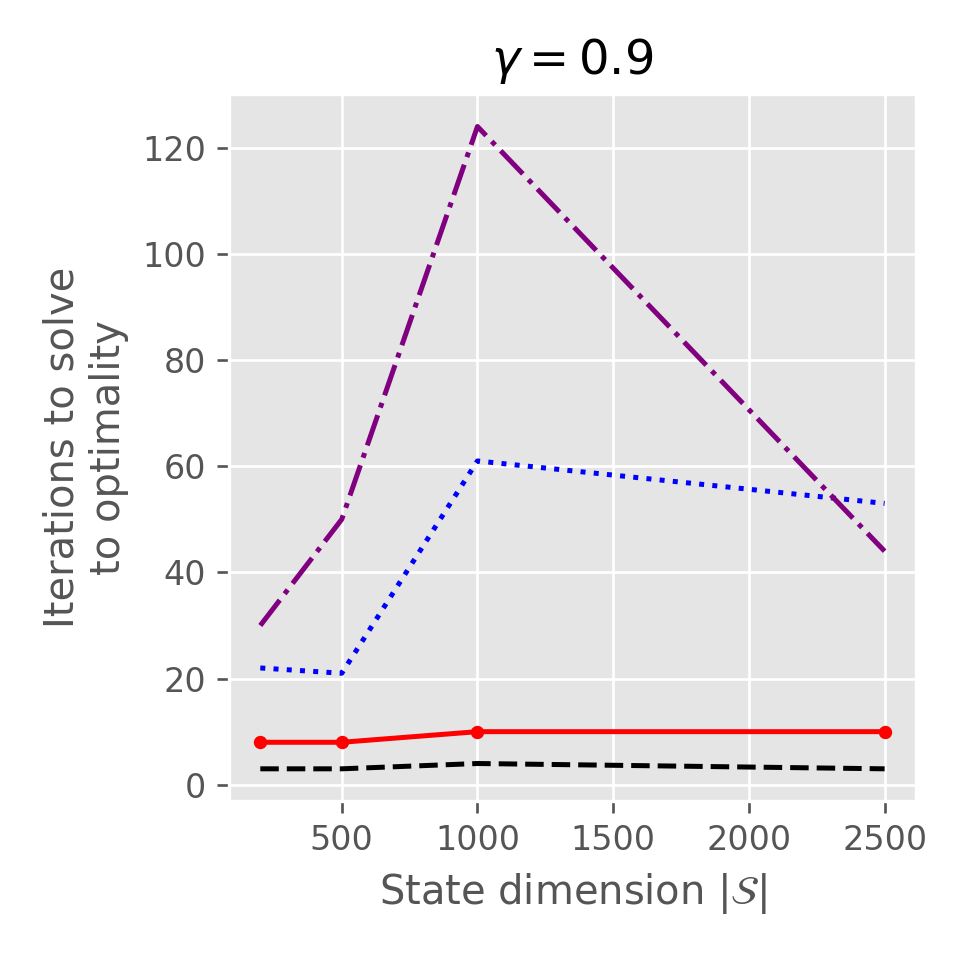}
   \end{subfigure}
   \begin{subfigure}{.29\textwidth}
     \centering
     \includegraphics[width=0.99\linewidth,valign=b]{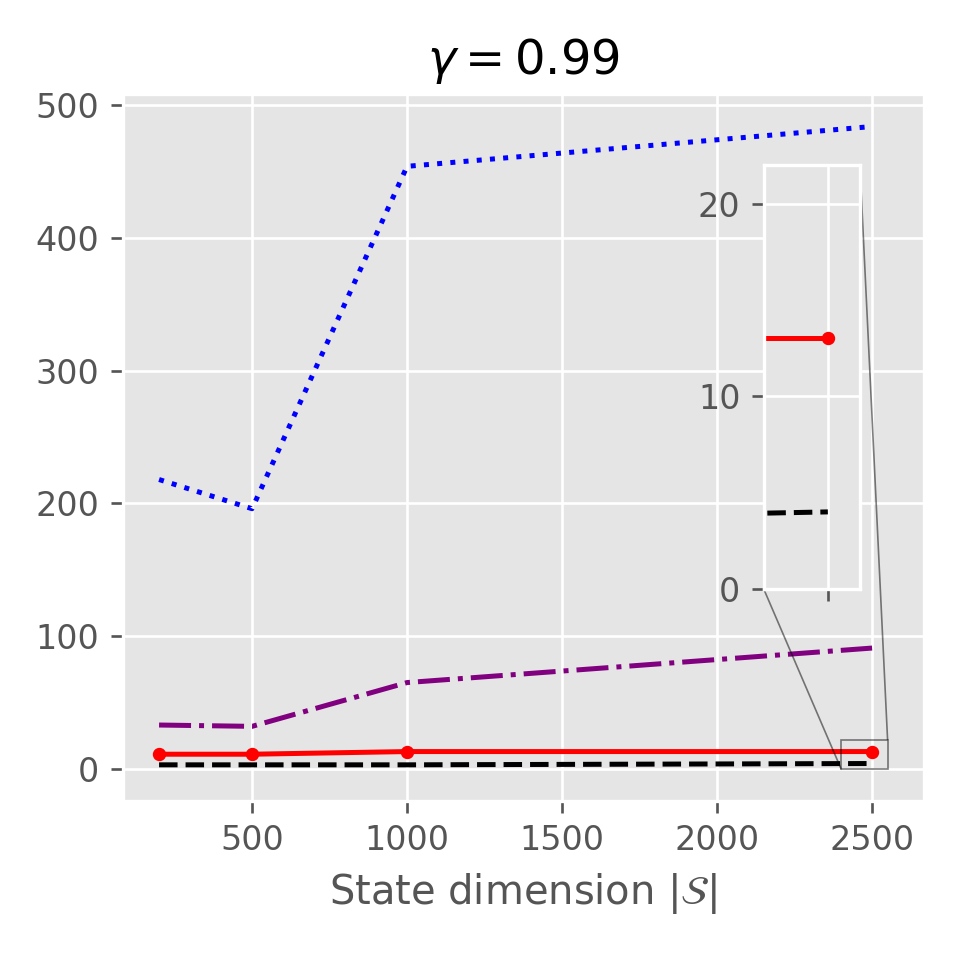}
   \end{subfigure}
   \begin{subfigure}{.36\textwidth}
     \centering
     \includegraphics[width=1.0\linewidth,valign=b]{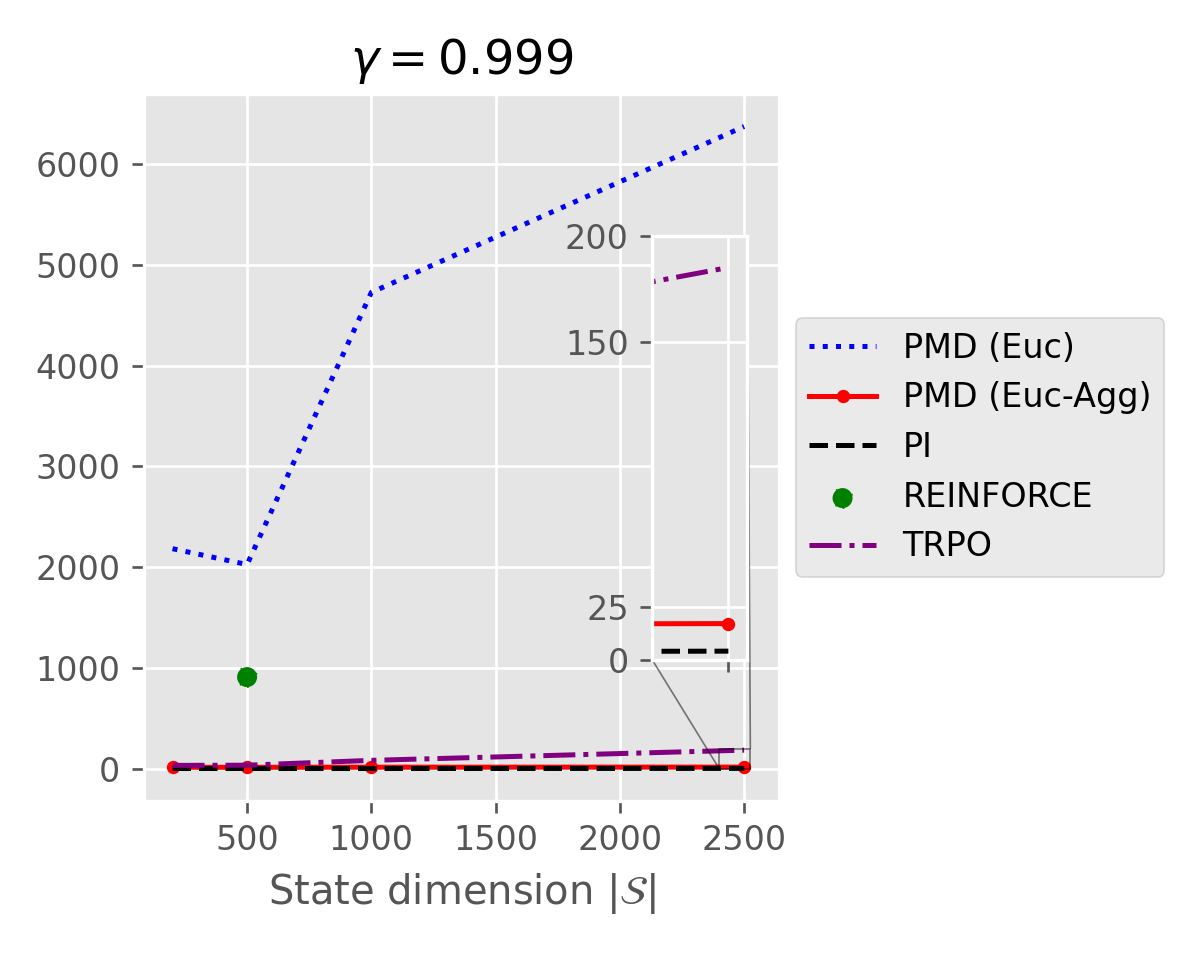}
   \end{subfigure}\\
   \caption{Number of iterations to solve GARNET MDP with $\vert \cA \vert = 30$ and increasing $\vert \cS \vert$. Most plots do not contain REINFORCE since it cannot solve to optimality within the iteration limit of 10000. The latter two plots have zoom-in plots for $\vert \cS \vert = 2500$.}
   \label{fig:garnet}
   \vspace{-5mm}
\end{figure}

\edits{Next, we compare the number of iterations to solve to optimality with increasing state space sizes and discount factors $\gamma$.  Using the same five methods, we solve different random Garnet MDPs with a fixed action space of $\vert \cA \vert = 30$. 
See~\cite{tarbouriech2019active} and references therein for more details. The results for different state space sizes $\vert \cS \vert = 500, 1000, 2000, 2500$ and discount factors $\gamma=0.9, 0.99, 0.999$ are displayed in Fig.~\ref{fig:garnet}.}

\edits{Similar to the results found in Table~\ref{tab:exact_results}, PMD (Euc-Agg) closely matches PI. Moreover, both methods seem to have mild dependence on larger $\vert \cS \vert$ and $\gamma$. On the other hand, REINFORCE has the worst performance, as there is only one instance where the problem is solved within the iteration limit, while PMD (Euc) is the second worst. On the other hand, TRPO has comparable but worse performance to PMD (Euc-Agg) and PI. This suggests that the line search procedure in TRPO can possibly improve the performance of policy gradient methods. However, as we saw in subsection~\ref{sec:det_mdp_exp}, its performance is not robust across all environments, while PMD (Euc-Agg) efficiently finds optimal solutions in all our tested environments. Again, this is supported by PMD's theory for strongly-polynomial runtime.}

\subsection{Termination and evaluation of stochastic PMD to RL}
In this section, we progress to the stochastic setting.
We use the same two environments from the previous section but now we assume the underlying MDP model is not known.
Instead, for simplicity, we assume a generative model.
% , where one can generative a sample from any state-action pair.
We start with the online validation step.
% analysis, where we monitor both the estimate of the value functions and estimate of the optimal value function. 
% These estimates are applied to both environments with two different discount factors in Figure~\ref{fig:online_validation}.
% We provide more details below on how we derived these estimates.

\subsubsection{Online validation analysis with various lower bound estimates}  \label{sec:online_exp}
\begin{table}[h]
\vspace{-5mm}
\caption{Various lower bound estimates, where the stochastic estimates of the aggregate value function $\bar{V}^k$ and advantage gap function $\tilde{G}^k$ are from~\eqref{eq:noisy_upper_bound} and~\eqref{eq:noisy_gap}, respectively, and $[\cdot]=\max\{0, \cdot\}$.
The overestimation error $\varepsilon_k(\rho)$ is defined in Corollary~\ref{cor:adaptive_lb}.
}
\centering
\label{tab:lb_estimates}
% \begin{tabular}{@{\extracolsep{1cm}}l@{}lr@{}} \toprule
\begin{tabular}{@{}lrrr@{}} \toprule
Name & Estimator & Lower bound of & Note \\ \midrule
universal aggregate & $\bar{V}^k(s) - \frac{1}{1-\gamma}\max_{s' \in \cS} \tilde{G}^k(s')$ & $V^{\pi^*}(s)$ &  Proposition~\ref{prop:gap_functions_2} \\
adaptive aggregate & $\mathbb{E}_{s \sim \rho}[\bar{V}^k(s) - \frac{1}{1-\gamma} [\tilde{G}^k(s)]_+]$ & $f_\rho(\pi^*) + \frac{\varepsilon_k(\rho)}{1-\gamma} $ & Corollary~\ref{cor:adaptive_lb} \\
worst-case & $\bar{V}^k(s) - \frac{2\sqrt{\log(\vert \cA \vert)(\bar{Q}^2 + M_h^2)}}{(1-\gamma) \sqrt{k}}$ & $V^{\pi^*}(s)$ & Theorem~\ref{thm:agg_convergence}\\
a priori & problem-dependent & $V^{\pi^*}(s)$ & Heuristic \\ \bottomrule
\end{tabular}
\end{table}

\begin{figure}[h]
   \centering
   \caption{Mean and confidence interval for estimates of the average value function $k^{-1}\sum_{t=0}^{k-1}f_\rho(\pi_t)$ and the optimal value $f_\rho(\pi^*)$, where $f_\rho(\pi) := \mathbb{E}_{s \sim \rho}V^{\pi}(s)$ and $\rho$ is the uniform distribution over states.
   Experiments are repeated over 10 seeds on the same environment.
   For the top right plot, the worst-case lower bound is not shown since it is smaller than the minimum of -200.}
   \label{fig:online_validation}
   \begin{subfigure}{.4\textwidth}
     \centering
     \includegraphics[width=\linewidth,valign=b]{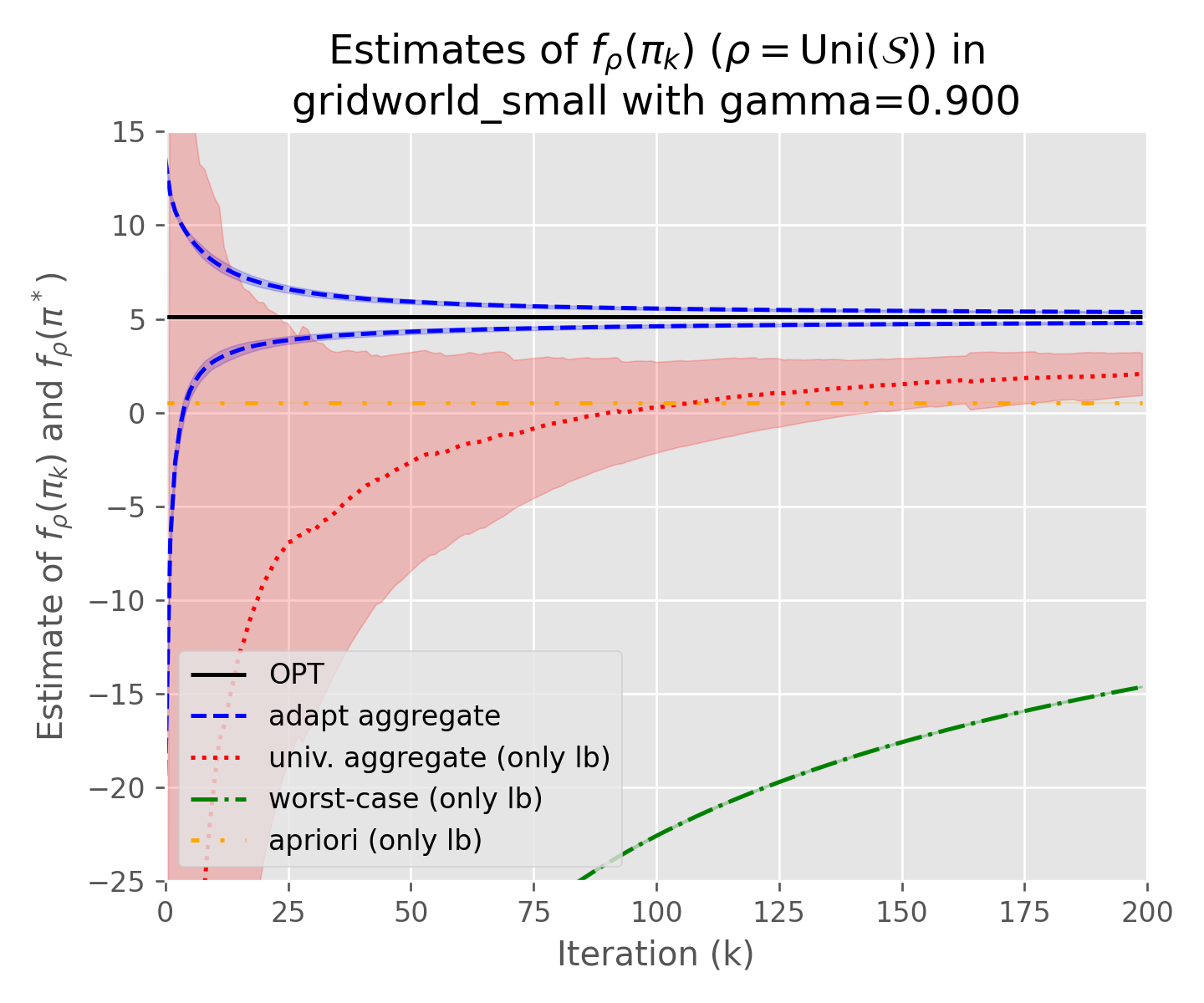}
   \end{subfigure}%
   \begin{subfigure}{.4\textwidth}
     \centering
     \includegraphics[width=\linewidth,valign=b]{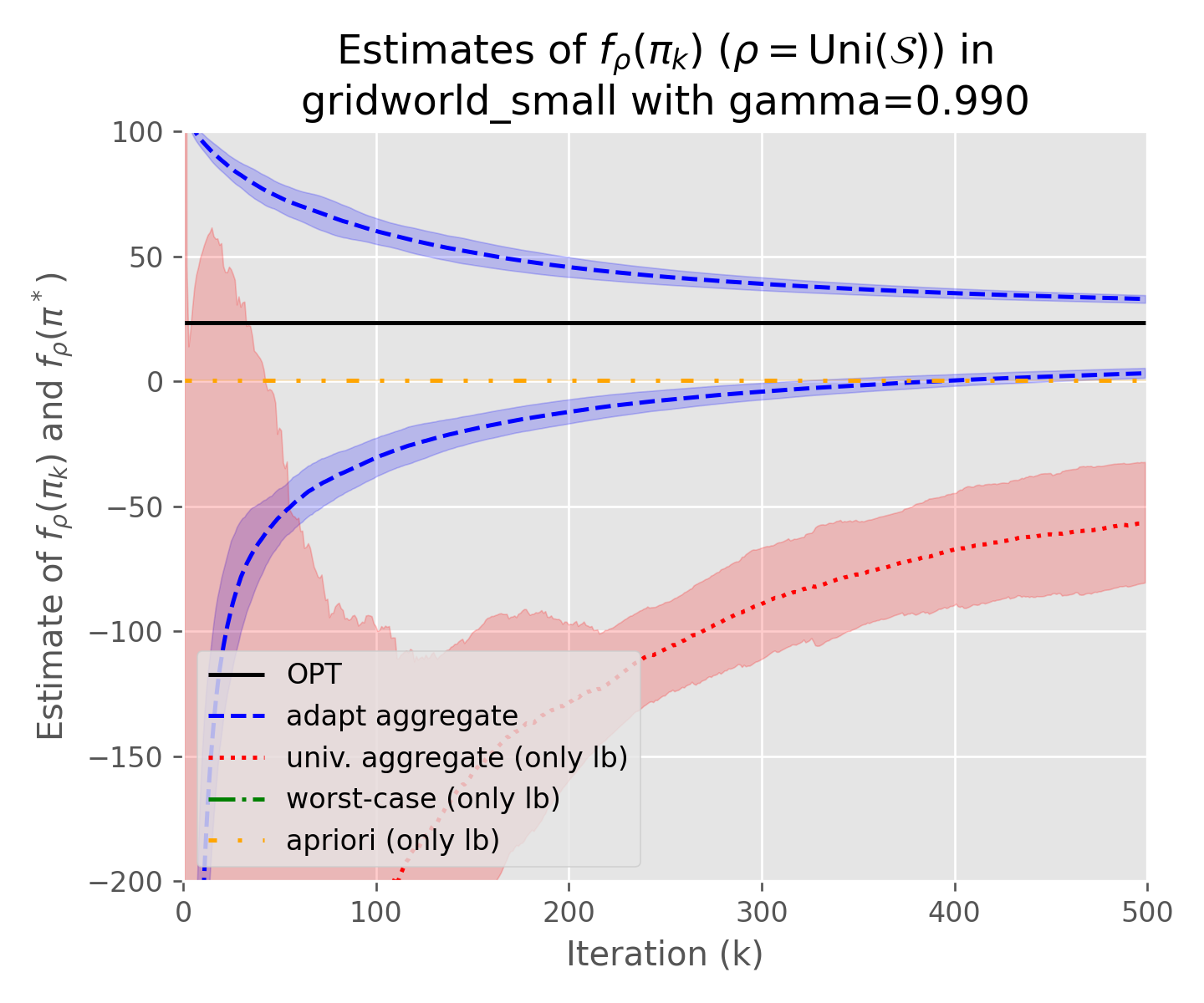}
   \end{subfigure}\\
   % \\[\smallskipamount]
   \begin{subfigure}{.4\textwidth}
     \centering
     \includegraphics[width=\linewidth,valign=b]{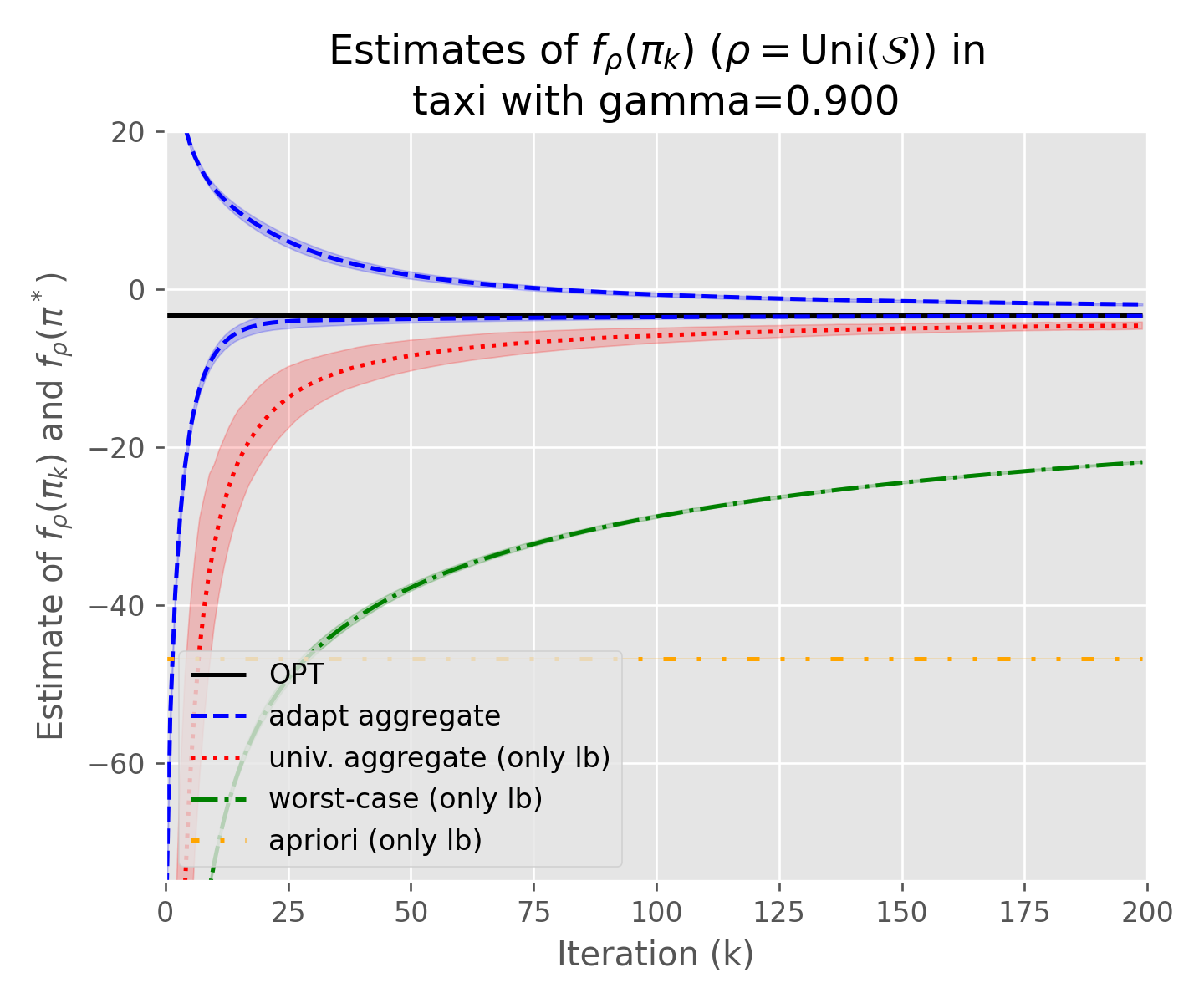}
   \end{subfigure}%
   \begin{subfigure}{.4\textwidth}
     \centering
     \includegraphics[width=\linewidth,valign=b]{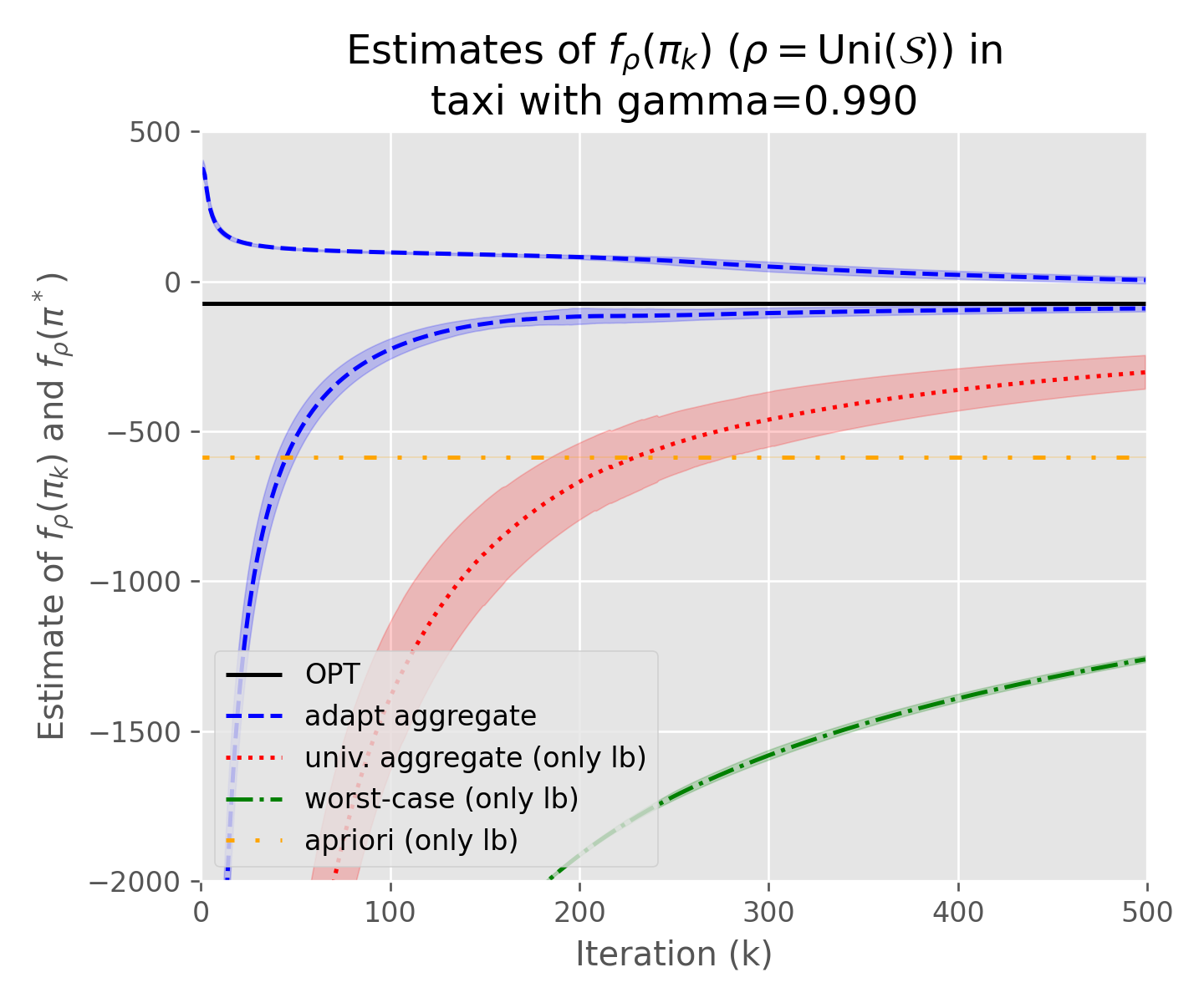}
   \end{subfigure}
   % \vspace{-1em}
\end{figure}

We consider both the noisy estimate of the aggregate (across iterations) value function $\bar{V}^k(s)$ from~\eqref{eq:noisy_upper_bound} and aggregate (across iterations) advantage gap function $\tilde{G}^k(s)$ from~\eqref{eq:noisy_upper_bound}.
Various lower bounds on the optimal value are shown in Table~\ref{tab:lb_estimates}, which we briefly discuss here.
% The first three estimates are either valid or nearly valid lower bound estimates of the optimal value. % $V^{\pi^*}(s)$ at every state or the average $f_\rho(\pi^*) := \mathbb{E}_{s \sim \rho}V^{\pi^*}(s)$ for an arbitrary distribution $\rho$ over states. 
% Their formulation, type of lower bound, and supporting theory are all displayed in Table~\ref{tab:lb_estimates}.
As mentioned after Corollary~\ref{cor:adaptive_lb}, 
the adaptive aggregate is tighter than the universal aggregate since the latter takes the universal norm (i.e., max) of the advantage gap function, while the former adaptively evaluates the function at each state.
While Corollary~\ref{cor:adaptive_lb} says the adaptive aggregate may overestimate $f_\rho(\pi^*) \equiv \mathbb{E}_{s \sim \rho}V^{\pi^*}(s)$, the overestimation $\varepsilon_k(\rho)$ will decrease with the iteration count $k$.
The worst-case lower bound is directly from the worst-case convergence analysis of stochastic PMD.
Meanwhile, the {a priori} estimate is a heuristic, where we use a priori knowledge of the MDP (i.e.,~the cost function $c$ and structural properties of the state-action space) to derive a worst-case lower bound of $V^{\pi^*}(s)$.

% In addition, one can interpret it as follows.
% By recalling~\eqref{eq:V_gap_ub_agap} and the definition of $G^k$, we have
% % \begin{talign*}
%     $k^{-1}\sum_{t=0}^{k-1}[f_\rho(\pi_t) - f_\rho(\pi^*)] = (1-\gamma)^{-1} \mathbb{E}_{q \sim \kappa^{\pi^*}_\rho}[G^k(q)]$,
% % \end{talign*}
% where $\kappa^{\pi^*}_\rho = \mathbb{E}_{s \sim \rho} \kappa^{\pi^*}_s$.
% Now, let us assume that the following takes place:
% $\mathbb{E}_{q \sim \kappa^{\pi^*}_\rho}[G^k(q)] \approx \mathbb{E}_{q \sim \rho}[G^k(q)]$. 
% That is, even when we take the expectation of the same advantage gap function w.r.t.~different distributions over states, their expected values are approximately close.
% Then 
% % \begin{talign*}
%     $k^{-1}\sum_{t=0}^{k-1}f_\rho(\pi_t) - (1-\gamma)^{-1} \mathbb{E}_{q \sim \kappa^{\pi^*}_\rho}[G^k(q)] \approx f_\rho(\pi^*)$.
% % \end{talign*}
% \sloppy Therefore, the adaptive aggregate lower bound provides an accurate (up to the error term $O(\vert \mathbb{E}_{q \sim \kappa^{\pi^*}_\rho}[G^k(q)] - \mathbb{E}_{q \sim \rho}[G^k(q)]\vert)$) approximation of the optimal function value. 

% \subsubsection{Comparisons between different value estimates for online validation analysis}
The estimate of the aggregate value function $\mathbb{E}_{s \sim \rho}[\bar{V}^k(s)]$ and various lower bounds of $f_\rho(\pi^*)$ are shown in Figure~\ref{fig:online_validation}.
% , where the distribution $\rho$ is chosen as the uniform distribution over all states.
The value $\mathbb{E}_{s \sim \rho}[\bar{V}^k(s)]$ seems to converge at a sublinear rate towards the optimal function value, as justified in Theorem~\ref{thm:agg_convergence}.
As expected, the adaptive aggregate outputs the tightest lower bound, the universal aggregate often yields the second best, while the worst-case is the least tight.
It seems both the adaptive and universal aggregate yield more informative bounds than the worst-case one.
Interestingly, the adaptive aggregate lower bound is always a valid lower bound of the optimal value,
suggesting the overestimation error $\varepsilon_k(\rho)$ from Corollary~\ref{cor:adaptive_lb} may not be too large.
% We caution that this is not to say the adaptive aggregate lower bound is a valid lower bound in general (as we have no supporting theory), but rather than one can consider using it in practice to obtain a more aggressive lower bound estimate.
% This highlights the fact these empirical metrics provide better estimations of optimal value function than the worst-case lower bound. 
% In particular, the gap from the universal aggregate lower bound seems to be much smaller than its worst-case size of $O((1-\gamma)^{-2}k^{-1/2})$ (Theorem~\ref{thm:validation_analysis}), while the worst-case gap is always of size $O((1-\gamma)^{-1}k^{-1/2})$.

\subsubsection{Comparisons between online and offline validation analysis}
We consider the offline validation step, which is run after SPMD terminates, as described in subsection~\ref{sec:offline_step}.
We only examine GridWorld since the online bounds are already tight for the Taxi environment.
% Recall in the online validation step (i.e., the previous subsection), we build upper and lower bound estimates by averaging across iterations.
% In contrast, as described in subsection~\ref{sec:offline_step}, the offline validation step generates fresh samples from the last-iterate $\pi_k$ only (i.e., new samples after finishing stochastic PMD) and forms the offline value function estimate~\eqref{eq:offline_noisy_value} and advantage gap function estimate~\eqref{eq:offline_noisy_gap}. 
% These offline estimates help construct estimates of the last-iterate value function $V^{\pi_k}$ and lower bound $V^{\pi^*}$ similarly to the online validation step.

% to the online validation step, this step generates a fresh set of samples to estimate both the last-iterate value function $V^{\pi_k}$ and a lower bound of $V^{\pi^*}$.
% We also modified the computation of the lower bound for the offline validation, where we combined (by averaging) the samples from both the online and offline estimates 
% (this is motivated by the fact both the online and offline validation step measure the same lower bound of $V^{\pi^*}(s)$, and adding more samples reduces estimation errors).
% For simplicity, we only consider the GridWorld environment since it had large gaps between the upper and lower bound values in Figure~\ref{fig:online_validation}. 

\begin{table}[h]
\caption{Online vs. offline validation analysis. 
The true upper (ub) of the last iterate and lower bound (lb) of the optimal value are shown, alongside the estimated (labeled ``Est'') mean and difference between the true and estimated quantity averaged across 10 seeded run on GridWorld.
We used 50, 125, and 250 additional samples $\xi_t$ for the offline validation with $\gamma=0.9$, $0.95$, and $0.99$, respectively (the online validation used 200, 350, and 500 samples, respectively).}
\centering
\label{tab:offline_validation}
% \begin{tabular}{@{\extracolsep{1cm}}l@{}lr@{}} \toprule
\begin{tabular}{@{}lrrrrrr@{}} \toprule
& \multicolumn{2}{c}{$\gamma=0.9$} & \multicolumn{2}{c}{$\gamma=0.95$}& \multicolumn{2}{c}{$\gamma=0.99$} \\ \cmidrule(r){2-3} \cmidrule(r){4-5} \cmidrule(r){6-7}
Metric & Est Mean & $\vert \text{Est - True}\vert$ & Est Mean & $\vert \text{Est - True} \vert$ & Est Mean & $\vert \text{Est - True} \vert$ \\ \midrule
% OPT & 5.108817 & - & 23.375970 & - \\
True ub & 5.169 & - & 7.629 & - & 25.039 & - \\
% Online ${V^*}^{k}$& 5.364327 & 0.086588 & 32.991901 & 0.549539 \\
Online ub & 5.364 & 0.195 & 8.148 & 0.519 & 32.992 & 7.952 \\
% Offline $V^{\pi_k}$& 5.161166 & 0.012358 & 23.324942 &  0.379462 \\
Offline ub & 5.161 & 0.008 & 7.623 & 0.006 & 23.325 &  1.714 \\
% True $V^{\pi_k}$& 5.169310 & 0.010890 & 25.039138 & 0.411845 \\
% Online $V^{\pi^*}$ lb& 4.786033 & 0.028969 & 3.247174 & 1.003833 \\
True lb & 5.070 & - & 7.344 & - & 20.756 & - \\ 
Online lb & 4.786 & 0.284 & 6.414 & 0.930 & 3.247 & 17.520 \\
% Offline $V^{\pi^*}$ lb & 4.830365 & 0.026306 & 7.303157 & 0.903335 \\
Offline lb & 4.830 & 0.240 & 6.598 & 0.746 & 7.303 & 13.452 \\ \bottomrule
% True $V^{\pi^*}$ lb & 5.070203 & 0.014691 & 20.755719 & 0.999302 \\ \bottomrule
\end{tabular}
\end{table}

Table~\ref{tab:offline_validation} depicts the true value of the last-iterate $V^{\pi_k}$ and its corresponding lower bound on $V^{\pi^*}$ (via Proposition~\ref{prop:gap_functions}), as well as upper and lower bound estimates from the last iteration of the online validation step and the offline validation step.
We add experiments for $\gamma=0.95$ to better understand the impact of the discount factor.
Both the online and offline validation step use the ``adaptive aggregate'' lower bound as shown in Table~\ref{tab:lb_estimates}.
% since this tends to output tighter lower bounds.
In addition, we modified the offline advantage gap function in~\eqref{eq:offline_noisy_gap} to include all samples from the online and offline validation step,
% (while the offline upper bound estimate~\eqref{eq:offline_noisy_value} only uses offline samples).
because empirically we found this outputs tighter lower bounds.
Table~\ref{tab:offline_validation} also shows the difference between the true and estimated values from the online/offline validation step.
This difference captures both estimation errors and variation in performance between the average iterate $\pi_t$ and the last-iterate $\pi_k$.
The variation arises because the online estimates average across iteration (see subsection~\ref{sec:online_exp}), 
% where earlier policies may perform worse than the last-iterate $\pi_k$,
while the offline validation step only measures w.r.t.~the last-iterate.

We can make a couple observations from Table~\ref{tab:offline_validation}.
First, upper bound estimates have less estimation errors than the lower bound.
This is because Theorem~\ref{thm:validation_analysis} says the upper bound error does not have an explicit dependence on $(1-\gamma)^{-1}$ while the lower bound does.
% Second, the estimation errors become larger as $\gamma$ gets closer to one.
Second, the offline validation step produces metrics closer to the true value. 
Since both the online and offline validation step's estimation errors decrease at similar rates (see Theorem~\ref{thm:validation_analysis} and Proposition~\ref{prop:offline_estimates}), this difference may be due to the variation in performance between the average and last iterate, as described in the previous paragraph.
% The first reason is because the online upper bound~\eqref{eq:noisy_upper_bound} estimates the performance averaged across all iterations, so it is susceptible to poor early performance, while the offline estimate only measures the last-iterate performance in~\eqref{eq:offline_noisy_value} (\eqref{eq:noisy_upper_bound}.
% The second reason is because the offline validation step includes more samples (recall we averaged across samples from both the online and offline step), which should reduce estimation errors.
% So when the policy performance greatly varies during the online validation step, 
So, we recommend the use of the offline validation step since it provides accurate estimates of the last-iterate policy, which tends to have superior performance (c.f.~Figure~\ref{fig:online_validation}).
Third, as $\gamma$ gets closer to one, the lower bound estimates get further from the true value, as supported by Proposition~\ref{prop:offline_estimates}.
Therefore, when $\gamma$ is close to 1 and the lower and upper bound gap is large, our suggestion is to only use the upper bound as a performance metric.
We finish with an experiment when the state space is infinitely large.

\section{Conclusion} \label{sec:conclusion}
We provide new convergence guarantees and validation analysis for policy mirror descent (PMD).
For the deterministic case, we introduce a novel step size that allows PMD to obtain the stronger distribution-free linear convergence.
Moreover, by incorporating our proposed advantage gap function into the step size, we improve PMD so it achieves, for the first time, strongly polynomial runtime to get the optimal solution.
For the stochastic setting, we show the stochastic PMD can also achieve the stronger distribution-free sublinear convergence, and it does so with the same constant step size as in previous developments~\cite{lan2023policy}.
We also pair this convergence analysis with a novel validation analysis,
which can be used to possibly terminate policy gradient methods sooner.
This extends the validation analysis for stochastic convex optimization~\cite{lan2012validation} to the challenging nonconvex landscape of policy optimization~\cite{agarwal2021theory}.
An important future work can be to extend the advantage gap function and its analysis to more realistic general state and action spaces that appear in RL~\cite{lan2022policy}.\newline

\appendix

% \appendix
\section{Proofs from Section~\ref{sec:prob_interest}} \label{sec:missing_pfs_from_prob_interest}
\begin{proof}[Proof of~\cref{lem:fpi_to_fx}]
    Similar to Lemma~\ref{lem:performance_diff_deter},
    \begin{talign*}
        V^\pi(s) 
        &=
        \mathbb{E}\big[ \sum_{t=0}^\infty \gamma^t [c(s_t, \pi(\cdot \vert s_t)) + h^{\pi(\cdot \vert s_t)}(s_t)] \vert s_0=s, a_t \sim \pi(\cdot \vert s_t), s_{t+1} \sim \mathcal{P}(\cdot \vert s_t,a_t) \big] \\
        &=
        \sum_{t=0}^\infty \gamma^t \mathbb{E}\big[c(s_t, \pi(\cdot \vert s_t)) + h^{\pi(\cdot \vert s_t)}(s_t) \vert s_0=s, a_t \sim \pi(\cdot \vert s_t), s_{t+1} \sim \mathcal{P}(\cdot \vert s_t,a_t) \big] \\
        &=
        \sum_{t=0}^\infty \gamma^t \sum_{q \in \cS} \mathrm{Pr}\{s_t=q \vert s_0=s\} \cdot \\
        &\hspace{10pt} 
        \mathbb{E}[c(s_t, \pi(\cdot \vert s_t)) + h^{\pi(\cdot \vert s_t)}(s_t)] \vert s_t=q, s_0=s, a_t \sim \pi(\cdot \vert s_t), s_{t+1} \sim \mathcal{P}(\cdot \vert s_t,a_t) \big] \\
        &=
        \sum_{q \in \cS} \sum_{t=0}^\infty \gamma^t \mathrm{Pr}^\pi\{s_t=q \vert s_0=s\} \cdot [c(q, \pi(\cdot \vert q)) + h^{\pi(\cdot \vert q)}(q)].
    \end{talign*}
    Noticing $\sum_{t=0}^\infty \gamma^t \mathrm{Pr}^\pi\{s_t=q \vert s_0=s\} = (1-\gamma)^{-1} \kappa_s^\pi(q)$, 
    we have
    \begin{talign*}
        f_\rho(\pi)
        &=
        \sum_{q \in \cS} [c(q, \pi(\cdot \vert q)) + h^{\pi(\cdot \vert q)}(q)] \sum_{s \in \cS} (1-\gamma)^{-1} \rho(s) \kappa^\pi_{s}(q).
    \end{talign*}
    The proof is complete by observing the last term $\sum_{s \in \cS} (1-\gamma)^{-1} \rho(s) \kappa^\pi_{s}(q) = \eta^\pi_\rho(q)$.
\end{proof}

\section{Proofs from Section~\ref{sec:improved_sublinear_convergence}} \label{sec:proofs_for_improved_sublinear_convergence}
\begin{proof}[Proof for Lemma~\ref{lem:light_tail_azuma_hoeffding}]
    % Brief note of no iid assumption
    % vvvvvvvvvvvvvvv
    % Before writing out the sketched proof, we briefly remark that one way our lemma differs from~\cite[Lemma 2]{lan2012validation} (which our result is based off of) is that we do not need the random variables $\xi_t$ to be iid. 
    % Indeed, one can show the proof holds without this assumption.
    % This is because we can invoke total law of expectation based on our assumption that bounds the conditional expectation (so we do not need independence to separate products of variables)
    % ^^^^^^^^^^^^^^^^^^^^^^^^^^^^^
    
    % Now moving onto the proof, 
    To start, Case 2 is the same as~\cite[Lemma 2]{lan2012validation}. Only Case 1 differs since we do not have zero mean, i.e., $\mathbb{E}_{\vert \xi_{[t-1]}} \phi_t = 0$.
    Our proof below shows how to handle this.

    Let $\bar{\phi}_t := \phi_t/\sigma_t$. 
    \sloppy By the given assumptions on $\phi_t$ and Jensen's inequality, $\mathbb{E}_{\vert \xi_{[t-1]}} [\exp\{a \bar{\phi_t}^2\}] \leq \exp\{a\}$ for any $a \in [0,1]$ (see~\cite[Lemma 2]{lan2012validation}).
    Now, we enter the step that differs from~\cite[Lemma 2]{lan2012validation}, since we have a nonzero expected value.
    Using the fact $\exp\{x\} \leq x + \exp\{9x^2/16\}$ for any $x$ and the assumption $\mathbb{E}_{\vert \xi_{[t-1]}}[\phi_t] \leq \sigma_t/N$, we deduce 
    \begin{talign*}
        \mathbb{E}_{\vert \xi_{[t-1]}} [\exp \{\lambda \bar{\phi_t} \}]
        \leq
        {\lambda} \mathbb{E}_{\vert \xi_{[t-1]}}[ \bar{\phi}_t] + \mathbb{E}_{\vert \xi_{[t-1]}}[\exp\{(9\lambda^2/16)\bar{\phi}_t^2 \}] 
        \leq
        {\lambda}/N
        +
        \exp\{9\lambda^2/16\}, ~~ \forall \lambda \in [0,4/3]. 
    \end{talign*}
    % for all $\lambda \in [0,4/3]$.
    Then similar to~\cite[Lemma 2]{lan2012validation}, it can be shown that (with the help of $x \leq \mathrm{exp}\{3x^2/4\}$ for any $x$)
    \begin{talign*}
        \mathbb{E}_{\vert \xi_{[t-1]}} [\exp \{\lambda \bar{\phi_t} \}]
        \leq
        {\lambda}/N
        +
        \exp\{3\lambda^2/4\} 
        \leq
        (1+N^{-1}) \mathrm{exp}\{3\lambda^2/4\}, 
        ~~\forall \lambda \geq 0.
    \end{talign*}
    By a change of variables, we equivalently have
    \begin{talign} \label{eq:mgf_quadratic_ub}
        \mathbb{E}_{\vert \xi_{[t-1]}}[\exp \{\kappa \phi_t\}]
        \leq 
        (1+N^{-1})\exp\{3\kappa^2 \sigma_t^2/4\}, \ \forall \kappa > 0.
    \end{talign}
    The rest of the proof follows similarly to~\cite[Lemma 2]{lan2012validation}, with the main difference being our bound incurs an addition factor of $\exp\{1\}$ since $(1+c \cdot N^{-1})^{N} \leq \exp\{c\}$ for any $c \geq 0$ and $N \geq 0$.
\end{proof}

\begin{proof}[Proof of Theorem~\ref{thm:strongly_convex_agg_convergence}]
    %% Show it below!!!!
    % vvvvvvvvvvvvvvvvvvvv
    First, note $\sum_{t=1}^k t^{-1} \leq \log(2k)$.
    % Then similar to Theorem~\ref{thm:agg_convergence},
    In view of the step size $\eta_t = (\mu_h(t+1))^{-1}$ and $\|\pi'(\cdot \vert s) - \pi(\cdot \vert s)\| \leq 2$ for any two policies $\pi'$ and $\pi$ (since we assumed $\max_{p \in \DA} \|p\|\leq 1$), then
    \begin{talign} 
        &(1-\gamma)\sum_{t=0}^{k-1} \mathbb{E}[V^{\pi_t}(s) - V^{\pi^*}(s)]  + \eta_{k-1}^{-1} \mathbb{E}\mathbb{E}_{q \sim \kappa^{\pi^*}_s} [D^{\pi^*}_{\pi_k}(q)] \nonumber \\
        &~~~\leftstackrel{\text{Lemma~\ref{lem:pmd_value_descent}}}{\leq}
        \mathbb{E}_{q \sim \kappa^{\pi^*}_s} \big[ \eta_0^{-1} D^{\pi^*}_{\pi_0}(q) + \sum_{t=1}^{k-1} (-\eta_{t-1}^{-1}(1+\eta_{t-1} \mu_h)  + \eta_t^{-1}) D^\pi_{\pi_t}(q) \big] \nonumber \\
        &\hspace{30pt}+
        \mathbb{E}_{q \sim \kappa^{\pi^*}_s} \big[
        \sum_{t=0}^{k-1} \eta_t(\mathbb{E}\|\tQtq\|_*^2 + M_h^2) + \sum_{t=0}^{k-1} \mathbb{E} \zeta_t(q,\pi^*) \big]  \label{eq:agg_convergence_strong_convex} \\
        &~~~\leftstackrel{\substack{\text{Choice of $\eta_t$,} \\~\eqref{eq:second_moment},~\eqref{eq:bias_bound}}}{\leq}
        \mu_h \bar{D}_0 + \mu_h^{-1} (\bar{Q}^2 + M_h^2) \log(2k)  + 2 \varsigma k,\nonumber 
    \end{talign}
    where $\zeta_t(q,\pi)$ is from Lemma~\ref{lem:pmd_value_descent}.
    Dividing by $(1-\gamma)k$ establishes the bound in expectation.
    % The second result can be shown similarly, so we omit it.
    %% Second result is below!
    %% vvvvvvvvvvvvvvvvv
    % Similarly,
    % \begin{talign*}
    %     &(1-\gamma)\sum_{t=k_0}^{k-1} \eta_t \mathbb{E}[V^{\pi_t}(s) - V^{\pi_{k_0}}(s)] \\
    %     &~~~\leftstackrel{\text{Lemma~\ref{lem:pmd_value_descent}}}{\leq}
    %     \mathbb{E} \mathbb{E}_{q \sim d^\pi_s}[D^{\pi_{k_0}}_{\pi_{k_0}}(q)] + (\bar{Q}^2 + M_h^2) \sum_{t=k_0}^{k-1} \eta_t^2 + 2\varsigma \sum_{t=k_0}^{k-1} \eta_t \\
    %     &~~~\leftstackrel{\text{Choice of $\eta_t$}}{\leq}
    %     (\bar{Q}^2 + M_h^2) \cdot \alpha^2 2\log(k/(k_0+1))  + 2\varsigma \cdot 4\sqrt{k-k_0}.
    % \end{talign*}
    % ^^^^^^^^^^^^^^^^^^^^^

    % For the third result, 
    To prove the second result, we need to decompose the bias terms differently.
    Our decomposition breaks the iterations into \textit{partitions} of consecutive iterations. 
    The i-th partition (starting with $i=0$) consists of iterations $[k_i,k_{i+1})$, where
    $k_0=0$ and $k_i = C \cdot 2^{i-1}$ for $i\geq 1$ and 
    $C \equiv C(k)$ (see Theorem~\ref{thm:strongly_convex_agg_convergence} for the definition of $C(k)$).
    To identify which partition iteration $t$ belongs to, we define the mapping $i(t) = \mathrm{argmax}_{i \in \mathbb{Z}_+} \{ k_{i} \leq t < k_{i+1}\}$.
    Since $C(k) \geq 1$, then $i(\tau) \leq i(k) \leq \log_2(k)$ for any $\tau \leq k$.

    First, for any $\tau \leq k$, we can use an argument similar to~\eqref{eq:high_tail_ah_second_momement} 
    \begin{talign} \label{eq:controlled_variance_recursive}
        \mathbb{E}_{q \sim \kappa^{\pi^*}_s} \big[ \sum_{t=0}^{\tau-1} \eta_t(\|\tQtq\|_*^2 + M_h^2)]
        &\leq
        \mathbb{E}_{q \sim \kappa^{\pi^*}_s} \big[ \sum_{t=0}^{k-1} \eta_t(\|\tQtq\|_*^2 + M_h^2)] \nonumber \\
        &\leq
        25\mu_h^{-1} (\bar{Q}^2 + M_h^2) \log(4k\vert \cS \vert/\delta), \ \forall s \in \cS,
    \end{talign}
    where the second inequality holds with probability $1-\delta$, 
    and we used the bounds $\sum_{t=0}^{\tau-1} t^{-1} \leq \log(2\tau) \leq \log(2k)$ and $\sum_{t=1}^{\tau} t^{-2} \leq 4$ to simplify the inequality.
    
    Second, we claim that for any $\tau$ where $\tau \in [k_i, k_{i+1}-1]$, then
    \begin{talign} \label{eq:monotone_l1_distance}
        D_{\|\cdot\|, [k_i,\tau]} 
        := 
        \max_{t' = k_i,\ldots,\tau} \max_{s' \in \cS} \|\pi_{t'}(\cdot \vert s')-\pi^*(\cdot \vert s')\|
        \leq
        2^{-i/2+1}, \ \forall s' \in \cS.
    \end{talign}
    The above says the distance to optimal solution is decreasing.
    Taking the above for granted as being true, then using an argument similar to~\eqref{eq:high_tail_ah_bias}, we have at iteration $\tau$ (recall $\tau < k_{i+1})$,
    \begin{talign}
        \sum_{t=k_i}^{\tau} \zeta_t(q, \pi^*) 
        &\leq
        \sigma D_{\|\cdot\|, [k_i,\tau]} \sqrt{3(\tau-k_i+1) \log(2 \vert \cS \vert/\delta)} \nonumber \\
        % &\leq
        % \sigma D_{\|\cdot\|, [k_i,\tau]} \sqrt{3(k_{i+1}-k_i) \log(2 \vert \cS \vert/\delta)} \nonumber \\
        &\leftstackrel{\substack{k_i = C \cdot 2^{i-1}\\ \text{and~\eqref{eq:monotone_l1_distance}}}}{\leq}
        % \sigma \sqrt{3C \log(2 \vert \cS \vert/\delta)} \cdot D_{\|\cdot\|, [k_i,\tau]} 2^{i/2} \nonumber \\ 
        % &\leftstackrel{\eqref{eq:monotone_l1_distance}}{\leq}
        \sigma \sqrt{3C \log(2 \vert \cS \vert/\delta)} \cdot 2, \ \forall q \in \cS, \label{eq:bias_recursive_bound}
    \end{talign}
    where the first inequality holds with probability $1-\delta$.
    % the cumulative bias is constant, i.e.,
    % \begin{talign} \label{eq:controlled_bias_recursive}
    %     \textstyle \sum_{t=k_{i(\tau)}}^{\min\{\tau, k_{i(\tau)+1}\} -1} \zeta_t(s)
    %     \leq
    %     2\sigma\sqrt{3C \log(2 \vert \cS \vert/\delta)}, \ \forall s \in \cS.
    % \end{talign}
    Then assuming the above bound holds for all iterations less than or equal to $\tau-1$, we arrive at
    \begin{talign} 
        &(1-\gamma)\sum_{t=0}^{\tau-1} [V^{\pi_t}(s) - V^{\pi^*}(s)]  + \eta_{\tau-1}^{-1}\mathbb{E}_{q \sim \kappa^{\pi^*}_s} [D^{\pi^*}_{\pi_\tau}(q)] \nonumber \\
        &\leftstackrel{\eqref{eq:agg_convergence_strong_convex}}{\leq}
        \mu_h \bar{D}_0 + \mathbb{E}_{q \sim \kappa^{\pi^*}_s} \big[ \sum_{t=0}^{\tau-1} \eta_t(\|\tQtq\|_*^2 + M_h^2) + \sum_{i=0}^{i(\tau)-1} \sum_{t=k_i}^{\min\{\tau, k_{i+1}\}-1 } \zeta_t(q, \pi^*) \big] \nonumber \\
        % &\leftstackrel{\eqref{eq:high_tail_ah_second_momement}}{\leq}
        % \mu_h \log \vert \cA \vert + 25\mu_h^{-1} (\bar{Q}^2 + M_h^2) \log(4k\vert \cS \vert/\delta) 
        % + \sum_{i=0}^{i(\tau)-1} \mathbb{E}_{q \sim \kappa^{\pi^*}_s} \big[  \sum_{t=k_i}^{\min\{\tau, k_{i+1}\} -1} \zeta_t(q) \big], \nonumber \\
        &\leftstackrel{\substack{\eqref{eq:controlled_variance_recursive}~\text{and}~\eqref{eq:bias_recursive_bound}\\
        \text{and}~i(\tau) \leq \log(k)}}{\leq}
        \mu_h \bar{D}_0 + 25\mu_h^{-1} (\bar{Q}^2 + M_h^2) \log(4k\vert \cS \vert/\delta) 
        + 2\log(k) \sigma \sqrt{3C \log(2 \vert \cS \vert/\delta)}. \label{eq:reindexed_agg_convergence_strong_convex} 
    \end{talign}
    Notice in the last line above, we removed all dependence on $\tau$ as long as $\tau \leq k$.
    
    So, if we can show~\eqref{eq:monotone_l1_distance} for all $\tau \leq k-1$, then we get~\eqref{eq:reindexed_agg_convergence_strong_convex} with $\tau=k$, which is what we wanted to show.
    So to prove~\eqref{eq:monotone_l1_distance} for $\tau \leq k-1$, we will use mathematical induction on $\tau$.
    The base case of $\tau=0,1,\ldots,k_1-1$ (i.e., $\tau \in [k_i,k_{i+1}-1]$ with $i=0$) is true because $\|\pi_{\tau}(\cdot \vert s) - \pi^*(\cdot \vert s)\| \leq 2$ (recall $\max_{p \in \DA}\|p\| \leq 1$).
    For the inductive hypothesis case,
    we have for $\tau = k_1$ and any state $s' \in \cS$,
    \begin{talign}
        \|\pi_\tau(\cdot \vert s') - \pi^*(\cdot \vert s')\|^2 
        &\leq D^{\pi^*}_{\pi_\tau}(s') \nonumber \\
        &\leftstackrel{\eqref{eq:visitation_measure}}{\leq}
        \frac{1}{1-\gamma} \mathbb{E}_{q \sim \kappa^{\pi^*}_{s'}} [D^{\pi^*}_{\pi_\tau}(q)] \nonumber \\
        &\leftstackrel{\substack{\eqref{eq:reindexed_agg_convergence_strong_convex}~ \text{and} \\
        \eta_{\tau-1} = 1/(\mu_h \tau)}}{\leq}
        \frac{\mu_h \bar{D}_0 + 25\mu_h^{-1} (\bar{Q}^2 + M_h^2) \log(4k\vert \cS \vert/\delta) 
        + 2\log(k) \sigma \sqrt{3C \log(2 \vert \cS \vert/\delta)}}{(1-\gamma) \mu_h \tau} \nonumber \\
        % &\leq
        % \frac{\mu_h \log \vert \cA \vert + [25\mu_h^{-1} (\bar{Q}^2 + M_h^2) + 2\sigma \sqrt{3C}](\log(4k\vert \cS \vert/\delta))^{3/2}}{(1-\gamma) \mu_h \tau} \nonumber \\
        % \frac{1}{(1-\gamma) \mu_h k} \Big[ \mu_h \log \vert \cA \vert + \frac{(\bar{Q}^2 + M_h^2) \log(2k)}{\mu_h}  + \frac{24\bar{Q}^2 \log(\frac{2 \vert \cS \vert}{\delta})}{\mu_h}
        % + 2\sigma T \sqrt{3C \log\frac{2\vert \cS \vert}{\delta}} \Big] \nonumber \\
        &\leftstackrel{\substack{{\tau \geq} k_{i}=C \cdot 2^{i-1}}}{\leq}
        2^{-i+1},  \label{eq:l1_geometric_decrease_recursion}
    \end{talign}
    % {\bf GL: In these lines above, do you want to show $\tau = k_1$ or $\tau = k_i$?}
    % {\textcolor{red}{Caleb: I want to show this for general $\tau$, which is why I used the index $\tau$ rather than $k_i$. We just need $k_i$ to satisfy $\tau \geq k_i$ (in order to use the partitioning), as I have added in the last line above.}}
    where in the last line we recall the definition of $C \equiv C(k)$ as defined in the statement of Theorem~\ref{thm:strongly_convex_agg_convergence}.
    Together with the inductive hypothesis, this establishes~\eqref{eq:monotone_l1_distance} for any $\tau \leq k_1$.
    Since the step size $\eta_t$ is decreasing, then one can similarly show~\eqref{eq:l1_geometric_decrease_recursion} for $\tau=k_{i}+1$, and therefore we can establish~\eqref{eq:monotone_l1_distance} for $\tau=k_{i}+1$.
    Successively repeating this argument for $\tau=k_i+2,\ldots,k_{i+1}-1,k_{i+1},\ldots,k-1$ will complete the proof by induction of~\eqref{eq:monotone_l1_distance} for all $\tau=0,\ldots,k-1$, as desired.

    The total failure rate is $(1+k)\delta$, where a failure of $\delta$ is from applying~\eqref{eq:high_tail_ah_second_momement} in~\eqref{eq:controlled_variance_recursive} and (by union bound) a failure rate of $k\delta$ from applying~\eqref{eq:high_tail_ah_bias} within~\eqref{eq:bias_recursive_bound} at most $k$ times.
% 
    % ({\bf GL: would you please take a look at a slightly more rigorous proof in Lemma 5.4 of ``Policy Mirror Descent Inherently Explores Action Space". The key requirement there is to assume the boundedness of $\zeta_t$. Without that assumption, the line in the bottom of page 22 of that paper does not necessarily go through. It is not clear to me how this requirement is removed in your argument above. 
    % ps. By generalizing Lemma 5.4 there to allow variable bias and variances for each summand, we do not need to decompose the iteration into partitions.
    % })
\end{proof}

\section{Proofs from Section~\ref{sec:validation_analysis}} \label{sec:proofs_for_validation_analysis}
To prove last-iterate convergence, we need a technical result (see~\cite[Lemma 10]{orabona2021parameter} for a proof).
% The variant of this result seems to have first appeared in~\cite{shamir2013stochastic}.
% We only state the result here, where a proof can be found in~\cite[Lemma 10]{orabona2021parameter}.
\begin{lemma} \label{lem:last_iter_general}
    For non-increasing constants $\alpha_t > 0$ and a nonnegative sequence $X_t$, 
    \begin{talign*}
        \alpha_{k-1} X_{k-1}
        &\leq
        k^{-1} \sum_{t=0}^{k-1} \alpha_t X_t
        +
        \sum_{\ell=1}^{k-1} \frac{1}{\ell(\ell+1)} \sum_{t=k-\ell}^{k-1} \alpha_t (X_t - X_{k-\ell-1}).
    \end{talign*}
\end{lemma}

\begin{proof}[Proof of Proposition~\ref{prop:convex_last_iterate_with_step_size}]
    We focus on the strongly convex regularization ($\mu_h > 0$), since the case for $\mu_h \geq 0$ is similar and simpler.
    We define $X_t := \mathbb{E}[V^{\pi_t}(s) - V^{\pi^*}(s)] \geq 0$ in Lemma~\ref{lem:last_iter_general} to show
    \begin{talign*} 
        &(1-\gamma) \mathbb{E}[V^{\pi_{k-1}}(s) - V^{\pi^*}(s)] \\
        % &\stackrel{\text{Corollary}~\ref{cor:fixed_mu_0}}{\leq}
        &\leq
        k^{-1} \sum_{t=0}^{k-1} \mathbb{E}[V^{\pi_t}(s) - V^{\pi^*}(s)] + 
        \sum_{\ell=1}^{k-1} \frac{1}{\ell(\ell+1)} \sum_{t=k-\ell}^{k-1} \alpha_t \mathbb{E}[V^{\pi_t}(s) - V^{\pi_{k-\ell-1}}(s)].
    \end{talign*}
    The first summand can be bounded by Theorem~\ref{thm:strongly_convex_agg_convergence}. 
    For the second summand (over index $\ell$), 
    we have the following auxiliary result for any $k_0 \in [0,k-1]$, which can be shown similarly to~\eqref{eq:agg_convergence_general},
    \begin{talign*}
        &(1-\gamma)\sum_{t=k_0}^{k-1} \mathbb{E}[V^{\pi_t}(s) - V^{\pi_{k_0}}(s)] \\
        &~~~\leftstackrel{\substack{\text{Lemma~\ref{lem:pmd_value_descent}} \\ \text{w/ $\pi=\pi_{t_0}$}}}{\leq}
        \mathbb{E}_{q \sim \kappa^{\pi^*}_s} \big[ D^{\pi_{t_0}}_{\pi_{t_0}}(q) + \sum_{t=t_0}^{k-1} \eta_t^2(\mathbb{E}\|\tQtq\|_*^2 + M_h^2) + \sum_{t=t_0}^{k-1} \mathbb{E} \eta_t \zeta_t(q,\pi_{t_0}) \big]   \\
        &\leq
        \mu_h^{-1} (\bar{Q}^2 + M_h^2) \sum_{t=k_0}^{k-1} \frac{1}{t+1} + 2\varsigma (k-k_0),
    \end{talign*}
    where $\zeta_t(q,\pi)$ is from Lemma~\ref{lem:pmd_value_descent}.
    %% Missing proof below
    % vvvvvvvvvvvvvvvvvvvvvvvv
    % We start with
    % \begin{talign} 
    %     &(1-\gamma)\sum_{t=0}^{k-1} \mathbb{E}[V^{\pi_t}(s) - V^{\pi^*}(s)] \nonumber \\
    %     &~~~~\leftstackrel{\text{Lemma~\ref{lem:pmd_value_descent}}}{\leq}
    %     \mathbb{E} \mathbb{E}_{q \sim d^\pi_s}[\eta_0^{-1} D^{\pi}_{\pi_0}(q)] + \sum_{t=0}^{k-2} \big[-(\eta_t^{-1} + \mu) + \eta_{t+1}^{-1} \big] \mathbb{E}\mathbb{E}_{q \sim d^\pi_s} [D^{\pi}_{\pi_t}(q)] \nonumber \\
    %     &~~~~\hspace{10pt} 
    %     + (\bar{Q}^2 + M_h^2) \sum_{t=0}^{k-1} \eta_t + 2\varsigma k \nonumber \\
    %     &~~~~\leftstackrel{\substack{\text{$\pi_0$ is uniform and}\\ \text{$\eta_{t+1}^{-1}-(\eta_t^{-1}+\mu) \leq 0$}}}{\leq}
    %     \eta_0^{-1} \log \vert \cA \vert + (\bar{Q}^2 + M_h^2) \sum_{t=0}^{k-1} \eta_t + 2\varsigma k
    %     \label{eq:agg_convergence_general_2}. 
    % \end{talign}
    % The second result can be similarly shown, so we skip it.
    % ^^^^^^^^^^^^^^^^^^^
    Putting everything we have established so far together, we derive
    \begin{talign*} 
        &(1-\gamma) \mathbb{E}[V^{\pi_{k-1}}(s) - V^{\pi^*}(s)] \\
        % &\stackrel{\text{Corollary}~\ref{cor:fixed_mu_0}}{\leq}
        &\leq
        \frac{\mu_h \bar{D}_0 + \mu_h^{-1} (\bar{Q}^2 + M_h^2) \log(2k)  + 2 \varsigma k}{k} 
        +
        \mu_h^{-1} (\bar{Q}^2 + M_h^2)\sum_{\ell=1}^{k-1} \frac{1}{\ell(\ell+1)}\sum_{t=k-\ell}^{k-1} \frac{1}{t+1} + 2\varsigma\sum_{\ell=1}^{k-1} \frac{\ell+1}{\ell(\ell+1)} \\
        &\leq
        \frac{\mu_h \bar{D}_0 + \mu_h^{-1} (\bar{Q}^2 + M_h^2) \log(2k)  + 2 \varsigma k}{k}
        +
        \frac{2\mu_h^{-1} (\bar{Q}^2 + M_h^2)\log(2k)}{k} + 2\varsigma \log(2k),
    \end{talign*}
    where the last line uses the bounds $\sum_{t=0}^{k-1} (t+1)^{-1} \leq \ln(2k)$ and
    \begin{talign*}
        \sum_{\ell=1}^{k-1} \frac{1}{\ell(\ell+1)} \sum_{t=k-\ell}^{k-1} \frac{1}{t+1}
        \leq
        \sum_{\ell=1}^{k-1} \frac{1}{\ell(\ell+1)} \cdot \frac{\ell+1}{k-\ell} 
        =
        \textstyle \sum_{\ell=1}^{k-1} \big( \frac{1}{\ell k} + \frac{1}{k(k-\ell)} \big)  
        =
        \frac{2}{k} \sum_{\ell=1}^{k-1} \frac{1}{\ell} 
        &\leq
        \frac{2\log(2k)}{k}. 
     \end{talign*}
\end{proof}

\bibliographystyle{siamplain}
\bibliography{references}
\end{document}